\documentclass[11pt,letterpaper]{article}
\usepackage[top=1in,bottom=1in,left=1in,right=1in]{geometry}
\usepackage[utf8]{inputenc}
\usepackage[T1]{fontenc}
\usepackage{microtype}
\usepackage{xcolor,graphicx,booktabs}
\usepackage{amsmath,amssymb}
\usepackage{natbib}
\setcitestyle{authoryear,square,comma}
\usepackage{xurl,hyperref}
\usepackage{authblk}
\usepackage{subcaption}
\renewcommand{\arraystretch}{2}
\allowdisplaybreaks
\usepackage{multirow,makecell,tabularx}
\usepackage{flafter,enumitem}
\usepackage[most]{tcolorbox}
\usepackage{appendix}
\usepackage{float,placeins,etoolbox}
\usepackage{etoc}
\usepackage{tikz}

\usepackage{amsmath,amsfonts,bm}

\def\eqref#1{equation~\ref{#1}}

\def\1{\bm{1}}

\DeclareMathAlphabet{\mathsfit}{\encodingdefault}{\sfdefault}{m}{sl}
\SetMathAlphabet{\mathsfit}{bold}{\encodingdefault}{\sfdefault}{bx}{n}

\usepackage{amsmath,amssymb,amsfonts,amsthm,mathtools}
\usepackage{xcolor}
\usepackage{enumitem}
\usepackage{hyperref}
\usepackage[nameinlink,capitalize,noabbrev]{cleveref}

\hypersetup{
  hypertexnames=false,
  colorlinks=true,
  linkcolor=blue!60!black,
  citecolor=blue!60!black,
  urlcolor=blue!60!black
}

\newtheorem{theorem}{Theorem}[section]
\newtheorem{lemma}[theorem]{Lemma}
\newtheorem{proposition}[theorem]{Proposition}
\newtheorem{corollary}[theorem]{Corollary}
\theoremstyle{definition}
\newtheorem{definition}[theorem]{Definition}

\theoremstyle{remark}

\crefname{theorem}{Theorem}{Theorems}
\Crefname{theorem}{Theorem}{Theorems}
\crefname{lemma}{Lemma}{Lemmas}
\Crefname{lemma}{Lemma}{Lemmas}
\crefname{proposition}{Proposition}{Propositions}
\Crefname{proposition}{Proposition}{Propositions}
\crefname{corollary}{Corollary}{Corollaries}
\Crefname{corollary}{Corollary}{Corollaries}
\crefname{definition}{Definition}{Definitions}
\Crefname{definition}{Definition}{Definitions}
\crefname{remark}{Remark}{Remarks}
\Crefname{remark}{Remark}{Remarks}
\crefname{equation}{Equation}{Equations}
\Crefname{equation}{Equation}{Equations}

\crefalias{lemma}{lemma}
\crefalias{proposition}{proposition}
\crefalias{corollary}{corollary}
\crefalias{definition}{definition}
\crefalias{remark}{remark}
\crefalias{example}{example}

\DeclarePairedDelimiterX{\inner}[2]{\langle}{\rangle}{#1,#2}

\providecommand{\CondAdaZeroPrimary}{0.458}
\providecommand{\CondAdaPointOnePrimary}{0.428}
\providecommand{\CondGruPrimary}{0.380}
\providecommand{\CondBestDelta}{+0.048}

\providecommand{\CondBestClosurePercent}{38.5}
\providecommand{\CondBestWorseGravityCount}{22}
\providecommand{\CondBlockWorseGravityCount}{25}
\providecommand{\CrossPredGruSource}{-0.107}
\providecommand{\CrossPredTfSource}{-0.103}

\definecolor{HighlightGreen}{HTML}{007F63}
\definecolor{HighlightGreenBg}{HTML}{F1F8F5}
\definecolor{TakeawayGreen}{HTML}{B5E300}
\definecolor{TakeawayGreenBg}{HTML}{F7FBEA}

\definecolor{HighlightBlue}{HTML}{046487}
\definecolor{HighlightBlueBg}{HTML}{DDEDEC}

\newtcolorbox{statusbox}[1][]{
  colframe=HighlightGreen,
  colback=HighlightGreenBg,
  boxrule=0.9pt,
  arc=3pt,
  left=6pt, right=6pt, top=5pt, bottom=5pt,
  title={#1},
  fonttitle=\bfseries\small
}
\usetikzlibrary{arrows.meta,backgrounds,calc,fit,positioning}
\newcommand{\stoptocentries}{\addtocontents{toc}{\protect\setcounter{tocdepth}{-1}}}
\newcommand{\starttocentries}{\addtocontents{toc}{\protect\setcounter{tocdepth}{2}}}

\title{Semigroup-JEPA: Latent Dynamics Consistency for\\
Zero-Shot Physics Generalization}
\author{%
Andy Zeyi Liu\textsuperscript{1,*}\quad
Haoran Sun\textsuperscript{1,*}\quad
Lucas Baker\textsuperscript{2}\\
Randall Balestriero\textsuperscript{3}\quad
John Sous\textsuperscript{1,\textdagger}\vspace{6pt}\\
\textsuperscript{1}Yale University\quad
\textsuperscript{2}Jump Trading\quad
\textsuperscript{3}Brown University\\
\textsuperscript{*}Equal contribution\quad
\textsuperscript{\textdagger}Correspondence:
\href{mailto:john.sous@yale.edu}{\texttt{john.sous@yale.edu}}\\

}
\date{}

\hypersetup{
  colorlinks=true,
  linkcolor=red!80!black,
  citecolor=blue,
  urlcolor=magenta,
  pdftitle={Semigroup-JEPA: Latent Dynamics Consistency for Zero-Shot Physics Generalization},
  pdfauthor={Andy Zeyi Liu, Haoran Sun, Lucas Baker, Randall Balestriero, John Sous}
}
\begin{document}
\stoptocentries

\maketitle
\vspace{-6mm}

\begin{abstract}
Joint-Embedding Predictive Architecture (JEPA) world models learn a compact latent representation of the world that supports prediction and planning, but their capability to learn physics and generate physically realistic dynamics remains hitherto untested. In this work, we introduce \textit{SemiGroup-JEPA (SG-JEPA)}, which extends the LeWorldModel framework by supplying the parameter governing the physics to the temporal model via action-conditioning and jointly training an encoder and predictor through an autoregressive latent rollout. To evaluate the model's ability to generalize out of distribution, we design dynamical tasks under different gravitational fields that, despite obeying the same physical law, exhibit qualitatively different dynamics, ranging from floating motion in weak gravitational fields to rapid bouncing in strong ones. In contrast to DINO-WM, SG-JEPA reduces open-loop prediction error by up to $2\times$ on two-dimensional datasets, and increases control success rate up to $2.5\times$ for three-dimensional robotic datasets, for which we train independent diffusion policies. To explain this advantage, we develop a linear feature model that separates local law-conditioned error from its recursive amplification under rollout. Guided by this model, we find that back-propagating the multi-step rollout loss into the representation trains the encoder to keep the features that the predictor can carry forward, and that those are the features the dynamics depend on, so most of the gain comes from the encoder learning better features rather than from the predictor learning better dynamics. See project page at \href{https://sg-jepa.github.io}{sg-jepa.github.io}.
\end{abstract}
\vspace{-2mm}

\begin{figure}[t]
    \centering
    \includegraphics[width=\linewidth]{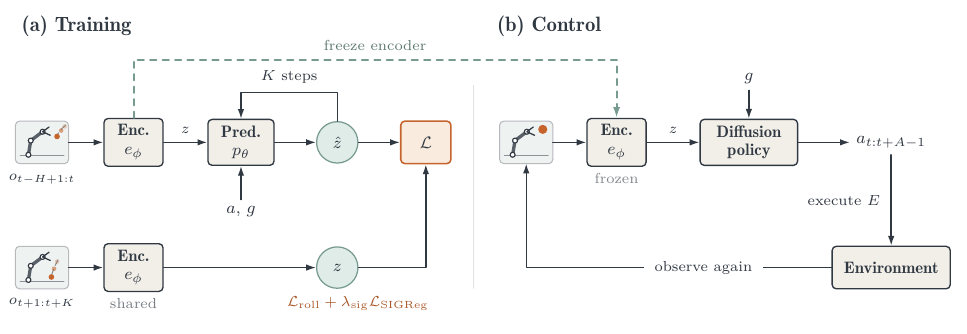}
    \caption{\textbf{SG-JEPA training and control.}
    (a) The encoder and gravity-conditioned predictor are jointly trained with a discounted $K$-step autoregressive rollout loss and SIGReg. Targets use the same trainable encoder. (b) A gravity-conditioned diffusion policy is trained on demonstrations from successful episodes with the encoder frozen. At inference, it generates $a_{t:t+A-1}$, executes the
    first $E$ actions, and replans from new observations.}
    \label{fig:glewm-framework}
\end{figure}

\section{Introduction}
The long-standing goal of world models is to understand dynamics of the environment in which the model agent is deployed so that the agent can predict subsequent states and consequences of its actions. Joint-Embedding Predictive Architecture (JEPA) emerged as a promising approach, whereby rather than reconstructing future pixels, it learns latent representations by jointly training an encoder and predictor \citep{lecun2022path,sobal2025jepa,zhou2025dinowm}. Recent proposals such as DINO-WM \citep{zhou2025dinowm} and LeWorldModel (LeWM) \citep{maes2026lewm} have demonstrated that these latent predictions support downstream planning and control under a fixed frame-to-frame transition law. We push this further by asking the question:
\begin{tcolorbox}[
  colframe=HighlightGreen,
  colback=HighlightGreenBg,
  boxrule=1.8pt,
  arc=4pt,
  left=7pt,
  right=7pt,
  top=6pt,
  bottom=4pt,
  width=\linewidth
]
\vspace{0em}
\begin{center}
\textit{Has the model actually learned the physical dynamics when its learned update is composed over long horizons? Does it generalize to out-of-distribution (OOD) dynamics?}
\end{center}
\vspace{-1em}
\end{tcolorbox}
\noindent If the answer is yes, the consequences are significant. A world model that genuinely learns dynamics would enable an agent to adapt in new environments, e.g. a model trained on Earth dynamics generalizing to Martian dynamics.

In this work, we take gravitational dynamics as the primary testbed for these questions, and introduce \textit{Semigroup-JEPA (SG-JEPA)}, a gravity-conditioned extension of LeWM with recursive physical prediction (Fig.~\ref{fig:glewm-framework}). SG-JEPA supplies gravity alongside action and jointly trains its encoder and predictor through a discounted autoregressive rollout loss, while SIGReg~\citep{balestriero2025lejepa} regularizes the latent space. On various MuJoCo datasets covering two-dimensional (2D) rigid body freefall, three-dimensional (3D) projectile motion and robot-arm control, we train SG-JEPA on a narrow range of gravity values and evaluate on a wider grid that extends far beyond it. We study two aspects of generalization, namely transfer to OOD gravity values and accuracy under long-horizon rollout. Our results show that the dynamics SG-JEPA learns within a narrow gravity range remain accurate over long horizons at OOD gravity values, demonstrating zero-shot physics generalization.

Across these gravitational environments we find that how well the encoder preserves dynamical information determines the ability to generalize out of distribution. To interpret this behavior, we mechanistically isolate the encoder from the predictor and measure how performance depends on the rollout horizon and on the distance of the gravity parameter from the training range. Starting from a trained checkpoint, we replace either the encoder or the predictor with a freshly initialized one and retrain jointly. Re-training the encoder leaves performance unchanged, whereas re-training the predictor degrades it. The gain is therefore associated with what the GRU predictor learns, and this advantage widens as gravity moves further from the training range and as the rollout horizon grows.

\paragraph{Related work.}
\label{sec:related}
JEPA models introduce a new paradigm of learning by predicting target representations rather than reconstructing pixels, and have been extended from images to video \citep{lecun2022path,assran2023ijepa,bardes2024vjepa,assran2025vjepa2}.
DINO-WM \citep{zhou2025dinowm} applies this idea with frozen DINOv2 features \citep{oquab2024dinov2}, while LeWM \citep{maes2026lewm} jointly trains its encoder and predictor with SIGReg \citep{balestriero2025lejepa}. Recent work probes physics generalization under held-out configurations and asks which physical parameters predictive latents retain \citep{xue2026acwmphys,tan2026physicalidentifiability}. Our setting differs in isolating gravity as a physical quantity governing the dynamics, which lets us systematically test the learned dynamics outside the training range. A separate line of work addresses rollout error by training on the model's own predictions or optimizing beyond one step \citep{bengio2015scheduledsampling,venkatraman2015dad,lambert2021longterm}. In this work, we address these two problems jointly and introduce SG-JEPA, which back-propagates a recursive latent rollout loss through both the encoder and predictor, enabling, as we show below, OOD generalization over long-horizon dynamics. The name reflects that, for action-free trajectories at a fixed gravity value, the repeated updates form a discrete semigroup. Related operator-based analyses of how local representation and predictor errors propagate appear in \citep{williams2015edmd,lusch2018deepkoopman}. Appendix~\ref{app:related} provides more discussion.

\section{Problem Setting} 
\label{sec:method}
In this section, we introduce Semigroup-JEPA (SG-JEPA). We first describe the architecture and training loss. We then introduce the customized, gravity-controlled MuJoCo environments and the baseline models we compare against. Finally we explain evaluation protocols for prediction of physical quantities and latent control.

\subsection{Model architecture}
\paragraph{Semigroup-JEPA.} In Fig.~\ref{fig:glewm-framework} we sketch the architecture of SG-JEPA. It keeps the two components of LeWM, an encoder and a predictor. At time step $t$, the model has observation $o_t$ and control $u_t$. The observation $o_t$ is then encoded using a Vision Transformer~\citep{dosovitskiy2021imageworth16x16words} into low-dimensional latents $z_t$. Unless stated otherwise, we use the ViT-Tiny configuration, with 12 Transformer layers, 3 attention heads, and hidden width 192. We use $8\times 8$ patches for $128\times 128$ resolution and $16\times 16$ for $256\times 256$ resolution. The encoder is trained from random initialization jointly with the predictor. We use the final layer $\texttt{[CLS]}$ token $h_t\in\mathbb{R}^{192}$ as a global summary of the observed frame $o_t$ and map it to $z_t\in\mathbb{R}^{256}$ using a one-hidden-layer projector. The latents $z_t$ are then fed into the predictor. Gravity enters the model through the action. Each episode has a signed gravity scalar parameter $g$, which we treat as an extra coordinate of the action, and maintain constant across frames,
\begin{equation}
    \tilde a_t = [u_t;g], \qquad c_t=q_\psi(\tilde a_t),
\end{equation}
where $q_{\psi}$ is the action encoder, a temporal convolution followed by a SiLU MLP, which yields $c_t\in\mathbb{R}^{256}$. The input $g$ is z-scored using statistics from the training split. In the case where there are no external actions, e.g. free-fall, we still input a single scalar $g$ in the action conditioning. 

We consider three choices of predictor, namely a Transformer, a GRU, and a state-space model (SSM). Unless otherwise noted, SG-JEPA refers to either the GRU or the SSM variant. The Transformer predictor has 6 layers and 16 attention heads, in line with the Original LeWM, with actions fed into every block through Adaptive Layer Normalization (AdaLN)~\citep{peebles2023scalablediffusionmodelstransformers}. The GRU predictor defaults to 3 residual layers of width 512. In each layer, the hidden state is concatenated with the projected action embedding, passed through an MLP, and then processed by a single-layer GRU. The SSM predictor uses the same depth, width, and action-concatenation scheme, replacing each GRU with a Mamba/S6-style selective state-space block~\citep{gu2023mamba}. The predictor takes as input a history of length $H$ and predicts the next latent
\begin{equation}
    \hat z_{t+1}
    =p_\theta\!\left(z_{t-H+1:t},c_{t-H+1:t}\right).
    \label{eq:glewm-predictor}
\end{equation}
Exact architectures and gravity-fusion ablations are given in Appendix~\ref{app:model-details}.

\begin{figure}[t]
    \centering
    \includegraphics[width=\linewidth]{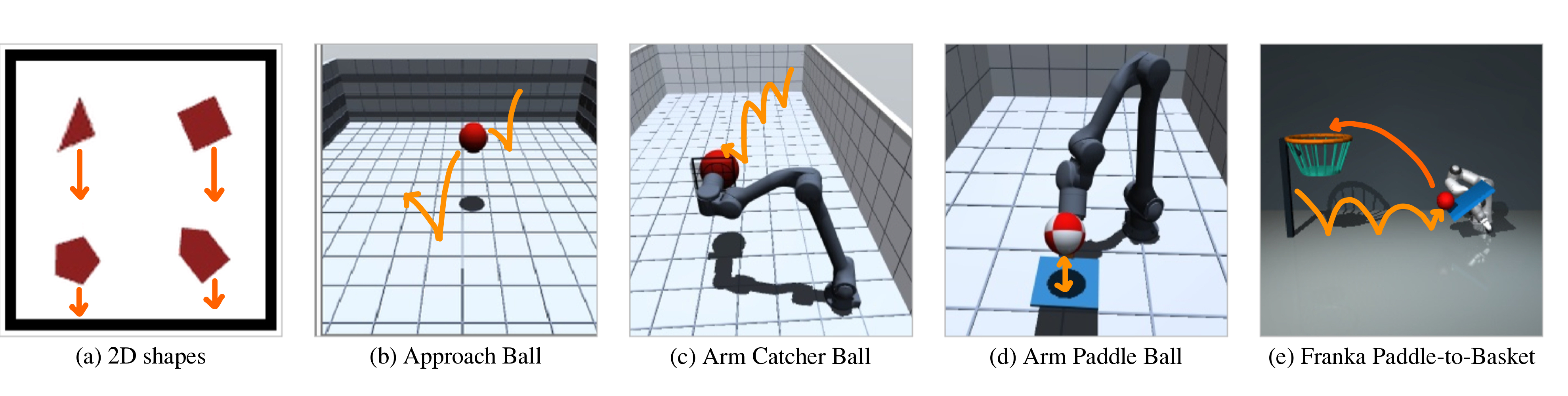}
    \caption{\textbf{Evaluation environments.}  (a) Planar shapes: each object is given an
    impulse at \(t=0\) and then fall and collide freely after inside the box.  (b) Approach
    Ball: projectile flight from the far end of the scene.  (c) Arm Catcher Ball: a
    robotic arm catches a ball with a catcher.  (d) Arm Paddle Ball: the
    same arm sustaining bounces with a tilting paddle.  (e) Franka
    Paddle-to-Basket: a Panda arm strikes an incoming ball into a
    basket. Orange arrows sketch the motion.}
    \label{fig:environments}
\end{figure}

\paragraph{Training objective.}
For fixed \(g\) and an action-free trajectory, iterating the shared history update \(F_{\theta,g}\) gives \(S_{\theta,g}(k)=F_{\theta,g}^{\circ k}\), with \(S_{\theta,g}(0)=I\) and \(S_{\theta,g}(k+\ell)=S_{\theta,g}(\ell)\circ S_{\theta,g}(k)\). These iterates form a discrete semigroup on the latent history. Without controls, consecutive update blocks compose into this semigroup. With controls, they still compose, but each block is modified by its action (Appendix~\ref{app:theory-history}).

SG-JEPA combines a \(K\)-step autoregressive rollout loss with SIGReg. Unlike LeWM's one-step teacher-forced objective, in which the predictor is always given true encoded latents as input rather than its own predictions, it starts from \(H\) encoded frames and recursively predicts \(K\) future latents, inserting each prediction into the next length-\(H\) history window. With \(z_s^{\mathrm{roll}}=z_s\) for \(s\leq t\) and \(z_s^{\mathrm{roll}}=\hat z_s\) for \(s>t\), the rollout is $\hat z_{t+k}=p_\theta\!\left(z^{\mathrm{roll}}_{t+k-H:t+k-1}, c_{t+k-H:t+k-1}\right)$, with $k=1,\ldots,K$. We use the normalized rollout loss
\begin{equation*}
\mathcal{L}_{\mathrm{roll}}
=
\sum_{k=1}^{K} w_k
\left\|\hat z_{t+k}-z_{t+k}\right\|_2^2,
\quad\text{where}\quad
w_k
=
\frac{\gamma^{k-1}}
{\sum_{j=1}^{K}\gamma^{j-1}}.
\end{equation*}
By construction, a composition law emerges from repeatedly applying the shared predictor. For \(\gamma<1\), later predictions receive less weight, so accumulated errors do not dominate the objective. The full regularized objective is
\begin{equation}
\mathcal{L}
=
\mathcal{L}_{\mathrm{roll}}
+
\lambda_{\mathrm{sig}}\mathcal{L}_{\mathrm{SIGReg}}.
\end{equation}
SIGReg acts only on encoded latents and encourages their random one-dimensional projections to match a standard Gaussian. Target latents come from the same trainable encoder, without stop-gradient.

\paragraph{Training setup.} For training, we use a hybrid Muon/AdamW optimizer. Muon~\citep{jordan2024muon} applies an orthogonalized update to matrix-valued parameters, so all 2D tensors are optimized with Muon at a learning rate of
$1\times 10^{-4}$ and all other tensors with AdamW at $5\times 10^{-5}$. In Appendix~\ref{sec:muon_vs_adam} we perform a controlled comparison against pure AdamW~\citep{adamw}, finding that Muon lowers the validation loss more rapidly and uses input data more efficiently. We train all models for 20 epochs. Unless stated otherwise, we use a history of $H=20,K=5$. The history-length ablation in Appendix~\ref{app:history-depth-sweep} demonstrates that this choice yields the best downstream performance and efficiency tradeoff. Appendix~\ref{app:gamma-decay-sweep} shows that $\gamma=0.95$ produces the strongest overall position and velocity performance. We study the effect of changing $\lambda_{\mathrm{sig}}$ in Appendix~\ref{app:sigreg-sweep}. Appendix~\ref{app:model-details} provides the complete training configuration.

\subsection{Datasets}
\label{sec:datasets}

We generate eight datasets in MuJoCo \citep{todorov2012mujoco}, covering passive rigid-body motion (freefall), projectile motion, and robot-arm control. We use six of them as environments for the experiments in the main text and the other two (2D pentagon and house) in the ablation studies in Appendix~\ref{app:house-shape-generalization}. Fig.~\ref{fig:environments} gives an overview and illustrates the ball dynamics. For training, $g$ is sampled from a narrow Gaussian, $\mathcal{N}(4,0.5^2)$ for the 2D planar shapes and Arm Catcher Ball, and $\mathcal{N}(9.8,2.0^2)$ for the remaining 3D datasets. Held-out test sets use a wider grid of $g$ values, including OOD ones. Five datasets test prediction: four planar shapes receive one initial impulse and then move freely inside a box, while Approach Ball contains 3D projectile motion. Three datasets test control: Arm Catcher Ball, Arm Paddle Ball, and Franka Paddle-to-Basket. Appendix~\ref{app:dataset-generation} gives the exact data splits, action schemas and other dataset details.

\paragraph{Baselines.}
We compare the performance of SG-JEPA against the following baselines. We first compare against the \emph{Original LeWM}, the vanilla model of \citep{maes2026lewm}, which uses a one-step latent prediction loss, SIGReg
on the encoded latents, and a causal Transformer predictor. We also compare against \textit{DINO-WM}~\citep{zhou2025dinowm}, for which we keep the pretrained DINOv2 encoder frozen and train an action-conditioned Transformer predictor, with action conditioning performed in the same way as in LeWM. We label the GRU- and SSM-predictor variants of SG-JEPA \emph{SG-JEPA (GRU)} and \emph{SG-JEPA (SSM)}, respectively. 

\subsection{Downstream evaluation}
\label{sec:method-eval}

\paragraph{Prediction.} For prediction tasks, we freeze the trained world model and fit an MLP probe $r_\eta$ that, given a short latent window, predicts the physical state $s_t$ (position, velocity, and rotation in 2D; position and velocity
in 3D) at time $t$.  The probe is trained only on training episodes (one episode is a video sample).  On held-out episodes we apply it to both encoded ground-truth frames and open-loop autoregressive rollouts at forecast horizons $\mathfrak{h}$ up to $\mathfrak{h}_{\max}=44$ frames. Let $\ell_r$ denote the probe-window length.  We report the excess physical-state error
\begin{equation}
    \frac{1}{\mathfrak{h}}\sum_{k=1}^{\mathfrak{h}}
    \left(
        \operatorname{NMSE}\!\left(
          r_\eta(z^{\mathrm{roll}}_{t+k-\ell_r+1:t+k}),s_{t+k}\right)
        -\operatorname{NMSE}\!\left(
          r_\eta(z_{t+k-\ell_r+1:t+k}),s_{t+k}\right)
    \right), 
    \label{eq:method-excess-endpoint}
\end{equation}
where $\operatorname{NMSE}$ denotes the normalized mean-squared error, with state coordinates normalized using statistics from the probe's training set. Subtracting the probe error on encoded ground-truth latents isolates the error introduced by the rollout.

\begin{figure}[t]
    \centering
    \includegraphics[width=\linewidth]{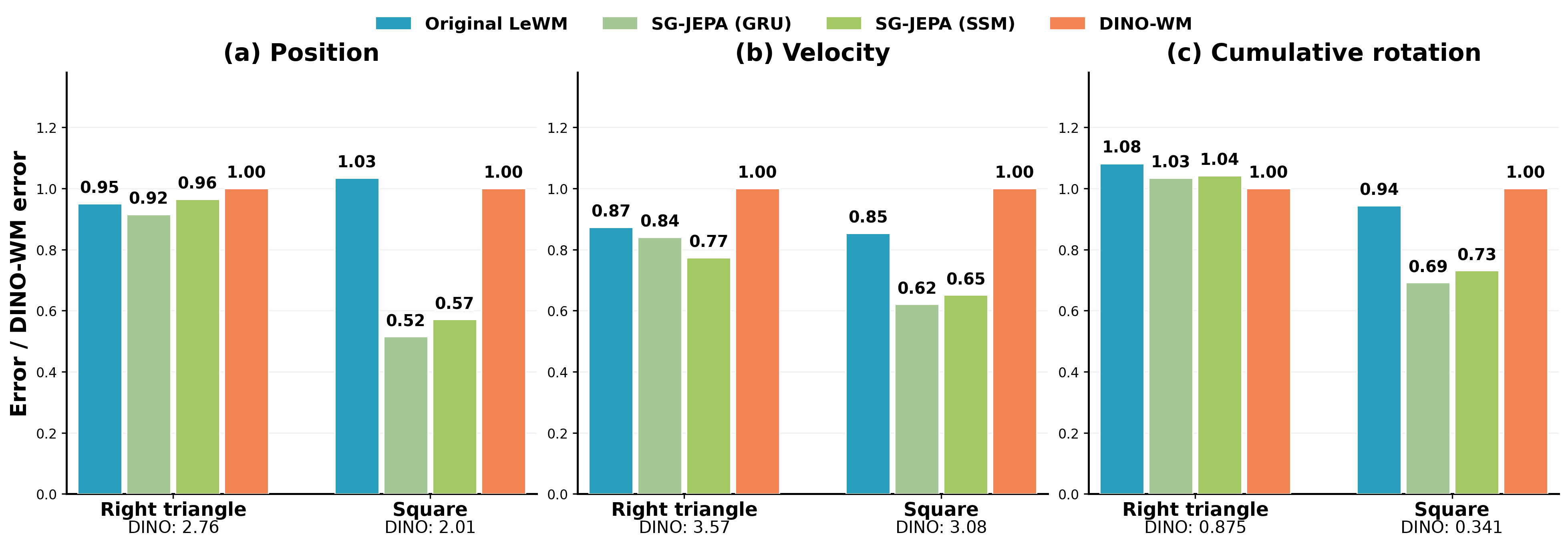}
    \caption{Long-horizon prediction at 44 rollout steps for 2D datasets, evaluated with MLP probes. Bars show position and velocity $L_2$ error and cumulative-rotation mean-absolute error (MAE), normalized by the DINO-WM error for each shape and metric. The absolute DINO-WM errors are shown below each group. Results are averaged over held-out episodes for each shape across 25 gravity values from $-2$ to 10, with training gravity sampled from $\mathcal{N}(4,0.5^2)$. Lower is better.}
    \label{fig:main-2d}
\end{figure}

\paragraph{Control.} For each control task and source model, we freeze the visual encoder and train a separate gravity-conditioned Diffusion Policy~\citep{chi2023diffusionpolicy} on its features. SG-JEPA and LeWM provide projected CLS latents, while DINO-WM provides mean-pooled frozen DINOv2 patch features. Fig.~\ref{fig:glewm-framework}(b) shows the resulting control
loop. A two-layer causal GRU compresses the most recent $H$ feature vectors into $\bar h_t=\operatorname{LayerNorm}\!\left(\operatorname{GRU}(z_{t-H+1:t})_{\mathrm{final}}\right)$, where $\operatorname{LayerNorm}$ denotes layer normalization, and the policy is conditioned on $\xi=[\bar h_t;g_{\mathrm{pol}}]$, with $g_{\mathrm{pol}}$ the z-scored gravity. The noise predictor $\epsilon_\omega$ is a conditional 1D U-Net. At diffusion step $\tau$, a control block $\mathbf{u}=u_{t:t+A-1}$ of $A$ steps from the training trajectories is corrupted as $\mathbf{u}^{(\tau)}=\sqrt{\bar\alpha_\tau}\,\mathbf{u}+\sqrt{1-\bar\alpha_\tau}\,\epsilon$, where $\epsilon\sim\mathcal{N}(0,I)$ and $\bar\alpha_\tau$ is the cumulative coefficient of the cosine noise schedule. The U-Net is trained to predict the added noise from $\mathbf{u}^{(\tau)}$, $\tau$, and the conditioning vector $\xi$ by minimizing
\begin{equation}
    \mathcal{L}_{\mathrm{DP}}
    =\mathbb{E}_{\tau,\epsilon}\!\left[
      \left\|\epsilon_\omega(\mathbf{u}^{(\tau)},\tau\mid\xi)
      -\epsilon\right\|_2^2\right].
    \label{eq:diffusion-policy-loss}
\end{equation}
We train the GRU and U-Net while keeping the visual encoder frozen. All policies use $H=20$, $A=16$, a 256-dimensional GRU, and a 100-step cosine diffusion schedule. At inference, the policy predicts $A$  controls, executes the first $E$, and replans from the new observation. We use $E=8$ for Arm Catcher Ball and Franka Paddle-to-Basket and $E=4$ for Arm Paddle Ball. Note that this execution interval is distinct from the world-model training rollout length $K$.

Our tasks differ from static-goal benchmarks such as Two Rooms and Push-T in that the ball moves, so the policy must infer its future motion from the observation history and gravity. Using a goal-conditioned Cross-Entropy Method planner in the same way would require a future target frame, which would give away the ball's future position. Our policies therefore do not use one.

\begin{figure}[t]
  \centering
  \includegraphics[width=\textwidth]{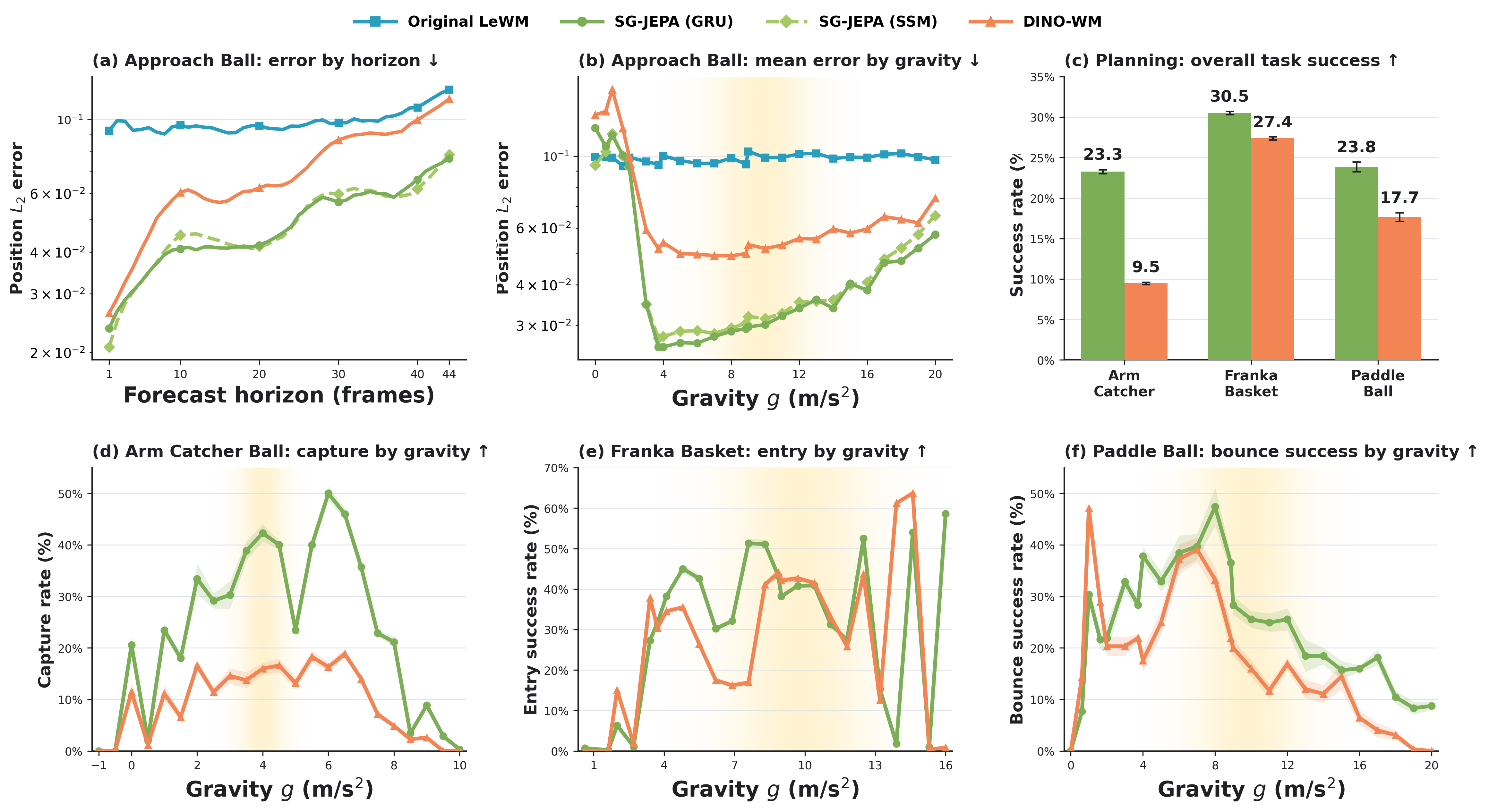}
  \caption{3D prediction and control tasks across held-out gravity values. {\emph{Open loop}}(a)--(b): Approach Ball position $L_2$ error on held-out episodes with a trained probe. {\emph{Closed loop}}(c)--(f): Evaluation results for Diffusion Policies trained on frozen world-model features, aggregated by task and resolved by gravity (higher is better). Results are averaged over five seeds. Arm Catcher Ball and Franka Paddle-to-Basket use $A=16$, $E=8$; Arm Paddle Ball uses $A=16$, $E=4$. Yellow shading marks the training-gravity distribution, $\mathcal{N}(4,0.5^2)$ for Arm Catcher Ball and $\mathcal{N}(9.8,2.0^2)$ for the other two.}
  \label{fig:main-3d}
\end{figure}

\section{Empirical results}
\label{sec:main_results}

In this section, we compare SG-JEPA against LeWM and DINO-WM across our 2D and 3D datasets. SG-JEPA improves long-horizon prediction and achieves higher average control success rates than the Original LeWM and DINO-WM, particularly at OOD gravity values.

\paragraph{2D datasets.} In Fig.~\ref{fig:main-2d} we present long-horizon physical-state prediction results. Each episode contains 64 frames and our models take a history of $H=20$, so at test time we roll out for 44 steps and probe the position, velocity, and rotation. Note that rotation is computed by integrating the probed angular velocity over the rollout, and its error is measured in number of turns. Complete per-gravity and error-by-horizon results are detailed in Appendix~\ref{sec:2d_full_results}. On the square object, we find that both variants of SG-JEPA have a clear advantage: relative to DINO-WM, SG-JEPA (GRU) reduces the three error types by 31--48\%, with SG-JEPA (SSM) close behind, while Original LeWM remains near the DINO-WM baseline. The advantage is most evident over longer rollouts (Appendix Table~\ref{tab:official-2d-shape-rollout-short-long}). An SG-JEPA variant leads every square metric from 5 through 44 rollout-steps.  On triangles, although DINO-WM wins by a small margin for short-horizon position and velocity, SG-JEPA achieves the lowest 44-rollout-step errors on both. This behavior is consistent with rollout training reducing the accumulation of errors when predictions are fed back into the model. The triangle dataset remains a challenging case: DINO-WM has slightly lower cumulative-rotation error, which we attribute mostly to asymmetric contacts amplifying small errors after collisions. As for OOD generalization, SG-JEPA has the lowest error across most of the held-out gravity grid on the square dataset, with significant advantages in velocity and rotation for large gravity parameter values (Appendix Fig.~\ref{fig:2d-per-gravity-results}).  Beyond gravity, we also test generalization to an unseen shape. We train SG-JEPA on triangles and squares and evaluate frozen-model rollouts on a ``house'' dataset, formed by attaching the right triangle to the square along the triangle's longest side, using a probe trained on house states. The frozen model still yields a useful representation of the house state, and its 44-rollout-step position and velocity predictions transfer reasonably well, particularly with a larger ViT encoder. A detailed breakdown is given in Appendix~\ref{app:house-shape-generalization}.

\paragraph{3D datasets.}
\phantom{xxx} 

\noindent{\textit{Prediction.}} On the Approach Ball dataset,  we find that the GRU and SSM SG-JEPA variants reduce mean position error by approximately $34\%$ relative to DINO-WM and $50\%$ relative to Original LeWM (Appendix Table~\ref{tab:approach-ball-aggregate}). From Fig.~\ref{fig:main-3d}(a), DINO-WM performs comparably well at short horizons but drifts more quickly during rollout. Both SG-JEPA variants remain more accurate over longer rollouts. The GRU and SSM variants perform similarly and alternate as the best method across horizons, suggesting that the improvement does not depend on a single temporal architecture. From Fig.~\ref{fig:main-3d}(b) we see that this pattern holds across gravity values. SG-JEPA is best at 22 of the 25 test values, with the exceptions being for very small gravity values, where the Original LeWM performs best. In Appendix~\ref{app:approach-ball-comparison} we present more results and observe similarly that  SG-JEPA retains its advantage for velocity generalization. 

\smallskip

\noindent{\textit{Control.}} For control tasks we evaluate each frozen encoder by training a task-specific Diffusion Policy on its latents. The policy runs in a receding-horizon loop using real observations. It uses the latest $H=20$ history frames to generate $A$ actions, executes the first $E$, and then observes the scene again before generating the next action sequence. We use $A=16,E=8$ for Arm Catcher Ball and Franka Paddle-to-Basket, and $A=16,E=4$ for Arm Paddle Ball. Control is thus open-loop within each action sequence but closed-loop over the episode. We use task-specific physical events to define success criteria. Arm Catcher Ball requires the ball to enter the central area of the catcher and be latched before the episode ends. Franka Paddle-to-Basket requires a valid paddle hit on the ball, followed by the ball landing in the basket. Arm Paddle Ball requires the robot arm to manipulate the paddle so that the ball never hits the floor throughout the rollout.  Moreover, we require it to complete at least one paddle bounce whose apex reaches \(0.28\,\mathrm{m}\). This is to avoid the degenerate case where the ball stops bouncing and stays on the paddle. 

In Fig.~\ref{fig:main-3d}(c) we show the aggregate results for the three tasks and (d)--(f) show their per-gravity breakdown. Each episode is evaluated with 5 policy rollouts, and we plot the mean and standard deviation. 
To keep the main figure readable, we present only the SG-JEPA (GRU) vs.\ DINO-WM comparison, as DINO-WM provides a stronger baseline for these tasks. Full results including Original LeWM can be found in the relevant appendices. Below we summarize our main observations.
\vspace{-2mm}
\begin{itemize}[noitemsep, leftmargin=1em, labelsep=0.45em]
    \item The largest and most consistent gain occurs for Arm Catcher Ball, where SG-JEPA (GRU) raises capture success from $9.5\%$ to $23.3\%$. The gain peaks around the training-set values and extends from $g=0$ to $g=9$. A success peak is observed at $g=0$ likely because for this gravitational field the ball is not bouncing so the trajectory is relatively easy to predict. On the other hand, for strong gravitational fields the ball bounces frequently,  making it difficult to accurately predict the dynamics. 
    
    \item The Franka Paddle-to-Basket results are less uniform across gravity values. The advantage of SG-JEPA (GRU) appears after impact. DINO-WM achieves a slightly higher paddle-hit success rate (both achieve above $95\%$ hit rates; Appendix Fig.~\ref{fig:basket-appendix}) but SG-JEPA (GRU) converts more contacts into basket entries, raising average success from $27.4\%$ to $30.5\%$. Thus the main difficulty lies in producing the post-contact position and velocity that send the ball into the basket. At low $g$ the struck ball can rise for too long and miss the basket, whereas at high $g$ a stronger and more precise impulse is needed to reach it.
    
    \item For Arm Paddle Ball, SG-JEPA (GRU) raises success from $17.7\%$ to $23.8\%$. SG-JEPA (GRU) both avoids floor contact more often and converts more retained episodes into qualifying bounces. At $g=0$, both policies enable the arm to retain the ball but not to complete a bounce because for this value of $g$ an upward-moving ball never falls back down. For high $g$, the flight time $2v_z/g$ and apex height $v_z^2/(2g)$ both shrink, leaving less time to reposition the paddle and making the $0.28\,\mathrm{m}$ apex threshold harder to reach. 
\end{itemize}

Overall, gravity exposes different bottlenecks in each task: interception timing, post-impact targeting, or repeated contact regulation. Together, they strengthen our claim that SG-JEPA's representations better preserve the dynamics and generalize better OOD  than existing baselines. Complete 3D dataset results are in Appendices~\ref{app:arm-catcher-ball-control}, \ref{app:basket-results}, and \ref{app:arm-paddle-ball-control-details}.

\begin{tcolorbox}[
  colframe=HighlightGreen,
  colback=HighlightGreenBg,
  boxrule=1.8pt,
  arc=4pt,
  left=7pt,
  right=7pt,
  top=6pt,
  bottom=4pt,
  width=\linewidth
]
 \textbf{Takeaway.} We propose \textit{Semigroup-JEPA} (SG-JEPA), a gravity-conditioned latent world model that jointly trains the encoder and predictor over recursive latent rollouts, demonstrating performance gains over Original LeWM and DINO-WM. Relative to DINO-WM, for example, it reduces position prediction error by $34\%$ on \textit{Approach Ball} and improves closed-loop capture success from $9.5\%$ to $23.3\%$ on \textit{Arm Catcher Ball} across a range of in-distribution and OOD gravity values.
\end{tcolorbox}

\section{Where Does the Long-Horizon Advantage Come From?}
\label{sec:dynamics-advantage}

SG-JEPA's long-horizon OOD advantage can arise at two stages. A difference may already be present in the one-step prediction (which we refer to as the local transition in the theoretical model below) at an unseen gravity value, before any prediction is fed back, and recursive rollout can then amplify or shrink that difference. In this section we use a linear feature model to separate these stages, which tells us what to measure to locate the advantage.

\begin{figure*}[t]
  \centering
  \includegraphics[width=\textwidth]{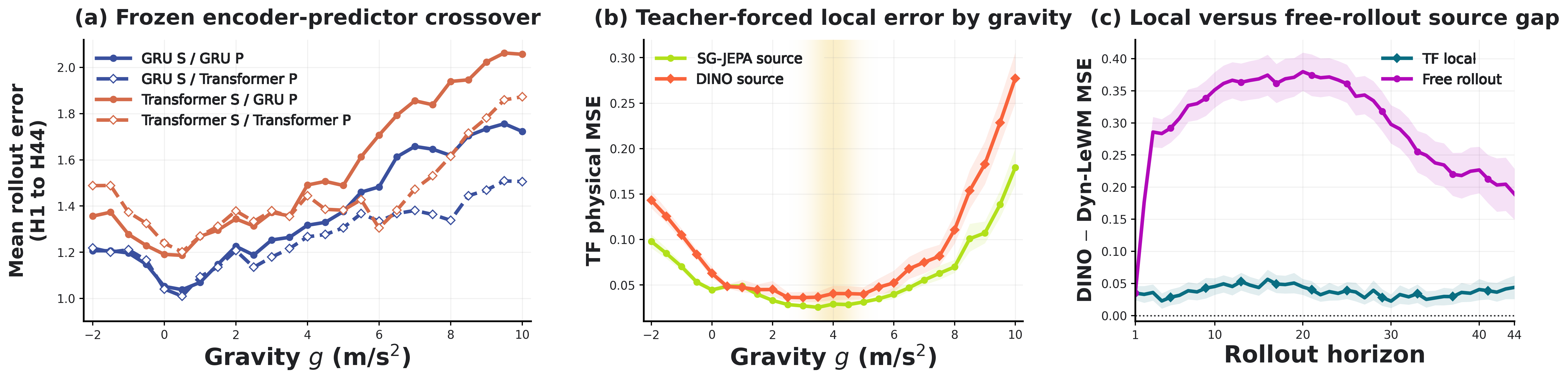}
  \caption{(a) Frozen-source predictor crossover on held-out 2D square episodes. Color denotes source; line style denotes fresh predictor. (b) Teacher-forced physical MSE from true $H=20$ context, with three predictor seeds and pointwise 95\% bootstrap intervals. (c) DINO-minus-Dyn teacher-forced and free-rollout gaps over 2,800 paired OOD episodes, with 95\% simultaneous intervals. Positive values favor SG-JEPA.}
  \label{fig:dynamics-advantage}
\end{figure*}

\paragraph{Theoretical model.} Let $\phi_t\in\mathbb R^{d_\phi}$ be a feature state with dynamics $\phi_{t+1}=T(g)\phi_t + \xi_{t+1}$, where $T(g)$ is the one-step transition at gravitational parameter $g$ and $\mathbb E[\xi_{t+1}\mid\phi_t,g]=0$. Let $z_t=W\phi_t\in\mathbb R^{d_z}$ be the learned latent. We assume $d_z\leq d_\phi$ and $W\in\mathbb R^{d_z\times d_\phi}$ is row-orthonormal, with $WW^\top=I_{d_z}$, and define $A_W(g)=WT(g)W^\top$ and $C_W(g)=WT(g)-A_W(g)W$. When a teacher-forced predictor $\widehat{A}(g)$ starts from the true latent, its conditional-mean defect is
\begin{equation}
  \delta_t(g)=C_W(g)\phi_t+[A_W(g)-\widehat A(g)]z_t.
  \label{eq:adv-local-defect}
\end{equation}
The closure term $C_W(g)\phi_t$ is the part of the next latent that still depends on features discarded by $W$. If $C_W(g)=0$, the representation is predictively closed at gravity $g$: the current latent contains all information needed for the conditional mean of the next latent. The realized teacher-forced error is $z_{t+1}-\widehat z_{t+1}^{\rm TF}=\delta_t(g)+W\xi_{t+1}$, where $\widehat z_{t+1}^{\rm TF}=\widehat A(g)z_t$. Composition at a fixed gravity and transfer across gravity values are separate questions. Suppose $T(g)=\sum_{k=1}^m\psi_k(g)T_k$ and the projected and learned operators share this finite law basis. Let $\psi(g)=(\psi_1(g),\dots,\psi_m(g))^\top$, $M_\psi=\mathbb E_{g\sim P_{\rm tr}}[\psi(g)\psi(g)^\top]\succ0$, and $\mathcal L_{\rm law}(g_\star)=\psi(g_\star)^\top M_\psi^{-1}\psi(g_\star)$, which measures how well the test gravity value $g_\star$ is covered by the training distribution.
Appendix~\ref{app:theory-law-transfer} proves that 
\begin{equation}
  \lVert\delta_t(g_\star)\rVert_2
  \leq \sqrt{\mathcal L_{\rm law}(g_\star)}
  \left(B_z\epsilon_{\rm op}+B_\phi\epsilon_{\rm cl}\right),
  \label{eq:adv-law-bound}
\end{equation}
where $B_z$ and $B_{\phi}$ bound the latent and feature states, while $\epsilon_{\rm op}$ and $\epsilon_{\rm cl}$ are the training-average predictor and closure errors. For free flight, gravity enters the transition affinely,  so $\psi(g)=(1,g)^\top$ and $\mathcal L_{\rm law}(g_\star) = 1+\frac{(g_\star-\mu_{\rm tr})^2}{\sigma_{\rm tr}^2}$. The bound therefore grows with the distance of $g_\star$ from the training gravity mean. Its main implication is that gravity conditioning alone does not guarantee transfer: unseen-gravity accuracy also depends on law coverage, predictor fit, and representation closure. 

\smallskip

\noindent \textit{From local to long-horizon error.} At fixed $g$, $S_g(\mathfrak{h})=T(g)^{\mathfrak{h}}$ and $\widehat{S}_g(\mathfrak{h})=\widehat{A}(g)^{\mathfrak{h}}$ form the true and learned discrete evolution semigroups. Starting from the same latent, the free-rollout error $e_{\mathfrak{h}}=z_{\mathfrak{h}}-\widehat z_{\mathfrak{h}}$ satisfies
\begin{equation}
  e_{\mathfrak{h}}
  = \sum_{j=0}^{\mathfrak{h}-1}
  \widehat A(g)^{\mathfrak{h}-1-j}
  \left[\delta_j(g)+W\xi_{j+1}\right].
  \label{eq:adv-error-recursion}
\end{equation}
An error introduced at step $j$ is therefore propagated through the $\mathfrak{h}-1-j$ learned updates that follow it. Importantly, one-step errors of the same size can have very different consequences after rollout, depending on how the learned updates transform them. A one-step objective is blind to this difference, whereas a multi-step objective penalizes it directly. Appendix~\ref{app:theory-local-defect} formalizes this through a semigroup-intertwining defect theorem, and Appendix~\ref{app:theory-events} treats dynamical regime changes at events such as collisions.

The theoretical model suggests two empirical tests. First, if SG-JEPA's advantage lies in the encoder's representation, it should survive replacing the original predictor with a freshly trained one, and it should already be visible under teacher forcing, before any predicted latent is fed back. Second, if feeding predictions back matters, the performance gap between encoders should differ between teacher forcing and free rollout, although the sign of that difference and how it varies with horizon are empirical questions.

\paragraph{Empirical observations.} Fig.~\ref{fig:dynamics-advantage} applies these two tests to the 2D square dataset. Panel (a) freezes 
the encoders learned jointly with the GRU and Transformer predictors, then fits a new GRU and a new Transformer predictor to each encoder. Panel (b) freezes the SG-JEPA (GRU) and DINO-WM encoders, fits a fresh GRU predictor of matched size to each, and measures teacher-forced mean-squared error (MSE) on collision-free transitions. Panel (c) compares this teacher-forced gap with the gap obtained when predictions are fed back recursively. The main takeaways are as follows.

\begin{itemize}[
    label={\textbullet},
    leftmargin=1.6em,
    labelsep=0.45em,
    topsep=0pt,
    itemsep=2pt,
    parsep=0pt,
    partopsep=0pt
]
    \item
    \textcolor{HighlightBlue}{\textbf{The advantage lies in the GRU-trained representation.}}
    Panel~\ref{fig:dynamics-advantage}(a) discards the original predictors and fits the same fresh GRU and Transformer architectures to each frozen encoder. With either predictor, the GRU-trained encoder lowers mean rollout error by about 12\% relative to the Transformer-trained encoder: $1.376$ versus $1.555$ with a fresh GRU, and $1.269$ versus $1.453$ with a fresh Transformer. That the ranking of the two encoders does not depend on which fresh predictor is fitted shows that the gain follows the representation learned during GRU training, rather than the original GRU predictor.

    \item
    \textcolor{HighlightBlue}{\textbf{SG-JEPA is better before recursive feedback.}}
    Panel~\ref{fig:dynamics-advantage}(b) tests each local transition before predicted latents are reused. For each frozen encoder, we fit a fresh GRU and restore the true $H=20$ context before every collision-free prediction. SG-JEPA has the lower local-error point estimate at all 25 test gravity values. On the far-OOD gravity values $g\leq2$ or $g\geq6$, it reduces local error by about $32\%$ relative to DINO-WM. Applying the same physical-state probe to the true and predicted next latents removes the probe error common to both models. The advantage therefore arises before rollout error accumulates.

    \item
    \textcolor{HighlightBlue}{\textbf{Recursive rollout widens the advantage.}}
    In Panel~\ref{fig:dynamics-advantage}(c), positive values favor SG-JEPA. From horizon 2 onward, the free-rollout gap is several times larger than the teacher-forced gap, peaking near $0.38$ around horizon 20, while the local gap remains below $0.06$. This is consistent with Eq.~\ref{eq:adv-error-recursion}, which says that once predictions are fed back, a small one-step difference between the encoders grows into a much larger gap. That gap shrinks again at long horizons, so the amplification is not monotone.
\end{itemize}

\paragraph{Ablations.} Ablations are presented in the appendices.
Appendix~\ref{app:gravity-cf} tests whether the predictor actually uses gravity through an evaluation-time counterfactual, in which we hold the physical trajectories fixed and feed the predictor a gravity value that differs from the true one. For both GRU and Transformer predictors, the correct gravity minimizes rollout error, while larger mismatches generally increase it. Appendix~\ref{app:sparse-gravity-posttraining} studies adaptation from sparse post-training data. We post-train on disjoint support episodes at $\mathcal G_{\mathrm{sup}}=\{0,2,6,8\}$ and evaluate on 13 unseen interpolation gravity values. We find that post-training reduces SG-JEPA (GRU) error by $13.6\%$ on the square and $20.8\%$ on the triangle. Across all eight model-shape settings, it improves on the pretrained model and outperforms post-training on the target gravity alone, averaging a $14.1\%$ reduction versus $6.9\%$.

\begin{tcolorbox}[
  colframe=HighlightGreen,
  colback=HighlightGreenBg,
  boxrule=1.8pt,
  arc=4pt,
  left=7pt,
  right=7pt,
  top=6pt,
  bottom=4pt,
  width=\linewidth
]
\textbf{Takeaway.} SG-JEPA's OOD advantage lies in the representation the encoder learns during GRU training. We know this because the advantage survives replacing the predictor and is already present before any feedback, where it lowers far-OOD one-step error by $32\%$. The linear feature model explains why such a small one-step edge matters. Gravity conditioning alone does not guarantee transfer, and once predictions are fed back, a small one-step difference between encoders grows into a much larger long-horizon gap, which is what we see under recursive rollout.
\end{tcolorbox}

\section{Conclusion}

In this work, we introduce Semigroup-JEPA (SG-JEPA), which combines gravity conditioning with joint encoder--predictor training through recursive latent rollouts. To test whether a world model learns physical dynamics rather than a single transition law, we build eight MuJoCo datasets in which gravity is sampled from a narrow band at training time and from a much wider grid at test time. On 2D and 3D prediction and robot control, SG-JEPA outperforms Original LeWM and DINO-WM both at longer rollout horizons and at OOD gravity values, with the largest gains at long horizons and far from the training gravity. The same frozen representation transfers to closed-loop control, where a Diffusion Policy trained on SG-JEPA features raises success rates on all three manipulation tasks. To explain these gains, we develop a linear feature model that separates local law-conditioned error from its amplification under rollout and bounds the one-step error at an unseen gravity value by how well the training gravities cover it. Guided by this model, we train fresh predictors on frozen checkpoints and find that SG-JEPA's advantage lies in the encoder's representation, and that recursive rollout amplifies this local advantage over the horizon. Ablations further show that the predictor actually uses the supplied gravity, and that sparse post-training at a few new gravity values improves accuracy at the unseen values between them.

\paragraph{Limitations.} Several important directions remain for future work. First, our experiments vary a single scalar, gravity. It would be of interest to study a setting with vector-valued physical variables, or even to have the model infer dynamical parameters directly from observation. Second, transfer across object shapes is uneven. Training on 2D triangle and square datasets transfers much of the house's translational dynamics but not its rotation, and more complex shapes such as the pentagon remain challenging for the current setup, which may require a generalist model trained on diverse data. Finally, our theory uses a linear feature model, whereas the neural predictor is nonlinear and history dependent, and contacts can perturb transition branches. Understanding the behavior of JEPA-type models calls for new theoretical models that incorporate nonlinearities as well as action-perturbing effects.

\bibliographystyle{unsrtnat}
\bibliography{references}

@inproceedings{zhou2025dinowm,
  author    = {Zhou, Gaoyue and Pan, Hengkai and LeCun, Yann and Pinto, Lerrel},
  title     = {{DINO-WM}: World Models on Pre-trained Visual Features Enable Zero-shot Planning},
  booktitle = {Proceedings of the 42nd International Conference on Machine Learning},
  series    = {Proceedings of Machine Learning Research},
  volume    = {267},
  pages     = {79115--79135},
  publisher = {PMLR},
  year      = {2025}
}

@article{balestriero2025lejepa,
  author  = {Balestriero, Randall and LeCun, Yann},
  title   = {{LeJEPA}: Provable and Scalable Self-Supervised Learning Without the Heuristics},
  journal = {arXiv preprint arXiv:2511.08544},
  year    = {2025},
}

@article{maes2026lewm,
  author  = {Maes, Lucas and Le Lidec, Quentin and Scieur, Damien and LeCun, Yann and Balestriero, Randall},
  title   = {{LeWorldModel}: Stable End-to-End Joint-Embedding Predictive Architecture from Pixels},
  journal = {arXiv preprint arXiv:2603.19312},
  year    = {2026},
}

@article{xue2026acwmphys,
  author  = {Xue, Haotian and Chen, Yipu and Ma, Liqian and Zhao, Zelin and Moukheiber, Lama and Zhu, Yuchen and Chen, Yongxin},
  title   = {{ACWM-Phys}: Investigating Generalized Physical Interaction in Action-Conditioned Video World Models},
  journal = {arXiv preprint arXiv:2605.08567},
  year    = {2026},
}

@article{tan2026physicalidentifiability,
  author  = {Tan, Kaizhen and Xu, Xin and Tao, Siru and Hong, Hanzhe and Feng, Yang and Du, Heqing},
  title   = {What Can Latent World Models Know? Physical Parameter Identifiability in Multimodal Predictive Representations},
  journal = {arXiv preprint arXiv:2607.27017},
  year    = {2026},
}

@inproceedings{chi2023diffusionpolicy,
  author    = {Chi, Cheng and Feng, Siyuan and Du, Yilun and Xu, Zhenjia and Cousineau, Eric and Burchfiel, Benjamin and Song, Shuran},
  title     = {Diffusion Policy: Visuomotor Policy Learning via Action Diffusion},
  booktitle = {Proceedings of Robotics: Science and Systems},
  year      = {2023}
}

@inproceedings{todorov2012mujoco,
  author    = {Todorov, Emanuel and Erez, Tom and Tassa, Yuval},
  title     = {{MuJoCo}: A Physics Engine for Model-Based Control},
  booktitle = {2012 IEEE/RSJ International Conference on Intelligent Robots and Systems},
  pages     = {5026--5033},
  year      = {2012},
  doi       = {10.1109/IROS.2012.6386109}
}

@misc{lecun2022path,
  author       = {LeCun, Yann},
  title        = {A Path Towards Autonomous Machine Intelligence},
  howpublished = {OpenReview},
  year         = {2022},
  note         = {Version 0.9.2},
}

@inproceedings{assran2023ijepa,
  author    = {Assran, Mahmoud and Duval, Quentin and Misra, Ishan and Bojanowski, Piotr and Vincent, Pascal and Rabbat, Michael and LeCun, Yann and Ballas, Nicolas},
  title     = {Self-Supervised Learning from Images with a Joint-Embedding Predictive Architecture},
  booktitle = {Proceedings of the IEEE/CVF Conference on Computer Vision and Pattern Recognition},
  pages     = {15619--15629},
  year      = {2023},
  doi       = {10.1109/CVPR52729.2023.01499}
}

@article{bardes2024vjepa,
  author  = {Bardes, Adrien and Garrido, Quentin and Ponce, Jean and Chen, Xinlei and Rabbat, Michael and LeCun, Yann and Assran, Mido and Ballas, Nicolas},
  title   = {Revisiting Feature Prediction for Learning Visual Representations from Video},
  journal = {Transactions on Machine Learning Research},
  year    = {2024}
}

@article{assran2025vjepa2,
  author  = {Assran, Mahmoud and Bardes, Adrien and Fan, David and Garrido, Quentin and Howes, Russell and Komeili, Mojtaba and Muckley, Matthew and Rizvi, Ammar and Roberts, Claire and Sinha, Koustuv and Zholus, Artem and Arnaud, Sergio and Gejji, Abha and Martin, Ada and Robert Hogan, Francois and Dugas, Daniel and Bojanowski, Piotr and Khalidov, Vasil and Labatut, Patrick and Massa, Francisco and Szafraniec, Marc and Krishnakumar, Kapil and Li, Yong and Ma, Xiaodong and Chandar, Sarath and Meier, Franziska and LeCun, Yann and Rabbat, Michael and Ballas, Nicolas},
  title   = {{V-JEPA} 2: Self-Supervised Video Models Enable Understanding, Prediction and Planning},
  journal = {arXiv preprint arXiv:2506.09985},
  year    = {2025},
}

@article{oquab2024dinov2,
  author  = {Oquab, Maxime and Darcet, Timoth{\'e}e and Moutakanni, Th{\'e}o and Vo, Huy V. and Szafraniec, Marc and Khalidov, Vasil and Fernandez, Pierre and Haziza, Daniel and Massa, Francisco and El-Nouby, Alaaeldin and Assran, Mido and Ballas, Nicolas and Galuba, Wojciech and Howes, Russell and Huang, Po-Yao and Li, Shang-Wen and Misra, Ishan and Rabbat, Michael and Sharma, Vasu and Synnaeve, Gabriel and Xu, Hu and J{\'e}gou, Herv{\'e} and Mairal, Julien and Labatut, Patrick and Joulin, Armand and Bojanowski, Piotr},
  title   = {{DINOv2}: Learning Robust Visual Features Without Supervision},
  journal = {Transactions on Machine Learning Research},
  year    = {2024}
}

@inproceedings{ha2018worldmodels,
  author    = {Ha, David and Schmidhuber, J{\"u}rgen},
  title     = {Recurrent World Models Facilitate Policy Evolution},
  booktitle = {Advances in Neural Information Processing Systems},
  volume    = {31},
  year      = {2018}
}

@inproceedings{hafner2019planet,
  author    = {Hafner, Danijar and Lillicrap, Timothy and Fischer, Ian and Villegas, Ruben and Ha, David and Lee, Honglak and Davidson, James},
  title     = {Learning Latent Dynamics for Planning from Pixels},
  booktitle = {Proceedings of the 36th International Conference on Machine Learning},
  series    = {Proceedings of Machine Learning Research},
  volume    = {97},
  pages     = {2555--2565},
  publisher = {PMLR},
  year      = {2019}
}

@article{hafner2023dreamerv3,
  author  = {Hafner, Danijar and Pasukonis, Jurgis and Ba, Jimmy and Lillicrap, Timothy},
  title   = {Mastering Diverse Control Tasks Through World Models},
  journal = {Nature},
  volume  = {640},
  number  = {8059},
  pages   = {647--653},
  year    = {2025},
  doi     = {10.1038/s41586-025-08744-2}
}

@inproceedings{hansen2024tdmpc2,
  author    = {Hansen, Nicklas and Su, Hao and Wang, Xiaolong},
  title     = {{TD-MPC2}: Scalable, Robust World Models for Continuous Control},
  booktitle = {International Conference on Learning Representations},
  year      = {2024}
}

@article{riochet2018intphys,
  author  = {Riochet, Ronan and Ynocente Castro, Mario and Bernard, Mathieu and Lerer, Adam and Fergus, Rob and Izard, V{\'e}ronique and Dupoux, Emmanuel},
  title   = {{IntPhys} 2019: A Benchmark for Visual Intuitive Physics Understanding},
  journal = {IEEE Transactions on Pattern Analysis and Machine Intelligence},
  volume  = {44},
  number  = {9},
  pages   = {5016--5025},
  year    = {2022},
  doi     = {10.1109/TPAMI.2021.3083839}
}

@inproceedings{bear2021physion,
  author    = {Bear, Daniel and Wang, Elias and Mrowca, Damian and Binder, Felix and Tung, Hsiao-Yu and Pramod, R. T. and Holdaway, Cameron and Tao, Sirui and Smith, Kevin and Sun, Fan-Yun and Li, Fei-Fei and Kanwisher, Nancy and Tenenbaum, Joshua and Yamins, Daniel and Fan, Judith},
  title     = {{Physion}: Evaluating Physical Prediction from Vision in Humans and Machines},
  booktitle = {Proceedings of the Neural Information Processing Systems Track on Datasets and Benchmarks},
  volume    = {1},
  year      = {2021}
}

@article{garrido2025intuitivephysics,
  author  = {Garrido, Quentin and Ballas, Nicolas and Assran, Mahmoud and Bardes, Adrien and Najman, Laurent and Rabbat, Michael and Dupoux, Emmanuel and LeCun, Yann},
  title   = {Intuitive Physics Understanding Emerges from Self-Supervised Pretraining on Natural Videos},
  journal = {arXiv preprint arXiv:2502.11831},
  year    = {2025},
}

@inproceedings{sanchez2020learningtosimulate,
  author    = {Sanchez-Gonzalez, Alvaro and Godwin, Jonathan and Pfaff, Tobias and Ying, Rex and Leskovec, Jure and Battaglia, Peter W.},
  title     = {Learning to Simulate Complex Physics with Graph Networks},
  booktitle = {Proceedings of the 37th International Conference on Machine Learning},
  series    = {Proceedings of Machine Learning Research},
  volume    = {119},
  pages     = {8459--8468},
  publisher = {PMLR},
  year      = {2020}
}

@inproceedings{tobin2017domainrandomization,
  author    = {Tobin, Josh and Fong, Rachel and Ray, Alex and Schneider, Jonas and Zaremba, Wojciech and Abbeel, Pieter},
  title     = {Domain Randomization for Transferring Deep Neural Networks from Simulation to the Real World},
  booktitle = {2017 IEEE/RSJ International Conference on Intelligent Robots and Systems},
  pages     = {23--30},
  year      = {2017},
  doi       = {10.1109/IROS.2017.8202133}
}

@inproceedings{yu2017universalpolicy,
  author    = {Yu, Wenhao and Tan, Jie and Liu, C. Karen and Turk, Greg},
  title     = {Preparing for the Unknown: Learning a Universal Policy with Online System Identification},
  booktitle = {Proceedings of Robotics: Science and Systems},
  year      = {2017},
  doi       = {10.15607/RSS.2017.XIII.048}
}

@inproceedings{kumar2021rma,
  author    = {Kumar, Ashish and Fu, Zipeng and Pathak, Deepak and Malik, Jitendra},
  title     = {{RMA}: Rapid Motor Adaptation for Legged Robots},
  booktitle = {Proceedings of Robotics: Science and Systems},
  year      = {2021},
  doi       = {10.15607/RSS.2021.XVII.011}
}

@inproceedings{bengio2015scheduledsampling,
  author    = {Bengio, Samy and Vinyals, Oriol and Jaitly, Navdeep and Shazeer, Noam},
  title     = {Scheduled Sampling for Sequence Prediction with Recurrent Neural Networks},
  booktitle = {Advances in Neural Information Processing Systems},
  volume    = {28},
  year      = {2015}
}

@inproceedings{talvitie2014hallucinated,
  author    = {Talvitie, Erik},
  title     = {Model Regularization for Stable Sample Rollouts},
  booktitle = {Proceedings of the Thirtieth Conference on Uncertainty in Artificial Intelligence},
  pages     = {780--789},
  year      = {2014}
}

@inproceedings{venkatraman2015dad,
  author    = {Venkatraman, Arun and Hebert, Martial and Bagnell, J. Andrew},
  title     = {Improving Multi-step Prediction of Learned Time Series Models},
  booktitle = {Proceedings of the Twenty-Ninth AAAI Conference on Artificial Intelligence},
  pages     = {3024--3030},
  year      = {2015},
  doi       = {10.1609/aaai.v29i1.9590}
}

@inproceedings{lambert2021longterm,
  author    = {Lambert, Nathan O. and Wilcox, Albert and Zhang, Howard and Pister, Kristofer S. J. and Calandra, Roberto},
  title     = {Learning Accurate Long-term Dynamics for Model-based Reinforcement Learning},
  booktitle = {2021 60th IEEE Conference on Decision and Control},
  pages     = {2880--2887},
  year      = {2021},
  doi       = {10.1109/CDC45484.2021.9683134}
}

@inproceedings{lambert2020objectivemismatch,
  author    = {Lambert, Nathan and Amos, Brandon and Yadan, Omry and Calandra, Roberto},
  title     = {Objective Mismatch in Model-based Reinforcement Learning},
  booktitle = {Proceedings of the 2nd Conference on Learning for Dynamics and Control},
  series    = {Proceedings of Machine Learning Research},
  volume    = {120},
  pages     = {761--770},
  publisher = {PMLR},
  year      = {2020}
}

@inproceedings{grimm2020valueequivalence,
  author    = {Grimm, Christopher and Barreto, Andr{\'e} and Singh, Satinder and Silver, David},
  title     = {The Value Equivalence Principle for Model-Based Reinforcement Learning},
  booktitle = {Advances in Neural Information Processing Systems},
  volume    = {33},
  year      = {2020}
}

@article{schrittwieser2020muzero,
  author  = {Schrittwieser, Julian and Antonoglou, Ioannis and Hubert, Thomas and Simonyan, Karen and Sifre, Laurent and Schmitt, Simon and Guez, Arthur and Lockhart, Edward and Hassabis, Demis and Graepel, Thore and Lillicrap, Timothy and Silver, David},
  title   = {Mastering {Atari}, {Go}, Chess and Shogi by Planning with a Learned Model},
  journal = {Nature},
  volume  = {588},
  pages   = {604--609},
  year    = {2020},
  doi     = {10.1038/s41586-020-03051-4}
}

@inproceedings{nair2022r3m,
  author    = {Nair, Suraj and Rajeswaran, Aravind and Kumar, Vikash and Finn, Chelsea and Gupta, Abhinav},
  title     = {{R3M}: A Universal Visual Representation for Robot Manipulation},
  booktitle = {Conference on Robot Learning},
  series    = {Proceedings of Machine Learning Research},
  volume    = {205},
  pages     = {892--909},
  year      = {2022}
}

@inproceedings{majumdar2023vc1,
  author    = {Majumdar, Arjun and Yadav, Karmesh and Arnaud, Sergio and Ma, Jason and Chen, Claire and Silwal, Sneha and Jain, Aryan and Berges, Vincent-Pierre and Wu, Tingfan and Vakil, Jay and Abbeel, Pieter and Malik, Jitendra and Batra, Dhruv and Lin, Yixin and Maksymets, Oleksandr and Rajeswaran, Aravind and Meier, Franziska},
  title     = {Where Are We in the Search for an Artificial Visual Cortex for Embodied Intelligence?},
  booktitle = {Advances in Neural Information Processing Systems},
  volume    = {36},
  year      = {2023},
  doi       = {10.52202/075280-0031}
}

@inproceedings{cho2014gru,
  author    = {Cho, Kyunghyun and van Merri{\"e}nboer, Bart and Gulcehre, Caglar and Bahdanau, Dzmitry and Bougares, Fethi and Schwenk, Holger and Bengio, Yoshua},
  title     = {Learning Phrase Representations using {RNN} Encoder--Decoder for Statistical Machine Translation},
  booktitle = {Proceedings of the 2014 Conference on Empirical Methods in Natural Language Processing},
  pages     = {1724--1734},
  publisher = {Association for Computational Linguistics},
  year      = {2014},
  doi       = {10.3115/v1/D14-1179}
}

@inproceedings{vaswani2017attention,
  author    = {Vaswani, Ashish and Shazeer, Noam and Parmar, Niki and Uszkoreit, Jakob and Jones, Llion and Gomez, Aidan N. and Kaiser, Lukasz and Polosukhin, Illia},
  title     = {Attention Is All You Need},
  booktitle = {Advances in Neural Information Processing Systems},
  volume    = {30},
  year      = {2017}
}

@inproceedings{gu2023mamba,
  author    = {Gu, Albert and Dao, Tri},
  title     = {{Mamba}: Linear-Time Sequence Modeling with Selective State Spaces},
  booktitle = {First Conference on Language Modeling},
  year      = {2024}
}

@article{williams2015edmd,
  author  = {Williams, Matthew O. and Kevrekidis, Ioannis G. and Rowley, Clarence W.},
  title   = {A Data-Driven Approximation of the {Koopman} Operator: Extending Dynamic Mode Decomposition},
  journal = {Journal of Nonlinear Science},
  volume  = {25},
  number  = {6},
  pages   = {1307--1346},
  year    = {2015},
  doi     = {10.1007/s00332-015-9258-5}
}

@article{lusch2018deepkoopman,
  author  = {Lusch, Bethany and Kutz, J. Nathan and Brunton, Steven L.},
  title   = {Deep Learning for Universal Linear Embeddings of Nonlinear Dynamics},
  journal = {Nature Communications},
  volume  = {9},
  pages   = {4950},
  year    = {2018},
  doi     = {10.1038/s41467-018-07210-0}
}

@article{heemels2001equivalence,
  author  = {Heemels, W. P. M. H. and De Schutter, Bart and Bemporad, Alberto},
  title   = {Equivalence of Hybrid Dynamical Models},
  journal = {Automatica},
  volume  = {37},
  number  = {7},
  pages   = {1085--1091},
  year    = {2001},
  doi     = {10.1016/S0005-1098(01)00059-0}
}

@inproceedings{sobal2025jepa,
  author    = {Sobal, Uladzislau and Zhang, Wancong and Cho, Kyunghyun and Balestriero, Randall and Rudner, Tim G. J. and LeCun, Yann},
  title     = {Learning from Reward-Free Offline Data: A Case for Planning with Latent Dynamics Models},
  booktitle = {Advances in Neural Information Processing Systems},
  volume    = {38},
  year      = {2025},
  doi       = {10.52202/085713-1465}
}

@inproceedings{killian2017hiddenparameter,
  author    = {Killian, Taylor W. and Daulton, Samuel and Konidaris, George and Doshi-Velez, Finale},
  title     = {Robust and Efficient Transfer Learning with Hidden Parameter Markov Decision Processes},
  booktitle = {Advances in Neural Information Processing Systems},
  volume    = {30},
  pages     = {6250--6261},
  year      = {2017}
}

@inproceedings{lee2020cadm,
  author    = {Lee, Kimin and Seo, Younggyo and Lee, Seunghyun and Lee, Honglak and Shin, Jinwoo},
  title     = {Context-aware Dynamics Model for Generalization in Model-Based Reinforcement Learning},
  booktitle = {Proceedings of the 37th International Conference on Machine Learning},
  series    = {Proceedings of Machine Learning Research},
  volume    = {119},
  pages     = {5757--5766},
  publisher = {PMLR},
  year      = {2020}
}

@inproceedings{kirchmeyer2022coda,
  author    = {Kirchmeyer, Matthieu and Yin, Yuan and Dona, Jeremie and Baskiotis, Nicolas and Rakotomamonjy, Alain and Gallinari, Patrick},
  title     = {Generalizing to New Physical Systems via Context-Informed Dynamics Model},
  booktitle = {Proceedings of the 39th International Conference on Machine Learning},
  series    = {Proceedings of Machine Learning Research},
  volume    = {162},
  pages     = {11283--11301},
  publisher = {PMLR},
  year      = {2022}
}

@article{wang2024protocad,
  author  = {Wang, Junjie and Zhang, Qichao and Mu, Yao and Li, Dong and Zhao, Dongbin and Zhuang, Yuzheng and Luo, Ping and Wang, Bin and Hao, Jianye},
  title   = {Prototypical Context-Aware Dynamics for Generalization in Visual Control With Model-Based Reinforcement Learning},
  journal = {IEEE Transactions on Industrial Informatics},
  volume  = {20},
  number  = {9},
  pages   = {10717--10727},
  year    = {2024},
  doi     = {10.1109/TII.2024.3396525}
}

@inproceedings{ball2021augwm,
  author    = {Ball, Philip J. and Lu, Cong and Parker-Holder, Jack and Roberts, Stephen},
  title     = {Augmented World Models Facilitate Zero-Shot Dynamics Generalization From a Single Offline Environment},
  booktitle = {Proceedings of the 38th International Conference on Machine Learning},
  series    = {Proceedings of Machine Learning Research},
  volume    = {139},
  pages     = {619--629},
  publisher = {PMLR},
  year      = {2021}
}

@inproceedings{lanier2026redraw,
  author    = {Lanier, JB and Kim, Kyungmin and Karamzade, Armin and Liu, Yifei and Sinha, Ankita and He, Kat and Corsi, Davide and Fox, Roy},
  title     = {Adapting World Models with Latent-State Dynamics Residuals},
  booktitle = {Proceedings of the 8th Annual Learning for Dynamics and Control Conference},
  series    = {Proceedings of Machine Learning Research},
  volume    = {331},
  pages     = {117--144},
  publisher = {PMLR},
  year      = {2026}
}

@inproceedings{li2020visualgrounding,
  author    = {Li, Yunzhu and Lin, Toru and Yi, Kexin and Bear, Daniel and Yamins, Daniel and Wu, Jiajun and Tenenbaum, Joshua and Torralba, Antonio},
  title     = {Visual Grounding of Learned Physical Models},
  booktitle = {Proceedings of the 37th International Conference on Machine Learning},
  series    = {Proceedings of Machine Learning Research},
  volume    = {119},
  pages     = {5927--5936},
  publisher = {PMLR},
  year      = {2020}
}

@inproceedings{motamed2026physicsiq,
  author    = {Motamed, Saman and Culp, Laura and Swersky, Kevin and Jaini, Priyank and Geirhos, Robert},
  title     = {Do Generative Video Models Understand Physical Principles?},
  booktitle = {Proceedings of the IEEE/CVF Winter Conference on Applications of Computer Vision},
  pages     = {948--958},
  year      = {2026},
  doi       = {10.1109/WACV61042.2026.00099}
}

@article{wang2026gauge,
  author  = {Wang, Shuai and Feng, Yaxin and Jiang, Xuekun and Tian, Shihan and Yan, Ningyu and Shen, Xing and Lyu, Chaoyang and Wang, Hui and Zhou, Yunsong and Wang, Hanqing and Pang, Jiangmiao and Xiang, Yang and Gao, Xing and Shen, Chunhua and Zhang, Weinan},
  title   = {{GAUGE}: A Measurement-Grounded Benchmark for Physical Fidelity in Simulation Engines and Video World Models},
  journal = {arXiv preprint arXiv:2608.05948},
  year    = {2026},
}

@inproceedings{peebles2023scalablediffusionmodelstransformers,
  author    = {Peebles, William and Xie, Saining},
  title     = {Scalable Diffusion Models with Transformers},
  booktitle = {Proceedings of the IEEE/CVF International Conference on Computer Vision},
  pages     = {4195--4205},
  year      = {2023},
  doi       = {10.1109/ICCV51070.2023.00387}
}

@inproceedings{dosovitskiy2021imageworth16x16words,
  author    = {Dosovitskiy, Alexey and Beyer, Lucas and Kolesnikov, Alexander and Weissenborn, Dirk and Zhai, Xiaohua and Unterthiner, Thomas and Dehghani, Mostafa and Minderer, Matthias and Heigold, Georg and Gelly, Sylvain and Uszkoreit, Jakob and Houlsby, Neil},
  title     = {An Image Is Worth 16x16 Words: Transformers for Image Recognition at Scale},
  booktitle = {International Conference on Learning Representations},
  year      = {2021}
}

@misc{jordan2024muon,
  author       = {Jordan, Keller and Jin, Yuchen and Boza, Vlado and You, Jiacheng and Cesista, Franz and Newhouse, Laker and Bernstein, Jeremy},
  title        = {Muon: An Optimizer for Hidden Layers in Neural Networks},
  howpublished = {Online},
  year         = {2024},
  note         = {Blog post},
}

@article{liu2025muonscalablellmtraining,
  author  = {Liu, Jingyuan and Su, Jianlin and Yao, Xingcheng and Jiang, Zhejun and Lai, Guokun and Du, Yulun and Qin, Yidao and Xu, Weixin and Lu, Enzhe and Yan, Junjie and Chen, Yanru and Zheng, Huabin and Liu, Yibo and Liu, Shaowei and Yin, Bohong and He, Weiran and Zhu, Han and Wang, Yuzhi and Wang, Jianzhou and Dong, Mengnan and Zhang, Zheng and Kang, Yongsheng and Zhang, Hao and Xu, Xinran and Zhang, Yutao and Wu, Yuxin and Zhou, Xinyu and Yang, Zhilin},
  title   = {Muon Is Scalable for {LLM} Training},
  journal = {arXiv preprint arXiv:2502.16982},
  year    = {2025},
}

@inproceedings{adamw,
  author    = {Loshchilov, Ilya and Hutter, Frank},
  title     = {Decoupled Weight Decay Regularization},
  booktitle = {International Conference on Learning Representations},
  year      = {2019}
}

\clearpage
\appendix

\newif\ifappendixfloatbarriers
\appendixfloatbarrierstrue
\pretocmd{\section}{\ifappendixfloatbarriers\FloatBarrier\fi}{}{}

\setcounter{topnumber}{4}
\setcounter{bottomnumber}{2}
\setcounter{totalnumber}{6}
\renewcommand{\topfraction}{0.90}
\renewcommand{\bottomfraction}{0.75}
\renewcommand{\textfraction}{0.08}
\renewcommand{\floatpagefraction}{0.72}

\starttocentries
\tableofcontents

\section{Additional related work}
\label{app:related}

This appendix expands Section~\ref{sec:related}.

\paragraph{Latent world models and physical prediction.}
Joint-embedding predictive architectures learn to predict features rather than
pixels \citep{lecun2022path,assran2023ijepa,bardes2024vjepa,assran2025vjepa2}.
DINO-WM fits dynamics over a frozen DINOv2 encoder
\citep{zhou2025dinowm,oquab2024dinov2}, while LeWM trains its encoder and
predictor jointly with the SIGReg regularizer from LeJEPA
\citep{balestriero2025lejepa,maes2026lewm}.  Such latent models can also plan
from reward-free offline data \citep{sobal2025jepa}; reconstruction-based world
models retain a pixel or observation decoder
\citep{ha2018worldmodels,hafner2019planet,hafner2023dreamerv3,hansen2024tdmpc2}.
Physical benchmarks provide a complementary test.  IntPhys and Physion evaluate
physical events or human judgments \citep{riochet2018intphys,bear2021physion},
whereas Physics-IQ and GAUGE measure trajectories and physical quantities
\citep{motamed2026physicsiq,wang2026gauge}.  Other work studies intuitive physics
in self-supervised video representations \citep{garrido2025intuitivephysics},
structured particle simulators \citep{sanchez2020learningtosimulate}, visual
inference of physical properties \citep{li2020visualgrounding}, and prediction
under held-out physical configurations
\citep{xue2026acwmphys,tan2026physicalidentifiability}.

\paragraph{Generalization across physical systems.}
Hidden-parameter MDPs represent related environments with a low-dimensional
dynamics variable \citep{killian2017hiddenparameter}.  CaDM and CoDA infer this
context from trajectories \citep{lee2020cadm,kirchmeyer2022coda}, and ProtoCAD
extends context-conditioned dynamics to high-dimensional visual observations
\citep{wang2024protocad}.  Universal policies, rapid motor adaptation, and
augmented world models also infer or adapt to changing dynamics
\citep{yu2017universalpolicy,kumar2021rma,ball2021augwm}; residual latent
dynamics provide a related approach to sim-to-real adaptation
\citep{lanier2026redraw}.  Domain randomization instead trains across parameter
variations without exposing the parameter to the policy
\citep{tobin2017domainrandomization}.  SG-JEPA removes system identification
from the experiment: the scene and control interface remain fixed, gravity is
given to the model, and only that scalar varies.  The evaluation therefore
isolates interpolation and extrapolation along a known physical coordinate.

\paragraph{Rollout training and downstream use.}
A one-step model accumulates error when its predictions become future inputs
\citep{bengio2015scheduledsampling}.  Existing remedies feed predictions back
during training \citep{talvitie2014hallucinated}, correct rollout states with
training data \citep{venkatraman2015dad}, supervise several latent horizons
\citep{hafner2019planet}, or optimize long-horizon accuracy directly
\citep{lambert2021longterm}.  SG-JEPA backpropagates its rollout loss through
both predictor and encoder, and the frozen-encoder crossover separates their
contributions across recurrent, state-space, and attention predictors
\citep{cho2014gru,vaswani2017attention,gu2023mamba}.  Because model loss need not
track control performance \citep{lambert2020objectivemismatch}, related work has
used value-equivalent objectives \citep{grimm2020valueequivalence,
schrittwieser2020muzero} and frozen representations for imitation learning
\citep{nair2022r3m,majumdar2023vc1,chi2023diffusionpolicy}.  We likewise report
physical prediction and downstream policy performance, while keeping the data
and policy training fixed within each controlled comparison.  The accompanying
analysis uses lifted linear dynamics and a piecewise-affine contact model
\citep{williams2015edmd,lusch2018deepkoopman,heemels2001equivalence}; these are
surrogate analyses, not claims that the learned predictors are linear.

\section{Datasets}
\label{app:dataset-generation}

We evaluate eight datasets generated with MuJoCo~\citep{todorov2012mujoco}:
four planar rigid-body prediction tasks, one 3D projectile prediction task, and
three 3D robotic control tasks.  Table~\ref{tab:dataset-summary} summarizes
their observation, action, and state records.

\paragraph{Planar rigid-body prediction.}
The four 2D datasets contain one red rigid body, a right triangle, square, pentagon, or "house" (concatenation of the right triangle and the square along the leg), moving and colliding inside a planar box.  We randomize its initial pose and apply a single planar impulse.  The eight-dimensional state is $[x,z,v_x,v_z,\vartheta,\omega,x_{\mathrm{anchor}},z_{\mathrm{anchor}}]$.
The physics record also stores the object mass and size, shape identifier, and
box dimensions.

\paragraph{3D prediction and control.}
\textit{Approach Ball} isolates projectile prediction: a ball is launched with randomized initial velocity and travels freely without actions. \textit{Arm Catcher Ball} and \textit{Arm Paddle Ball} use a Unitree Z1 arm.  The catcher task moves a catcher to intercept the bouncing ball, while the paddle task controls position and tilt to sustain repeated bounces.  Their per-frame records include the commanded and realized catcher or paddle state, contact flags, arm joint positions and velocities, and task success. \textit{Franka Paddle-to-Basket} instead uses a Franka arm with a blue paddle attached to strike an incoming ball toward a basket. Success requires a valid blade contact followed by the ball entering the basket. In addition to the ball, paddle, contact, physics, and arm records, this dataset stores basket geometry, task events, a strike diagnostic, and a failure label. The scripted outcome cycle contains 70\% successes and 5\% from each of six
informative failure modes.

\begin{table*}[!htbp]
    \centering
    \small
    \setlength{\tabcolsep}{4.2pt}
    \renewcommand{\arraystretch}{1.12}
    \caption{
        Dataset summary. Each episode consists of 64 frames at 16\,Hz.
        State records contain the task-relevant object and robot states used for
        supervision and evaluation.
    }
    \label{tab:dataset-summary}
    \resizebox{\textwidth}{!}{%
    \begin{tabular}{lllcl}
        \toprule
        Dataset & Role & RGB resolution & Action record & Principal state record \\
        \midrule
        Right Triangle
            & Prediction & $128^2$
            & $(J_x, J_z, g)$
            & Rigid body (8) \\
        Square
            & Prediction & $128^2$
            & $(J_x, J_z, g)$
            & Rigid body (8) \\
        Pentagon
            & Prediction & $128^2$
            & $(J_x, J_z, g)$
            & Rigid body (8) \\
        House
            & Prediction & $128^2$
            & $(J_x, J_z, g)$
            & Rigid body (8) \\
        \midrule
        Approach Ball
            & Prediction & $256^2$
            & $g$
            & Ball (16) \\
        Arm Catcher Ball
            & Control & $256^2$
            & $(g, \Delta x, \Delta y, \Delta z)$
            & Ball (16), catcher (9), Z1 arm (18) \\
        Arm Paddle Ball
            & Control & $256^2$
            & $(g, \Delta x, \Delta y, \Delta z, \phi, \theta)$
            & Ball (16), paddle (9), Z1 arm (18) \\
        Franka Paddle-to-Basket
            & Control & $256^2$
            & $(g, \Delta x, \Delta y, \Delta z, \phi, \theta)$
            & Ball (16), paddle (9), Panda arm (21), basket (6) \\
        \bottomrule
    \end{tabular}%
    }
\end{table*}

\paragraph{Gravity sampling and input encoding.}
The gravity $g$ is constant within each episode.  For the 2D planar tasks, training gravity follows $g=\max(\mathcal{N}(4,0.5^2),0.1)$ with 8000 episodes each, and the held-out test split uses the 25-point grid $\{-2,-1.5,\ldots,10\}$ with 200 episodes per value. For 3D datasets, \textit{Approach Ball}, \textit{Arm Paddle Ball} and \textit{Franka Paddle-to-Basket} again samples 8000 episodes from $g=\max(\mathcal{N}(9.8,2^2),0)$ for training; and 25-value held-out test set $\{0,1,\ldots,20,0.62,1.63,3.72,8.87\}$, with the last 4 values being the gravity values in Pluto, Moon, Mars and Venus respectively. \textit{Arm Catcher Ball} uses $g=\max(\mathcal{N}(4,0.5^2),0.1)$ and the 23-value grid $\{-1,-0.5,\ldots,10\}$.  

Gravity is stored in physical units in the physics record and as an action coordinate.  The planar action is $(J_x,J_z,g)$, where the two impulse channels are nonzero only at $t=0$.  \textit{Approach Ball} receives only $g$.  The arm tasks append Cartesian catcher or paddle commands from a scripted expert, and \textit{Arm Paddle Ball} and \textit{Franka Paddle-to-Basket} append two paddle-orientation commands. Before a learned component receives gravity, we z-score it as $(g-\mu_{g,\mathrm{train}})/\sigma_{g,\mathrm{train}}$.  We compute these statistics only on the corresponding training split and keep them fixed for validation, test, and post-training evaluation.  Evaluation data never enter the normalization statistics.

\paragraph{Format and storage.}
Each episode contains 64 frames at 16\,Hz over 4 seconds.  The 2D videos are $128\!\times\!128$, and the 3D videos are $256\!\times\!256$.  Action $a_t$ is aligned with the transition from frame $t$ to frame $t+1$.  The Lance datasets contain one row per frame, including episode and step indices, split
identifier, JPEG-encoded RGB, simulator state, aligned action, reward, a compact
physics vector, raw gravity, and source-episode index.  Controlled tasks add
JSON episode metadata and the task-specific records above.  Headline 2D plots
use fixed named subsets of the larger test split, while some 3D probe diagnostics
use separate evaluation manifests; each result states its cohort size.

Figure~\ref{fig:dataset-trajectory-overview} gives representative trajectories
from all eight datasets, grouped by their prediction or control role.
Figure~\ref{fig:decoded-rollout-id-ood} compares ground-truth frames with
decoded open-loop predictions for one 2D task and two 3D tasks under both
in-distribution and out-of-distribution gravity.  For the decoded rollouts, we
attach a separate online CNN decoder to each world model.  The decoder maps the
256-dimensional latent to RGB from an $8\!\times\!8$ feature map, with 256
base channels, a 32-channel minimum, and one residual block.  We train it on
detached features with RGB MSE, using Muon at $5\times10^{-4}$ for eligible
matrix parameters and AdamW at $10^{-4}$ for the remaining parameters.  The
reconstruction loss therefore does not update the world model.  The decoder is
used only for visualization; quantitative prediction results use frozen
physical-state probes.

\begin{figure*}[!htbp]
    \centering
    \includegraphics[width=0.90\textwidth]{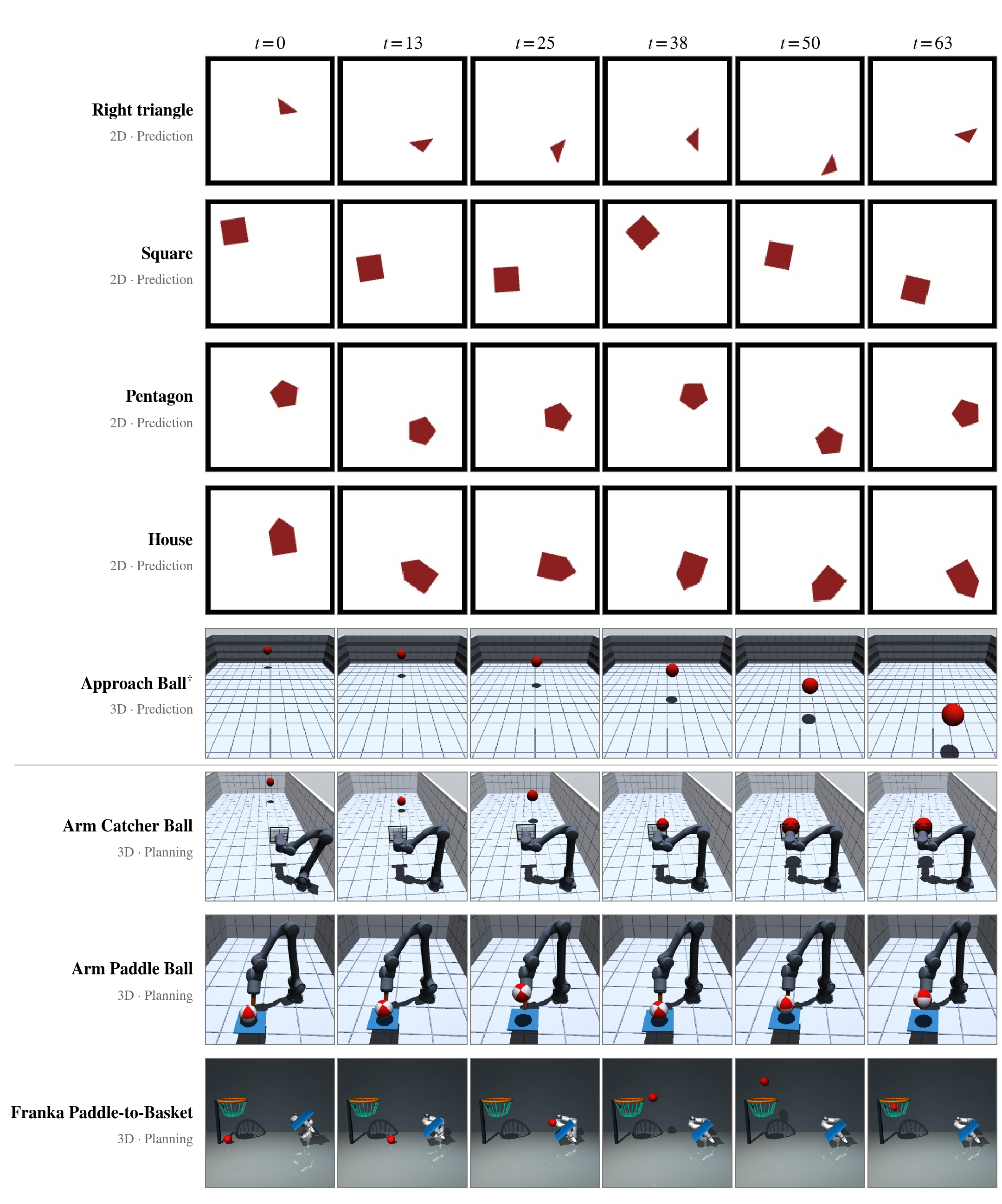}
    \caption{Representative evolution of one episode from each dataset.  We
    show frames $t\in\{0,13,25,38,50,63\}$.  Rows above the divider are passive prediction
    environments; rows below it require arm control and are used for control.}
    \label{fig:dataset-trajectory-overview}
\end{figure*}

\begin{figure*}[!htbp]
    \centering
    \includegraphics[width=\textwidth]{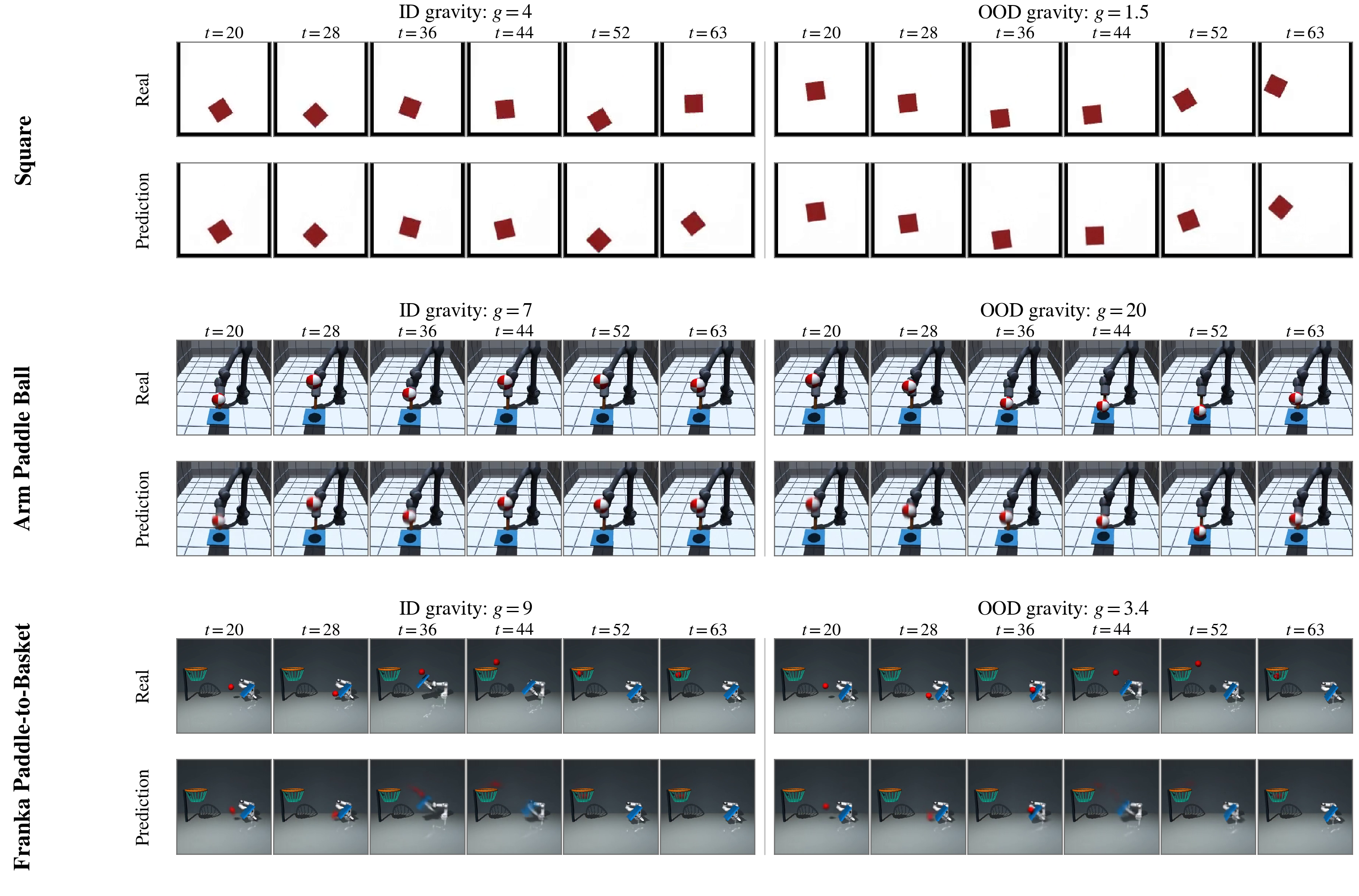}
    \caption{Decoded open-loop rollouts under in-distribution (ID) and
    out-of-distribution (OOD) gravity.  For each dataset, the upper row shows
    ground-truth and the lower row decodes the autoregressively
    predicted latents.  The model receives 20 context frames and then rolls out
    through $t=63$.}
    \label{fig:decoded-rollout-id-ood}
\end{figure*}

\appendixfloatbarriersfalse
\section{Model and training details}
\label{app:model-details}

\subsection{World Model}
\appendixfloatbarrierstrue

Section~\ref{sec:method} gives the model architecture and training objective; in this section, we elaborate on the implementation details needed to reproduce the reported runs. The ViT-Tiny CLS token is projected from 192 to 256 dimensions. The GRU and SSM predictors both use three residual layers of width 512, an MLP width of 2048, dropout 0.1, and action concatenation at every layer. As with the other action channels, gravity $g$ is z-scored before entering the action encoder.

During rollout training, each predicted latent is fed back into the predictor. For the next step, the latent and action histories shift forward together, retaining only their most recent \(H\) entries. Every final SG-JEPA model uses \(H=20\), a five-step training rollout \(K=5\), discount \(\gamma=0.95\), and trainable target latents without stop-gradient since SIGReg already regulates the latent space and avoids collapse. SIGReg is evaluated on encoded latents only, using 1024 random projections and 17 knots.  The coefficient is selected separately for each task family, as summarized in Table~\ref{tab:model-details-dims}. All models are trained for 20 epochs with batchsize 64. Ablation studies for these hyperparameters are presented in the next sections. 

\begin{table*}[!htbp]
\centering
\footnotesize
\caption{SG-JEPA configurations used in the main textheadline experiments. All reported models use the fixed epoch-20
checkpoint.}
\label{tab:model-details-dims}
\begin{tabular}{@{}p{0.34\textwidth}p{0.19\textwidth}p{0.21\textwidth}p{0.15\textwidth}@{}}
\toprule
\textbf{Task family} & \textbf{Image / patch} & \textbf{Predictor} &
\(\boldsymbol{\lambda_{\mathrm{sig}}}\) \\
\midrule
Planar rigid bodies & \(128^2 / 8\) & GRU, SSM & 0.72 \\
Approach Ball & \(256^2 / 16\) & GRU, SSM & 0.18 \\
Arm Catcher Ball & \(256^2 / 16\) & GRU & 0.18 \\
Arm Paddle Ball & \(256^2 / 16\) & GRU & 0.72 \\
Franka Paddle-to-Basket & \(256^2 / 16\) & GRU & 0.09 \\
\bottomrule
\end{tabular}
\end{table*}

All rows use the same hybrid optimizer.  Every trainable two-dimensional
parameter tensor is optimized by Muon at \(10^{-4}\); the remaining parameters use AdamW at \(5\times10^{-5}\) with weight decay \(10^{-3}\).  Muon uses Nesterov momentum 0.95, five Newton--Schulz steps, and the adjustment with scale 0.4~\cite{liu2025muonscalablellmtraining}.

\paragraph{Baseline configurations.}
Original LeWM retains the same ViT-Tiny encoder, 256-dimensional projector,
action preprocessing, and \(H=20\) context, but uses a causal Transformer with six layers, width 256, 16 attention heads, MLP width 2048, and dropout 0.1.  It is trained with the teacher-forced one-step latent MSE and SIGReg coefficient 0.09. These runs use the same hybrid Muon--AdamW recipe above and fixed epoch-20 checkpoints. 

DINO-WM freezes a pretrained DINOv2 ViT-S/14 encoder and extracts a grid of 64 normalized patch tokens, each 384-dimensional, from \(112\times112\) inputs. Its action-conditioned, frame-causal Transformer uses \(H=20\), six layers, 16 attention heads, MLP width 2048, dropout 0.1, and a 10-dimensional action embedding.  The predictor is trained by teacher forcing with one-step MSE to the next frozen patch-token grid.  The reported DINO-WM predictors are trained for 20 epochs with AdamW at \(5\times10^{-4}\), weight decay \(10^{-2}\), batch size 32. 

\subsection{Downstream model}
\label{app:downstream-training}

\paragraph{State probes.}
For each frozen source representation, we fit a separate MLP probe. The probe layer-normalizes the concatenated latent window, followed by two GELU layers of widths 512 and 256 with dropout 0.05 and a linear readout.  We z-score each target coordinate using statistics from the probe-training split and optimize the probe with AdamW at \(10^{-3}\), weight decay \(10^{-5}\), and batch size 256 for at most 50 epochs, retaining the checkpoint with the lowest validation NMSE. Probes are trained only on encoded ground-truth sequences and are then frozen before being applied to predicted rollouts.

\paragraph{Diffusion policies.}
We train a separate policy for every frozen encoder and control task.  In addition to the latent-history GRU described in Section~\ref{sec:method-eval}, the noise 
network is a conditional 1D U-Net with channel widths \(256,512,1024\), kernel size 5, eight GroupNorm groups, and a 256-dimensional diffusion-step embedding. Each policy is trained for 300k updates with batch size 256 using AdamW at \(10^{-4}\), weight decay \(10^{-6}\). The learning rate warms up linearly for 500 updates and then decays with a cosine schedule. The visual encoder remains frozen throughout. In particular, the DINO-WM policy mean-pools frozen patch 
tokens, and neither it nor the SG-JEPA policies query a world-model predictor 
online.

\section{Ablation studies}
\subsection{Optimizer and hyperparameter choices}

\subsubsection{Learning rate sweep}
\label{app:lr-sweep}

We run learning rate sweep on the 2D \(128\times128\) square dataset using a
Transformer predictor.  The 8,000 source episodes are partitioned by episode into 7,200 training and 800 validation episodes; panels (d)--(f) use extra held-out test split.  We vary the AdamW learning rate over
\(\{10^{-5},3\times10^{-5},5\times10^{-5},10^{-4},2\times10^{-4},
5\times10^{-4}\}\), with weight decay \(10^{-3}\), batch size 64,
\(H=20\), \(K=5\), \(\gamma=0.95\), and \(\lambda_{\mathrm{SIG}}=0.72\).  Each run trains for at most 20 epochs.  Early stopping monitors the validation objective with patience three, and we retain the checkpoint with the lowest validation objective.  We then fit temporal-window-4 linear and MLP probes to each checkpoint.

\begin{figure}[!htbp]
    \centering
    \includegraphics[width=\linewidth]{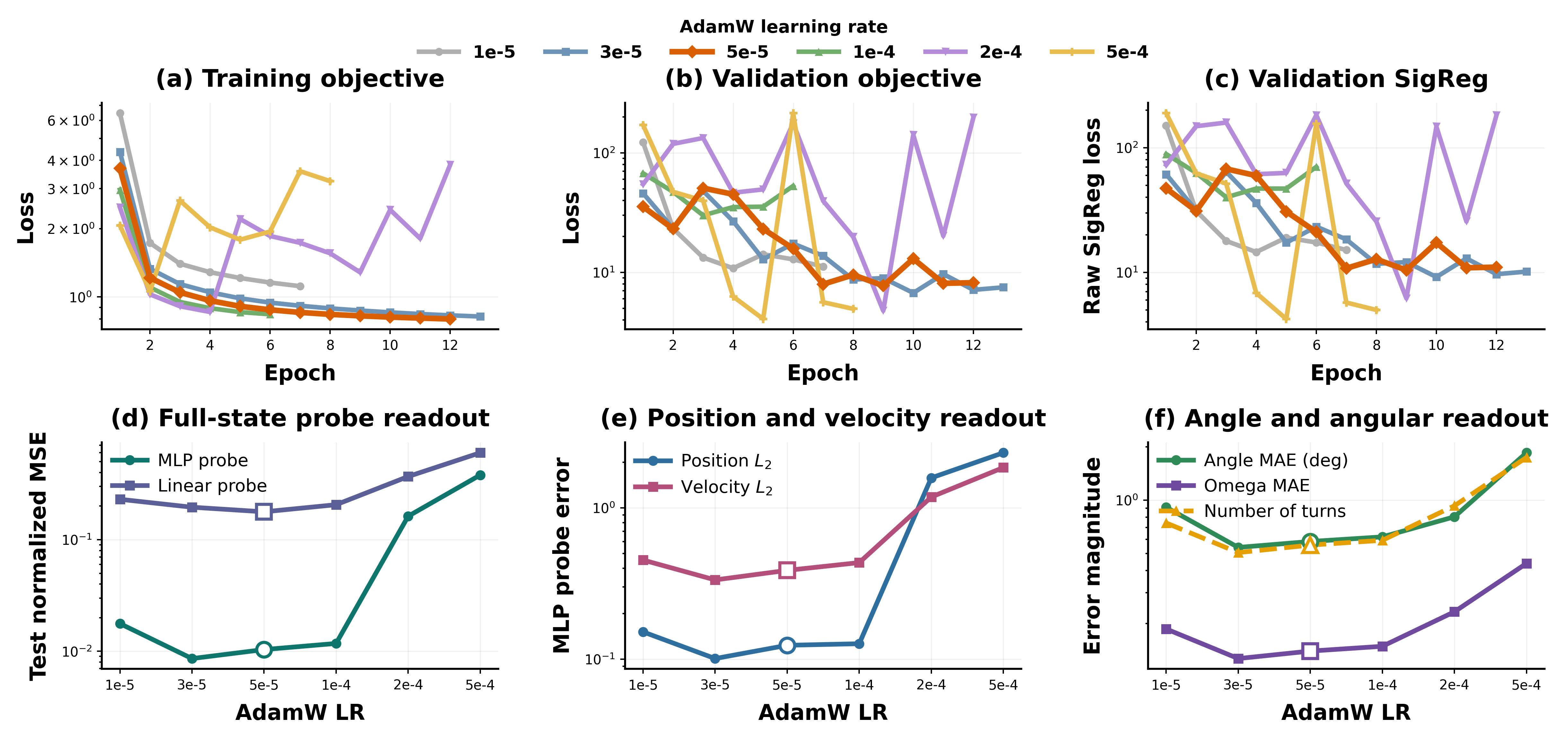}
    \caption{\textbf{AdamW learning-rate sweep on 2D square dataset}}
    \label{fig:lr-sweep}
\end{figure}

\paragraph{Results} The lowest validation objective occurs at \(5\times10^{-4}\), but its poor
probe readout shows that the objective alone is not a reliable measure of
representation quality.  The \(5\times10^{-5}\) setting nearly matches the
best MLP-probe validation NMSE and gives the best linear-probe NMSE; the
held-out test results in Figure~\ref{fig:lr-sweep} show the same trend.  We
therefore use \(5\times10^{-5}\) for the AdamW parameters.

\subsubsection{Muon versus AdamW}
\label{sec:muon_vs_adam}

\begin{figure*}[!htbp]
    \centering
    \includegraphics[width=0.8\textwidth]{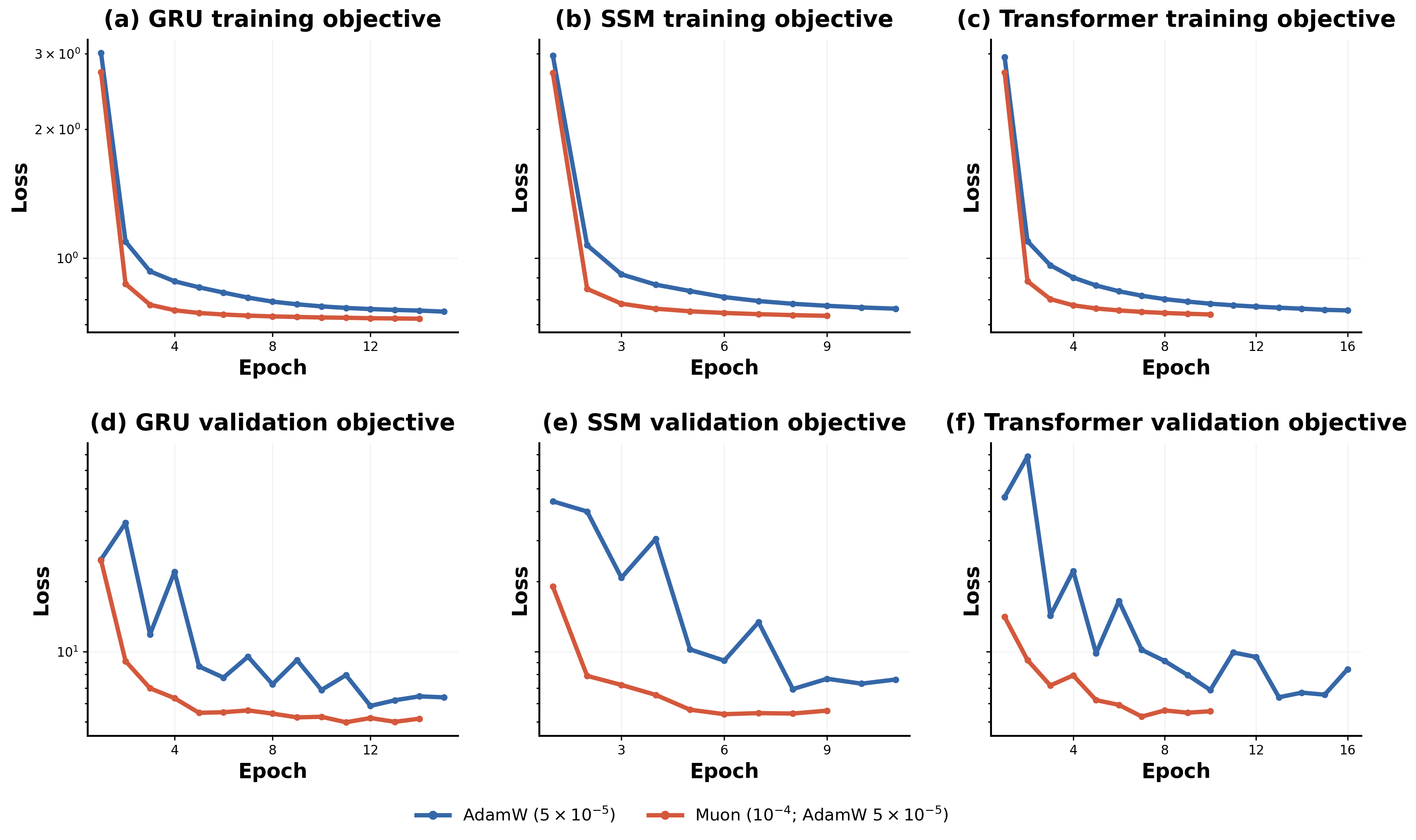}
    \caption{\textbf{AdamW versus hybrid Muon/AdamW}
    Panels (a-c) show training objectives and panels (d-f) show validation
    objectives for the GRU, SSM, and Transformer predictors. Lower is better.}
    \label{fig:muon_adam}
\end{figure*}
\FloatBarrier

In Fig.~\ref{fig:muon_adam}, we compare AdamW with the hybrid Muon/AdamW on the \(128\times128\) square dataset.  For each of the GRU, SSM, and
Transformer predictors, both conditions use a 256-dimensional latent,
\(H=20\), \(K=5\), \(\gamma=0.95\), \(\lambda_{\mathrm{SIG}}=0.72\), batch size
64.  Training is capped at 20 epochs, with early
stopping after three epochs without validation improvement.  AdamW uses
\(5\times10^{-5}\) for all parameters.  The hybrid recipe assigns every
two-dimensional parameter tensor to Muon at \(10^{-4}\) and the remaining
parameters to AdamW at \(5\times10^{-5}\).  Muon uses Nesterov momentum 0.95,
five Newton-Schulz iterations, and per-matrix RMS-matching multiplier
\(0.4\sqrt{\max(A,B)}\).

Across all three predictors, Muon lowers the training objective
faster and reaches a lower best validation objective: \(4.990\) versus \(5.864\) for the GRU, \(5.403\) versus \(6.922\) for the SSM, and \(5.278\) versus \(6.384\) for the Transformer.  These are relative reductions of \(14.9\%\), \(21.9\%\), and \(17.3\%\), respectively.  The selected checkpoints also occur earlier, at epochs 11, 6, and 7 instead of 12, 8, and 13.

\subsubsection{SIGReg coefficient sweep}
\label{app:sigreg-sweep}

In this case we vary \(\lambda_{\mathrm{SIG}}\in\{0.09,0.18,0.36,0.72,1.44\}\) on 2D square dataset. Rollout metrics and decoded examples use a 5,000-episode selection
cohort spanning 25 gravity values from \(-2\) to \(10\); effective rank uses a 2,000-episode well-spread subset. We select SIGReg coefficient based on the downstream rollout performance.

\begin{figure*}[!htbp]
    \centering
    \includegraphics[width=0.8\textwidth]{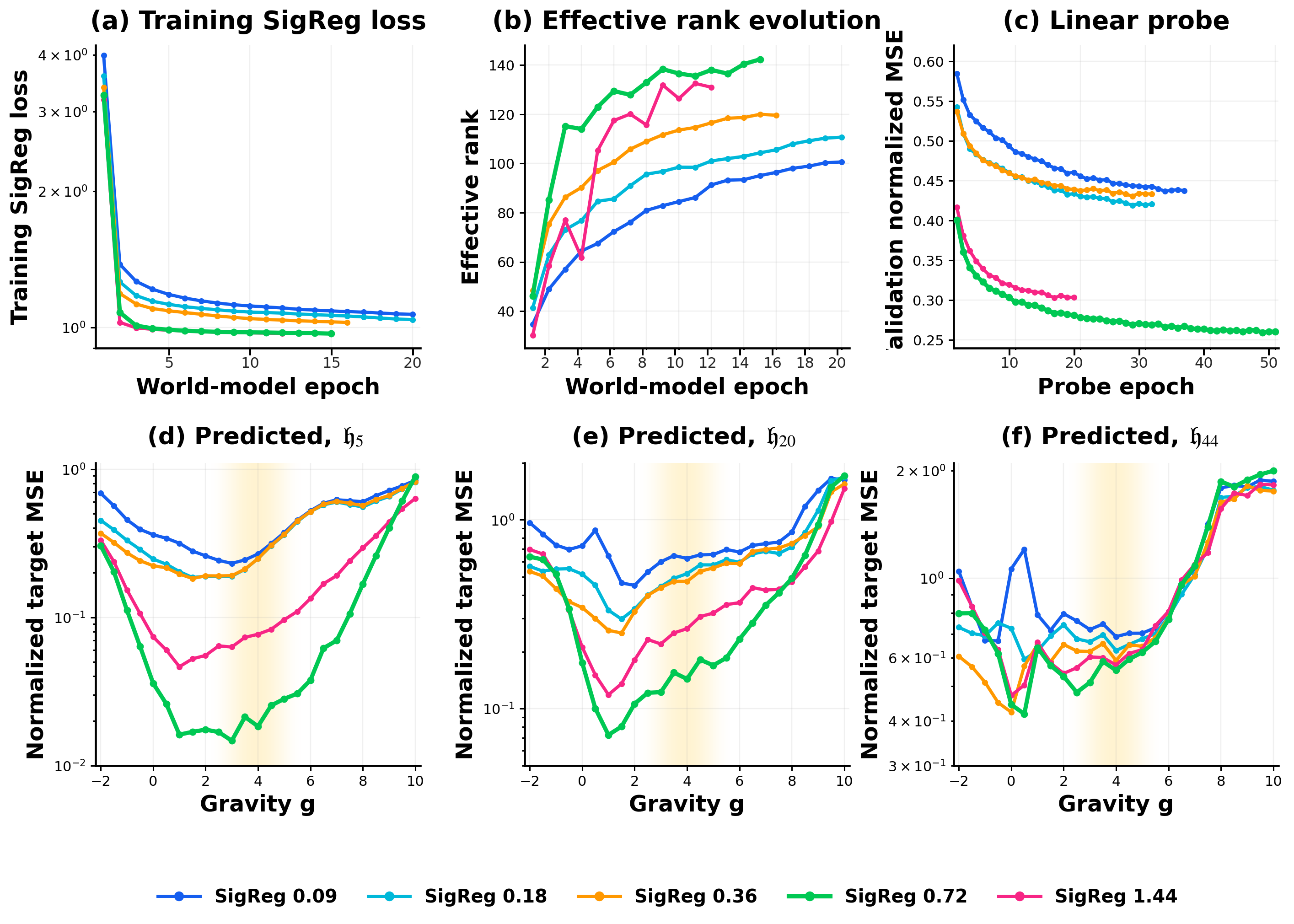}
    \caption{\textbf{SIGReg coefficient sweep.}
    (a) Raw training SIGReg loss; (b) Effective rank on the same 2,000 episodes at every checkpoint; (c) Linear-probe validation NMSE; and (d--f) MLP-probe NMSE after 5, 20, and 44 autoregressive steps across \(g\in[-2,10]\). Panels (c--f) use the same epoch checkpoints. The yellow shade in the background indicates training distribution. Lower is better except in (b).}
    \label{fig:sigreg-sweep}
\end{figure*} 
We note that raw SIGReg loss measures agreement with a standard Gaussian under random one-dimensional projections, not dimensional collapse. We therefore also report the effective rank of the final-context latent covariance. For \(z_i\in\mathbb{R}^{256}\) and \(N=2000\),
\begin{equation}
    \bar z=\frac{1}{N}\sum_{i=1}^{N}z_i,
    \qquad
    C=\frac{1}{N-1}\sum_{i=1}^{N}
    (z_i-\bar z)(z_i-\bar z)^{\top}.
    \label{eq:sigreg-covariance}
\end{equation}
Let \(\lambda_1\geq\cdots\geq\lambda_D\geq0\) denote the eigenvalues of
\(C\). We summarize the spectrum by its effective rank,
\begin{equation}
    p_j=\frac{\lambda_j}{\sum_k\lambda_k},\qquad
    r_{\mathrm{eff}}(C)
    =\exp\!\left(-\sum_{j:p_j>0}p_j\log p_j\right).
    \label{eq:sigreg-effective-rank}
\end{equation}
Fig.~\ref{fig:sigreg-sweep} separates optimization behavior from
representation quality. Although \(0.72\) and \(1.44\) reach similarly low
raw SIGReg losses in panel (a), panels (b) and (c) favor \(0.72\): the effective rank is \(138.0\), compared with \(131.0\) for \(1.44\), and it gives the lowest linear-probe validation NMSE. Panels (d) and (e) show that \(0.72\) also gives the best short- and medium-horizon readout over most of the gravity range. Early rank remains a useful screening signal,
with epoch-5 correlations of \(\rho_s=-0.90\) and \(-1.00\) against later linear- and MLP-probe NMSE, respectively.

\begin{figure*}[!htbp]
    \centering
    \includegraphics[width=\textwidth]{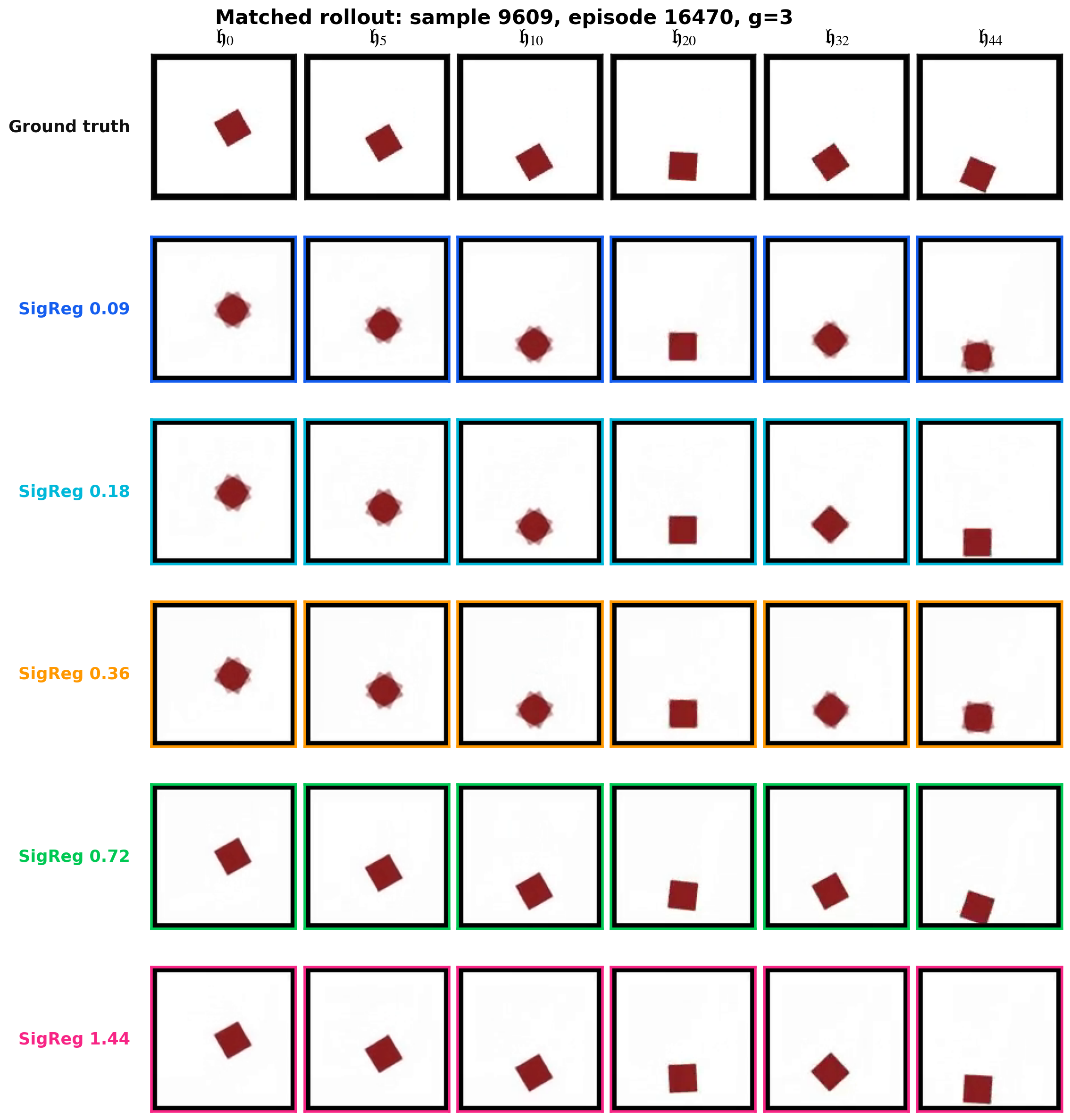}
    \caption{\textbf{Qualitative comparison of the SIGReg coefficient.}
    Matched rollout predictions for one representative square episode at \(g=3\). The top row shows ground-truth frames, and the remaining rows show decoded predicted latents for each value of \(\lambda_{\mathrm{SIG}}\). Columns show the autoregressive rollout horizons 5, 10, 20, 32 and 44. Each prediction row is rendered using the online decoder trained
    jointly with that model.}
    \label{fig:sigreg-example}
\end{figure*}

Fig.~\ref{fig:sigreg-example} shows the same pattern. The representations obtained with \(\lambda_{\mathrm{SIG}}=0.09,0.18,0.36\) decode the square as a rounded or blurred object and accumulate visible pose and position errors. At \(\lambda_{\mathrm{SIG}}=0.72\), the higher-rank representation most clearly preserves corners, orientation, and trajectory over the rollout. Increasing the coefficient to \(1.44\) retains relatively sharp geometry but introduces larger long-horizon trajectory errors. Finally we select \(\lambda_{\mathrm{SIG}}=0.72\) for our 2D experiments because it gives the best balance of effective rank, frozen-probe accuracy, and rollout quality.

\subsection{History length and rollout horizon ablations}

To choose a common temporal configuration without running a separate sweep for every dataset, we perform detailed ablations on one planar task (2D Square) and one 3D task (Approach Ball). The 2D study varies history length and predictor depth, whereas the 3D study also varies the SIGReg coefficient (2D we've done in the section above). Both tasks favor a 20-frame history,
and Approach Ball favors a five-step training rollout. We therefore use \(H=20\) and \(K=5\) across all datasets for a consistent experimental setup. The detailed results are as follows.

\subsubsection{2D square: history length and predictor depth}
\label{app:history-depth-sweep}

On the 2D square dataset, we cross \(H\in\{5,10,20,30\}\) with predictor depth \(L\in\{2,3,4\}\) ($H=40$ causes OOM error on one B200).  All runs train for at most 10 epochs. Fig.~\ref{fig:history-depth-sweep} shows prediction results from trained MLP probes. We see that \(H=20\) gives the best long-horizon accuracy across depths.  At this history length, \(L=4\) is numerically best, with position, orientation, and signed-turn errors of
\(0.745\), \(13.91^\circ\), and \(0.153\) turns.  The \(L=3\) model is within
\(2\%\) to \(4\%\) on all three metrics, while reducing the predictor from \(19.30\)M to \(14.57\)M parameters. We therefore use \(H=20,L=3\) as a balance between performance and training cost.

The real-latent controls explain why we do not use \(H=30\).  Their orientation error rises from at most \(0.66^\circ\) for \(H\leq20\) to about \(11.1^\circ\), and rotation error rises from at most \(0.052\) to \(0.446\) turns.  Thus, the degradation is present in the learned
representation before autoregressive error accumulates, even though the \(H=30\) runs attain lower training objectives.

\begin{figure*}[!htbp]
    \centering
    \includegraphics[width=\textwidth]{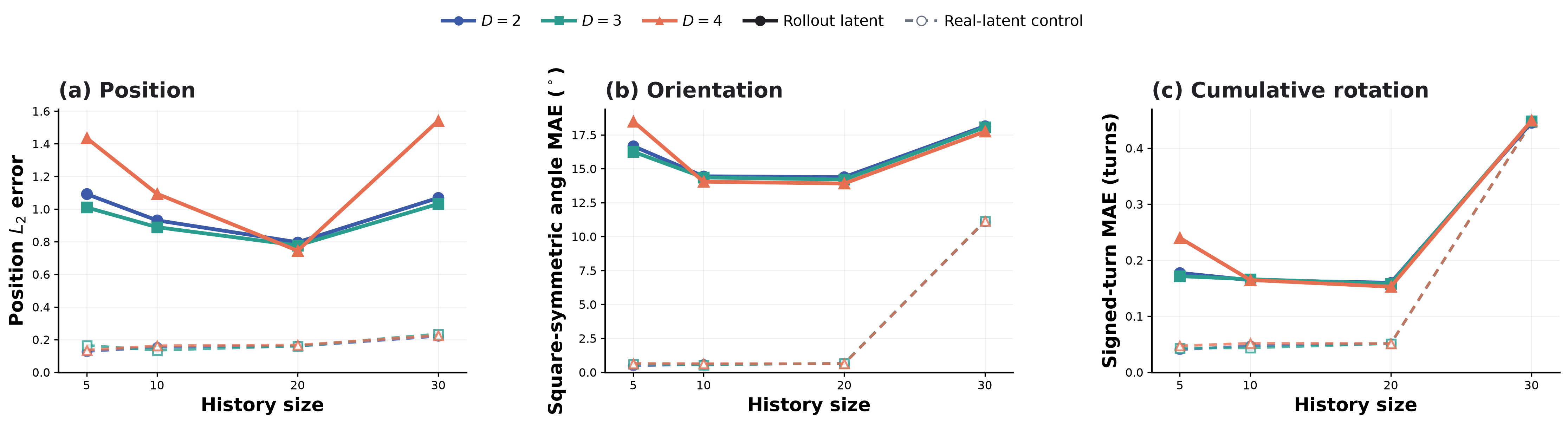}
    \caption{History-length and predictor-depth ablation on 2D quare
    dataset at a common 32-step rollout horizon. Solid
    lines use predicted rollout latents; dashed lines use real encoder latents
    and isolate representation and probe error.  Color and marker shape denote
    predictor depth.  Lower is better.}
    \label{fig:history-depth-sweep}
\end{figure*}

\subsubsection{Approach Ball: history and training rollout horizon}
\label{app:approach-ball-horizon-sigreg}

In this case we jointly vary $(H,K)\in\{(20,5),(20,10),(30,10),(30,20)\}$ and
$\lambda_{\mathrm{SIG}}\in\{0.18,0.36,0.72\}$. For each model, an MLP probe predicts the ball state $(x,y,z,v_x,v_y,v_z)$.  Within each $(H,K)$ pair, we select $\lambda_{\mathrm{SIG}}$ by probe-validation NMSE. Table~\ref{tab:approach-ball-horizon-sigreg} and
Fig.~\ref{fig:approach-ball-ablation} show that $(H,K,\lambda_{\mathrm{SIG}})=(20,5,0.18)$ gives the lowest probe-validation NMSE, direct-probe test NMSE, and 20-step rollout NMSE. Its 32-step rollout NMSE is $0.185$, within $1.3\%$ of the best
value of $0.182$ from $(20,10,0.18)$, while using 61.3 rather than 73.5 GiB of peak host RAM. The
best $H=30$ variants are both less accurate and more expensive.  We therefore
use the task-specific setting $H=20,K=5$, and
$\lambda_{\mathrm{SIG}}=0.18$ for Approach Ball.  The sweep uses one seed,
and the resource values describe training-time host memory.

\begin{table*}[!htbp]
    \centering
    \scriptsize
    \setlength{\tabcolsep}{5pt}
    \begin{tabular}{@{}ccc ccc rr@{}}
        \toprule
        \multicolumn{3}{c}{Training configuration} &
        \multicolumn{3}{c}{Normalized MSE $\downarrow$} &
        \multicolumn{2}{c}{Training resource} \\
        \cmidrule(lr){1-3}\cmidrule(lr){4-6}\cmidrule(l){7-8}
        $H$ & $K$ & $\lambda_{\mathrm{SIG}}$ &
        Probe val. & Probe test & Rollout@32 &
        \shortstack{Decoder frames\\per batch} &
        \shortstack{Peak host\\RAM (GiB)} \\
        \midrule
        \textbf{20} & \textbf{5} & \textbf{0.18} & \textbf{0.078} &
        \textbf{0.110} & \textbf{0.185} & \textbf{800} & \textbf{61.3} \\
        20 & 5  & 0.36 & 0.118 & 0.148 & 0.232 & 800  & 69.4 \\
        20 & 5  & 0.72 & 0.192 & 0.225 & 0.363 & 800  & 67.2 \\
        \midrule
        20 & 10 & 0.18 & 0.080 & 0.114 & \underline{0.182} & 960 & 73.5 \\
        20 & 10 & 0.36 & 0.092 & 0.125 & 0.198 & 960 & 80.4 \\
        20 & 10 & 0.72 & 0.110 & 0.143 & 0.237 & 960 & 66.4 \\
        \midrule
        30 & 10 & 0.18 & 0.137 & 0.166 & 0.347 & 1280 & 81.9 \\
        30 & 10 & 0.36 & 0.135 & 0.167 & 0.236 & 1280 & 75.5 \\
        30 & 10 & 0.72 & 0.152 & 0.185 & 0.263 & 1280 & 87.6 \\
        \midrule
        30 & 20 & 0.18 & 0.316 & 0.325 & 0.392 & 1600 & 111.7 \\
        30 & 20 & 0.36 & 0.321 & 0.330 & 0.401 & 1600 & 113.8 \\
        30 & 20 & 0.72 & 0.150 & 0.182 & 0.260 & 1600 & 116.3 \\
        \bottomrule
    \end{tabular}
    \caption{Approach Ball temporal ablation.}
    \label{tab:approach-ball-horizon-sigreg}
\end{table*}

\begin{figure*}[!htbp]
    \centering
    \includegraphics[width=0.8\textwidth]{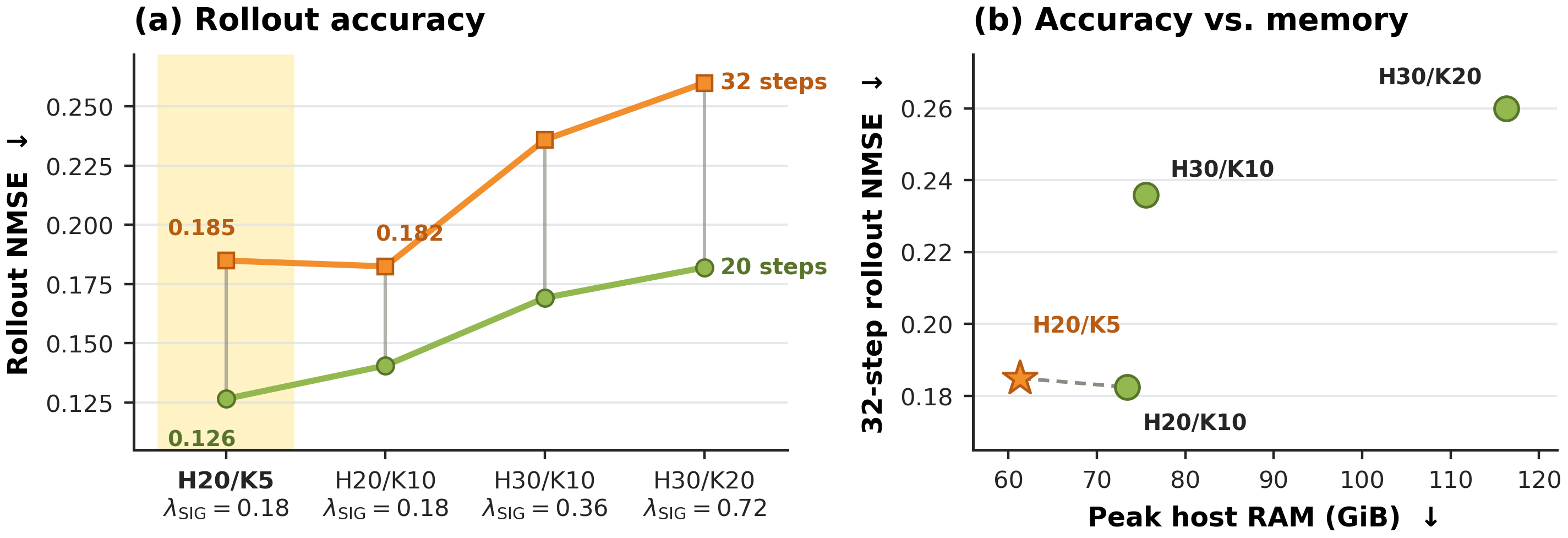}
    \caption{Approach Ball accuracy and resource tradeoffs after selecting
    $\lambda_{\mathrm{SIG}}$ by probe-validation NMSE within each $(H,K)$
    pair.  (a) Rollout NMSE at horizons 20 and 32; (b) peak host RAM and
    decoder frames per microbatch; and (c) the SIGReg sweep for $H=20,K=5$.
    The highlighted configuration is selected.  Lower is better.}
    \label{fig:approach-ball-ablation}
\end{figure*}

\subsection{Rollout discount}
\label{app:gamma-decay-sweep}

We take Arm Catcher Ball dataset and sweep $\gamma\in\{0.90,0.95,0.98,1.00,1.05\}$.  All runs are trained to 20 epochs.  We freeze the epoch-20 representation and evaluate an MLP state probe on $\mathfrak{h}_{44}$ predicted latents. Table~\ref{tab:gamma-decay-sweep} shows that $\gamma=0.95$ gives the lowest joint NMSE and velocity error at horizon 44.  Its position error is $0.122$, close to the best value of $0.118$ at $\gamma=1.00$.  We therefore use
$\gamma=0.95$ as the default across datasets. 

\begin{table}[!htbp]
    \centering
    \small
    \setlength{\tabcolsep}{4pt}
    \caption{Arm Catcher Ball rollout-discount ablation at horizon 44. }
    \label{tab:gamma-decay-sweep}
    \begin{tabular}{@{}cccc@{}}
        \toprule
        $\gamma$ & Joint NMSE $\downarrow$ & Position $L_2$ $\downarrow$ &
        Velocity $L_2$ $\downarrow$ \\
        \midrule
        0.90 & 0.753 & 0.134 & 0.398 \\
        \textbf{0.95} & \textbf{0.600} & 0.122 & \textbf{0.346} \\
        0.98 & 0.683 & 0.138 & 0.351 \\
        1.00 & 0.610 & \textbf{0.118} & 0.354 \\
        1.05 & 0.730 & 0.145 & 0.377 \\
        \bottomrule
    \end{tabular}
\end{table}

\subsection{Counterfactual gravity conditioning}
\label{app:gravity-cf}

In this section, we want to test whether the model actually uses the input gravity. We conduct experiments by supplying wrong gravity values in the action, while holding other coordinates fixed. We again train and evaluate on the 2D square dataset. For each \(g_{\mathrm{true}}\in\{2,4,6\}\), we use the same 200 episodes under
\(g_{\mathrm{in}}\in\{0,2,4,6,8\}\). The full-state probe MSE is normalized
by the training-set variance of each target coordinate. Let
\((\mathfrak{h}_1,\ldots,\mathfrak{h}_7)=(1,3,5,9,20,32,44)\) denote the
evaluated forecast steps.
We summarize rollout error using the horizon-normalized area under the NMSE
curve, computed by the trapezoidal rule:
\[
    \mathrm{AUC}(g_{\mathrm{in}})
    =
    \frac{1}{\mathfrak{h}_7-\mathfrak{h}_1}
    \sum_{i=1}^{6}\frac{\mathfrak{h}_{i+1}-\mathfrak{h}_i}{2}
    \left[
        \mathrm{NMSE}_{\mathfrak{h}_i}(g_{\mathrm{in}})
        +\mathrm{NMSE}_{\mathfrak{h}_{i+1}}(g_{\mathrm{in}})
    \right].
\]
We report the paired difference
\(\Delta\mathrm{AUC}
=\mathrm{AUC}(g_{\mathrm{in}})
-\mathrm{AUC}(g_{\mathrm{true}})\); positive values indicate that supplying
the mismatched gravity increases rollout error.

\begin{figure*}[!htbp]
    \centering
    \includegraphics[width=\textwidth]{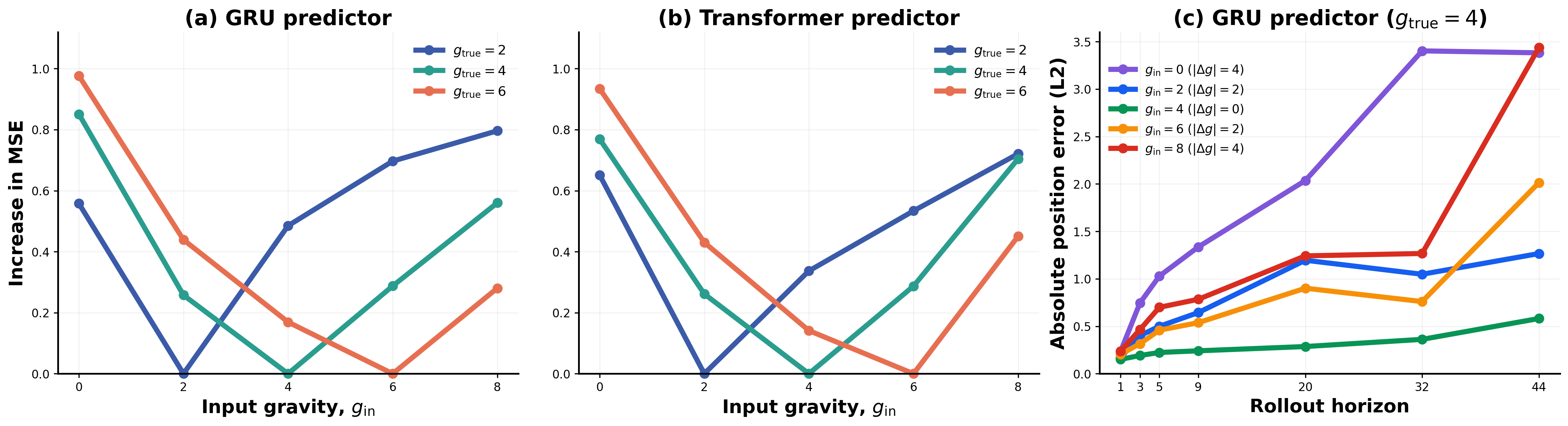}
    \caption{\textbf{Counterfactual gravity conditioning} Only the gravity supplied through the action channel is changed. Panels (a) and (b) report the increase in normalized full-state probe NMSE AUC for the GRU and Transformer predictors. Panel (c) reports the mean GRU position \(L_2\) error in the \(x\)-\(z\) plane across rollout horizons for \(g_{\mathrm{true}}=4\).}
    \label{fig:gravity-cf}
\end{figure*}
\paragraph{Results} From the figure we see that for both predictors, the error is minimized when \(g_{\mathrm{in}}=g_{\mathrm{true}}\), while mismatched inputs increase the rollout error. Larger gravity mismatches generally produce larger increases, although the response is not symmetric around the true value. At \(g_{\mathrm{true}}=4\), correct conditioning also gives the lowest GRU position error at every reported horizon. These interventions show that the predictor does use the supplied gravity when forecasting dynamics.

\subsection{Where does GRU advantage come from}
\label{app:transformer-gru-predictor}

In this section we investigate where the jointly trained GRU advantage comes from: the recurrent predictor itself, the representation it shapes, or a weakness in Transformer
conditioning? To this end we conduct three controlled studies. Matching Transformer action conditioning and initialization narrows, but does not close the gap. Under encoder--predictor crossover, the advantage follows the source representation. Reducing history has little source-specific effect over the first five forecasts, but the GRU-trained representation is less sensitive to short context once recursive errors accumulate over a 44-step rollout. Together, these results identify the representation learned via GRU predictor during joint training as the main source of the gain. Detailed results are as follows.

All comparisons use the same seed and the same 5,000 held-out square episodes, with 200 episodes at each of 25 gravity values from \(-2\) to \(10\). We first average across target dimensions, then across horizons within each episode, and finally across episodes. Subtracting the error obtained from the real encoder latent controls for source-specific probe readability. These intervals measure evaluation-episode variability; all world-model
and predictor training uses one seed.

\subsubsection{Can Transformer conditioning close the joint-training gap?}
\label{app:transformer-conditioning}

The first experiment asks whether the joint-training gap comes from how actions are supplied to the Transformer or from the initialization of its conditioning layers. For GRU we concatenate each latent with its action embedding at the recurrent input. For the Transformer, we test three alternatives: concatenating each latent state with its action once before the first block, reintroducing the action in the attention and MLP paths of every block, or using the action to modulate every block through AdaLN. For AdaLN, we also initialize the conditioning gate to \(0\), \(0.1\), or \(1\), giving five Transformer variants in total. All variants use \(H=20\), \(K=5\), \(\gamma=0.95\), SIGReg coefficient \(0.72\), batch size 64, 20 epochs. More details are in Table~\ref{tab:transformer-gru-conditioning-architectures}

\begin{table*}[!htbp]
    \centering
    \scriptsize
    \setlength{\tabcolsep}{3pt}
    \caption{Predictor architectures in the action conditioning study.}
    \label{tab:transformer-gru-conditioning-architectures}
    \begin{tabular}{@{}lllrrrr@{}}
        \toprule
        Method & Action conditioning & Initialization & Depth & Width & MLP width & Parameters \\
        \midrule
        GRU & Action concatenated at recurrent input & residual 0.1 & 3 & 512 & -- & 14.57M \\
        AdaLN-0 & AdaLN in every block & gate 0 & 6 & 256 & 2048 & 14.98M \\
        AdaLN-0.1 & AdaLN in every block & gate 0.1 & 6 & 256 & 2048 & 14.98M \\
        AdaLN-1 & AdaLN in every block & gate 1 & 6 & 256 & 2048 & 14.98M \\
        Input concat. & One pre-stack action fusion & residual 0.1 & 6 & 256 & 2776 & 14.98M \\
        Per-block concat. & Action fusion in attention and MLP paths & residual 0.1 & 6 & 256 & 2304 & 14.97M \\
        \bottomrule
    \end{tabular}
\end{table*}

\begin{figure*}[!htbp]
    \centering
    \includegraphics[width=\textwidth]{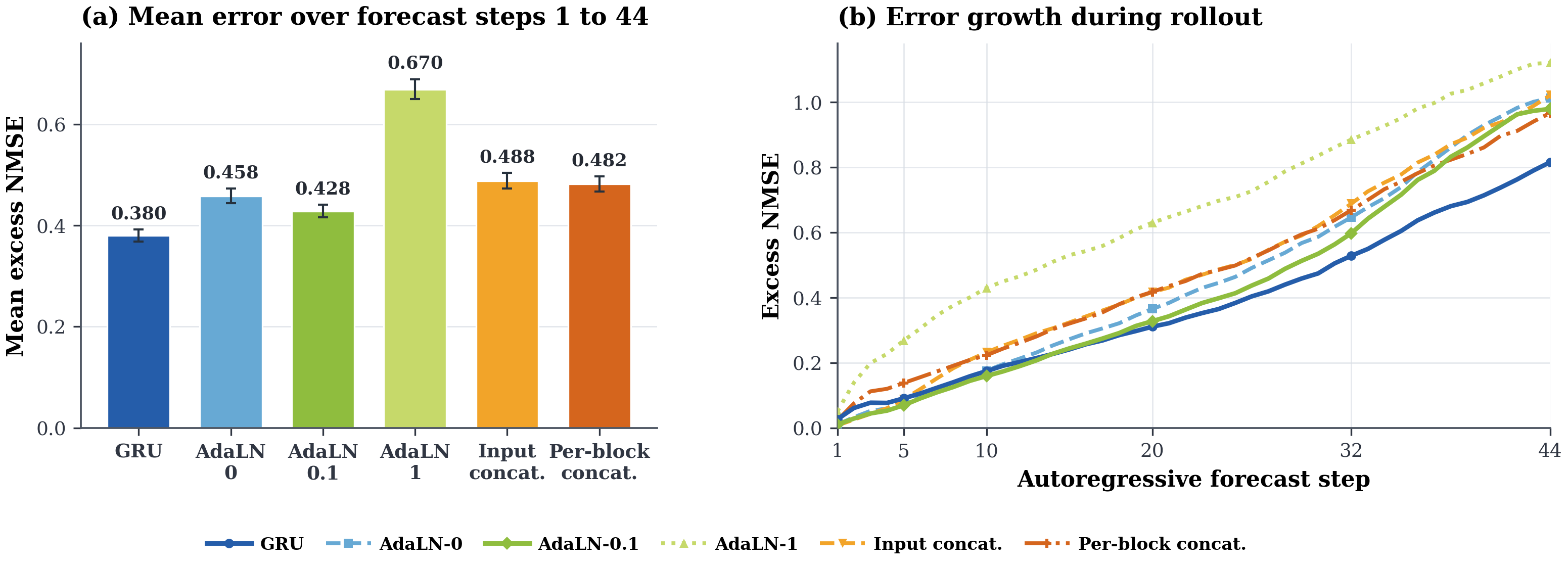}
    \caption{Performance comparison on held-out test set. (a) Mean excess normalized target MSE over forecast steps 1--44; error bars are 95\% episode-bootstrap intervals. (b) Excess error by autoregressive forecast step. All models use the same episode IDs and gravity assignments. Lower is better.}
    \label{fig:transformer-conditioning-results}
\end{figure*}

\paragraph{Results} Fig.~\ref{fig:transformer-conditioning-results} shows the results. Initializing the AdaLN gate to \(0.1\) lowers the Transformer endpoint from
\CondAdaZeroPrimary\ to \CondAdaPointOnePrimary\ and closes \CondBestClosurePercent\% of the gap to the GRU. It remains worse than GRU performance of \CondGruPrimary, however, with a paired difference of \CondBestDelta\. A gate of \(1.0\) degrades the rollout, while input and per-block action concatenation do not improve the
aggregate endpoint over AdaLN-0. The best Transformer is worse than the GRU at
\CondBestWorseGravityCount\ of 25 gravity values, and per-block concatenation
is worse at \CondBlockWorseGravityCount\ of 25. Some variants improve
individual physical quantities, so the conclusion applies to aggregate
rollout error: neither direct action injection nor gate initialization closes
the jointly trained GRU gap.

\subsubsection{Encoder--predictor crossover: full results}
\label{app:frozen-encoder-crossover}

For the second experiment we freeze the encoder from the jointly trained GRU and Transformer with AdaLN, then fit a fresh GRU and Transformer predictor to each latent source. The training setups and hyperparameters are the same except for the predictor. In Section~\ref{sec:dynamics-advantage} we report the overall crossover result and the main interpretation. Here we provide the full metric breakdown in Table~\ref{tab:frozen-encoder-crossover-results}. The representation learned jointly with the GRU gives lower overall physical-state error. The new Transformer modestly reduces error on both sources, by \CrossPredGruSource\ on the GRU-trained representation and \CrossPredTfSource\ on the Transformer-trained representation, yet the GRU-trained representation remains better. This result shows that encoder trained with the GRU predictor produces latents that preserve dynamics-relevant information that either predictor can use. The lower Euclidean latent MSE of the Transformer-trained representation does not translate into better physical-state prediction, especially for velocity. This contrast confirms that preserving dynamics-relevant state is more informative than raw latent distance.

\begin{table*}[!htbp]
    \centering
    \scriptsize
    \setlength{\tabcolsep}{4.2pt}
    \caption{Numerical results for the frozen-representation crossover.  The
    primary endpoint and latent MSE are excess errors averaged over horizons 1
    to 44; physical errors use predicted latents at horizon 44.}
    \label{tab:frozen-encoder-crossover-results}
    \begin{tabular}{@{}llrrrrrr@{}}
        \toprule
        Frozen source & Fresh predictor & Primary $\downarrow$ & Latent MSE $\downarrow$ & Position $\downarrow$ & Velocity $\downarrow$ & Orientation ($^\circ$) $\downarrow$ & $\omega$ $\downarrow$ \\
        \midrule
        GRU-trained & GRU & 1.376 & 2.247 & 3.004 & 3.803 & 22.88 & \textbf{2.116} \\
        GRU-trained & Transformer & \textbf{1.269} & 2.511 & 3.036 & \textbf{3.769} & \textbf{21.77} & 2.139 \\
        Transformer-trained & GRU & 1.555 & \textbf{1.638} & \textbf{2.828} & 5.343 & 23.15 & 2.282 \\
        Transformer-trained & Transformer & 1.453 & 1.732 & 2.854 & 4.849 & 22.57 & 2.401 \\
        \bottomrule
    \end{tabular}
\end{table*}

\subsubsection{History dependence of the frozen representations}
\label{app:history-markov-sufficiency}

The crossover shows that the advantage follows better representation. We next ask whether the GRU-trained representation is less dependent on long context. For each frozen source, we fit new GRU and Transformer predictors with \(m\in\{2,4,8,20\}\) context frames, where \(m\) is the number of most recent latents available to the predictor. Each model predicts the same five target frames and is evaluated through 44 autoregressive steps on held-out episodes. The primary comparison shortens the context from \(m=20\) (H20) to \(m=4\)
(H4). Following Equation~\ref{eq:method-excess-endpoint}, we measure how much the
excess physical-state NMSE changes when history is shortened relative to H20.
A positive source difference means that the Transformer-trained representation loses more
from shorter context. Intervals use 10,000 paired episode bootstrap resamples.

\begin{figure*}[!htbp]
    \centering
    \includegraphics[width=\textwidth]{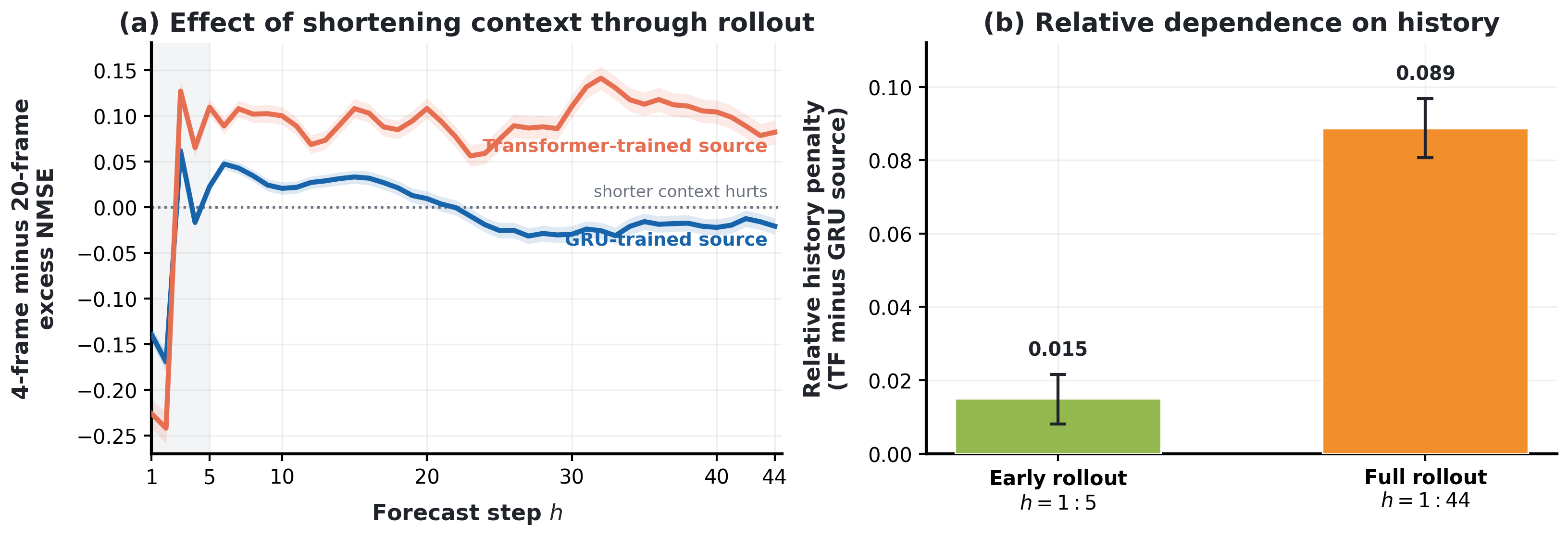}
    \caption{History dependence of the frozen representations. (a) Change in
    excess physical-state NMSE at each forecast step when context is shortened
    from H20 to H4, averaged over newly trained GRU and Transformer predictors.
    Positive values mean that shorter context increases error. The gray region
    marks the first five forecasts; bands show 95\% confidence intervals over
    paired episode differences. (b) Transformer-source minus GRU-source history
    penalty over the first five forecasts and the full 44-step rollout.
    Positive values mean that the Transformer-trained representation depends
    more on long context. }
    \label{fig:history-markov-sufficiency}
\end{figure*}

\paragraph{Results} Fig.~\ref{fig:history-markov-sufficiency} shows the results. (a) plots the change in error at each forecast step when the predictor receives four rather than 20 context frames. Negative values favor H4; positive values mean that shortening the context is
harmful. We see that both representations benefit from H4 during the first two forecasts. From the third forecast onward, however, the Transformer-trained source remains above zero, while the GRU-trained source stays closer to zero and becomes consistently negative after \(\mathfrak{h}=22\). The GRU-trained representation therefore retains its rollout accuracy with substantially less context. Panel (b) summarizes this divergence. The Transformer-source minus GRU-source history penalty is only \(0.015\) over the first five forecasts, but grows to
\(0.089\) over the full 44-step rollout. The gap is therefore small at the start
of prediction and becomes pronounced only after predicted latents are fed back
recursively. This long-rollout robustness provides a concrete explanation for
the stronger dynamics performance of the GRU-trained representation.

\subsection{Performance on conventional control tasks}
\label{sec:conventional-tasks}

\begin{figure}[!htbp]
  \centering
  \includegraphics[width=0.5\linewidth]{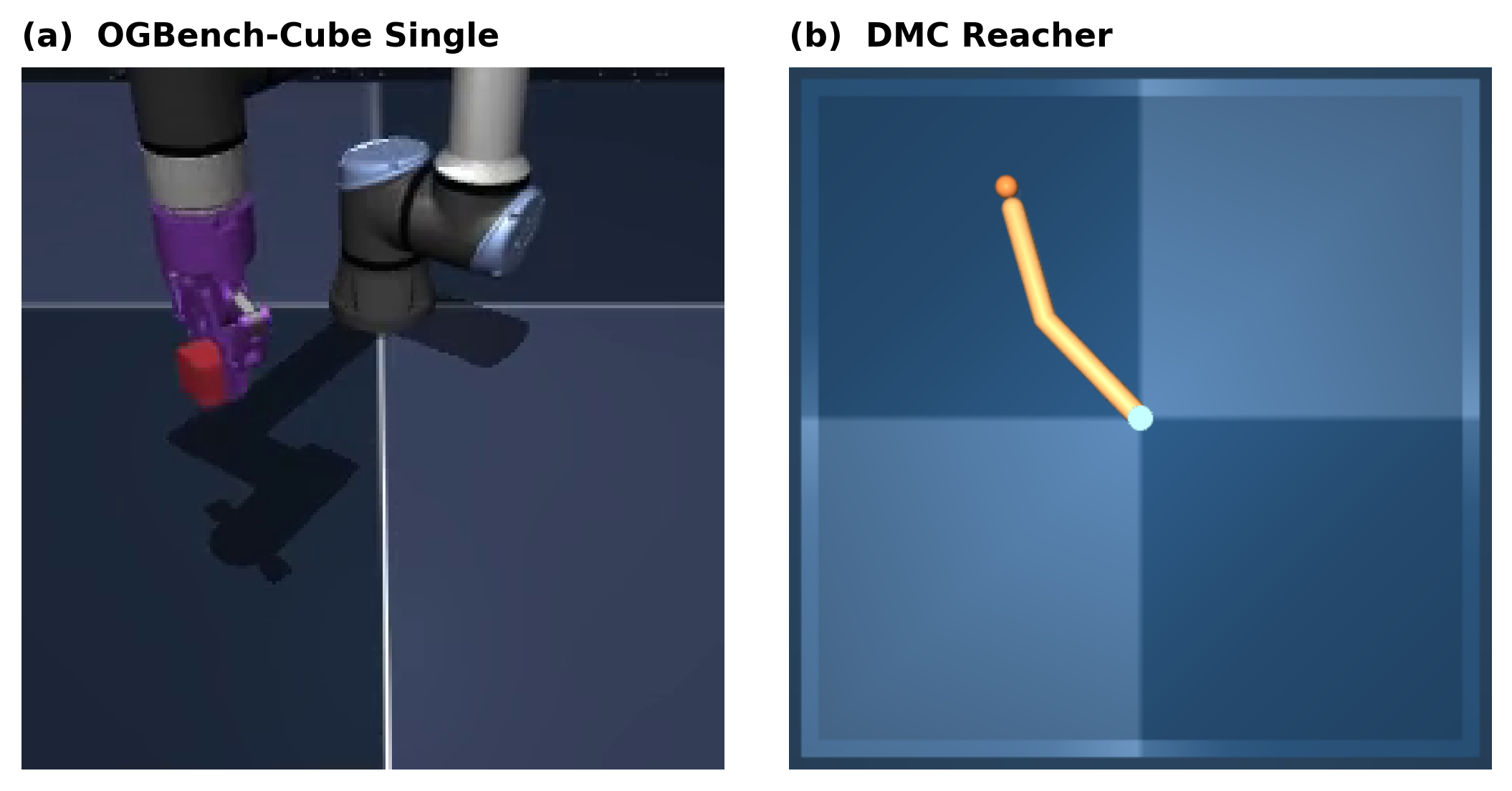}
  \caption{Representative illustration of the datasets (a) OGBench-Cube Single and
  (b) DMC Reacher.}
  \label{fig:conventional-tasks}
\end{figure}

The dynamics experiments above favor the GRU predictor over Transformer. We next test whether this choice degrades performance on conventional control tasks. We compare Transformer, GRU, and SSM predictors on OGBench-Cube Single and DMC Reacher (Figure~\ref{fig:conventional-tasks}), changing only the predictor architecture. All models use the same ViT-tiny encoder, H10/K5 objective, five-step action blocks and 20-epoch training schedule. We evaluate the same latent-space CEM planner on 50 fixed held-out cases per model. 

\begin{table*}[!htbp]
  \centering
  \small
  \caption{Predictor comparison on conventional control tasks. Success is measured over 50 fixed held-out cases, with episode-bootstrap 95\% confidence intervals. Latent metrics use all 127,000 held-out H10/K5 clips. }
  \label{tab:conventional-tasks}
  \resizebox{\textwidth}{!}{%
  \begin{tabular}{llrrrrr}
    \toprule
    Dataset & Predictor & Success $\uparrow$
      & K5 latent MSE $\downarrow$ & Mean K1--K5 MSE $\downarrow$ & Planning latency (s) $\downarrow$ \\
    \midrule
    OGBench-Cube & Transformer & 68\% (54--80) & 0.00590 & 0.00445 & 0.195 \\
    & GRU         & \textbf{70\% (56--82)} & 0.00561 & 0.00430 & \textbf{0.111} \\
    & SSM         & 68\% (54--80) & \textbf{0.00497} & \textbf{0.00378} & 0.279 \\
    \midrule
    DMC Reacher & Transformer  & \textbf{100\% (100--100)} & 0.00226 & 0.00182 & 0.568 \\
    & GRU          & 98\% (94--100) & 0.00323 & 0.00289 & \textbf{0.316} \\
    & SSM          & 98\% (94--100) & \textbf{0.00159} & \textbf{0.00132} & 1.115 \\
    \bottomrule
  \end{tabular}%
  }
\end{table*}

We observe that SG-JEPA (GRU/SSM) predictors preserves downstream control success on both tasks(Table~\ref{tab:conventional-tasks}). It also has the lowest planning latency and the highest effective rank on both datasets. Thus, the conventional benchmarks show no material loss in control success from using a GRU/SSM predictor. Together with its advantage on the dynamics benchmarks, this result supports our use of the GRU as the default predictor.

\section{2D datasets full results}
\label{sec:2d_full_results}

\subsection{Gravity-conditioned rollouts}
In Section~\ref{sec:main_results} we present overall comparison on 2D triangle and square datasets. In this section we elaborate on the detailed breakdown, first by per-held-out-test-gravity evaluation metrics, as in Fig.~\ref{fig:2d-per-gravity-results}; then results by rollout horizon, as in Table~\ref{tab:official-2d-shape-rollout-short-long}.

\begin{figure}[!htbp]
    \centering
    \includegraphics[width=0.9\linewidth]{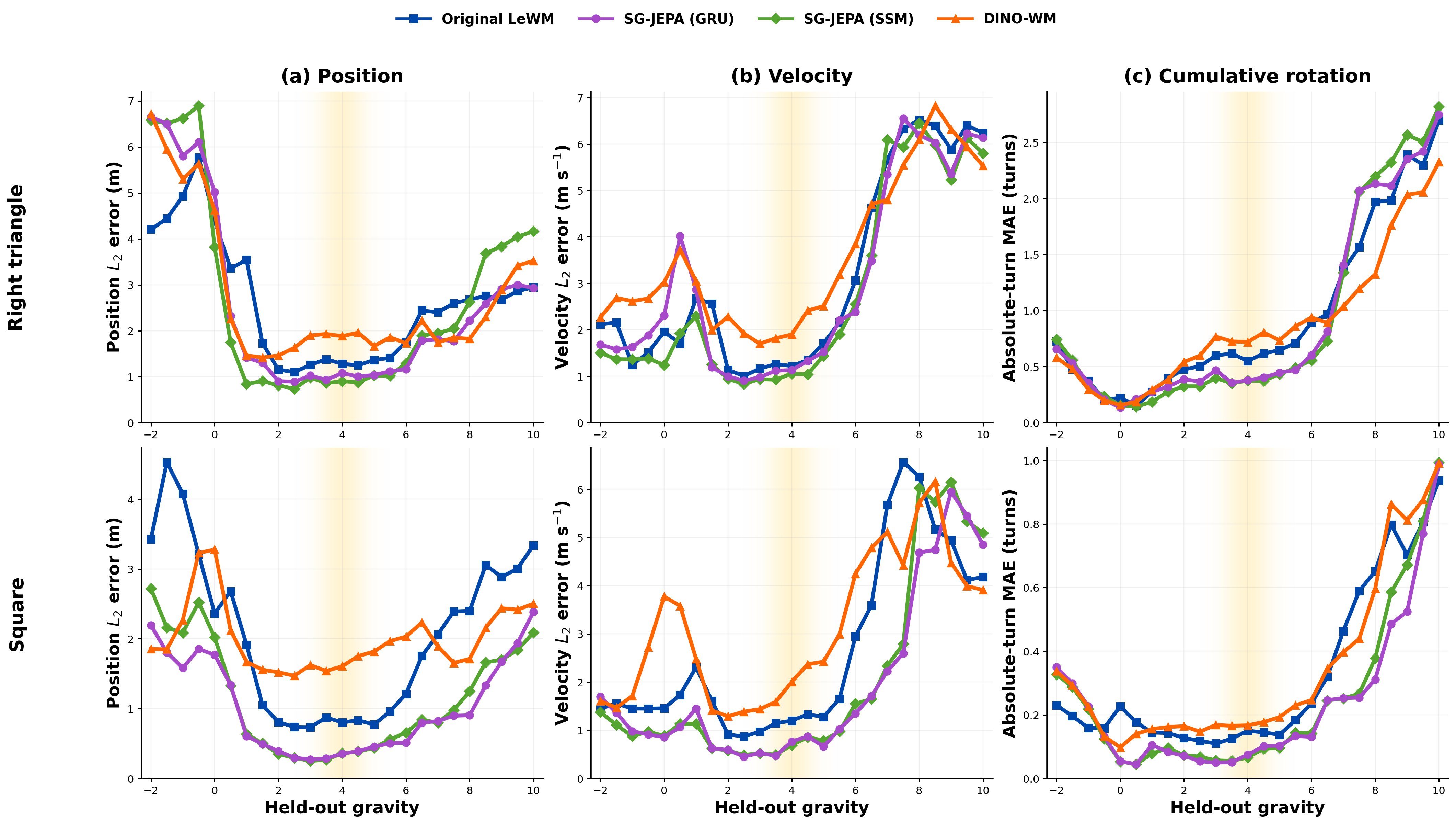}
    \caption{\textbf{Per-gravity evaluation at
    \(\mathfrak{h}_{44}\).} Rows show right-triangle and square
    episodes; columns report planar position \(L_2\) error, planar
    velocity \(L_2\) error and cumulative-rotation MAE. The fading
    yellow background follows the relative training-gravity density
    \(g\sim N(4,0.5^2)\). }
    \label{fig:2d-per-gravity-results}
\end{figure}

\begin{table*}[!htbp]
    \centering
    \small
    \setlength{\tabcolsep}{5.2pt}
    \renewcommand{\arraystretch}{1.10}
    \begin{tabular}{@{}llrrrr@{}}
        \toprule
        Metric & Method
        & \(\mathfrak{h}_{5}\) & \(\mathfrak{h}_{16}\)
        & \(\mathfrak{h}_{32}\) & \(\mathfrak{h}_{44}\) \\
        \midrule
        \multicolumn{6}{@{}l}{\textit{Right triangle}} \\
        Position \(L_2\) (m) & Original LeWM
        & 1.105 & \underline{1.943} & \textbf{2.325} & \underline{2.625} \\
        & SG-JEPA (GRU)
        & \underline{0.968} & 2.167 & 2.486 & \textbf{2.529} \\
        & SG-JEPA (SSM)
        & 1.251 & 2.455 & 2.628 & 2.665 \\
        & DINO-WM
        & \textbf{0.532} & \textbf{1.510} & \underline{2.329} & 2.764 \\
        \addlinespace[2pt]
        Velocity \(L_2\) (m/s) & Original LeWM
        & 2.324 & \textbf{1.598} & \textbf{2.499} & 3.121 \\
        & SG-JEPA (GRU)
        & \underline{2.258} & 1.920 & 2.837 & \underline{3.001} \\
        & SG-JEPA (SSM)
        & 2.343 & \underline{1.667} & \underline{2.592} & \textbf{2.765} \\
        & DINO-WM
        & \textbf{1.897} & 2.420 & 3.275 & 3.574 \\
        \addlinespace[2pt]
        Cumulative rotation (turns) & Original LeWM
        & 0.065 & 0.269 & 0.640 & 0.946 \\
        & SG-JEPA (GRU)
        & \underline{0.061} & 0.268 & \underline{0.617} & \underline{0.904} \\
        & SG-JEPA (SSM)
        & \textbf{0.059} & \underline{0.264} & 0.618 & 0.911 \\
        & DINO-WM
        & 0.062 & \textbf{0.255} & \textbf{0.609} & \textbf{0.875} \\
        \addlinespace[3pt]
        \midrule
        \multicolumn{6}{@{}l}{\textit{Square}} \\
        Position \(L_2\) (m) & Original LeWM
        & 0.930 & 1.934 & 2.042 & 2.074 \\
        & SG-JEPA (GRU)
        & \textbf{0.374} & \textbf{0.658} & \textbf{0.862} & \textbf{1.034} \\
        & SG-JEPA (SSM)
        & \underline{0.395} & \underline{0.950} & \underline{1.094} & \underline{1.148} \\
        & DINO-WM
        & 0.454 & 1.117 & 1.570 & 2.006 \\
        \addlinespace[2pt]
        Velocity \(L_2\) (m/s) & Original LeWM
        & 1.791 & 1.545 & 2.735 & 2.631 \\
        & SG-JEPA (GRU)
        & \underline{1.158} & \textbf{0.896} & \textbf{1.455} & \textbf{1.915} \\
        & SG-JEPA (SSM)
        & \textbf{1.051} & \underline{1.007} & \underline{1.661} & \underline{2.007} \\
        & DINO-WM
        & 1.793 & 1.921 & 3.011 & 3.082 \\
        \addlinespace[2pt]
        Cumulative rotation (turns) & Original LeWM
        & 0.029 & 0.104 & 0.215 & 0.321 \\
        & SG-JEPA (GRU)
        & \underline{0.018} & \underline{0.066} & \textbf{0.143} & \textbf{0.236} \\
        & SG-JEPA (SSM)
        & \textbf{0.015} & \textbf{0.061} & \underline{0.146} & \underline{0.249} \\
        & DINO-WM
        & 0.027 & 0.104 & 0.223 & 0.341 \\
        \bottomrule
    \end{tabular}
    \caption{\textbf{Rollout results}.  Forecast steps
    \(\mathfrak{h}_{5},\mathfrak{h}_{16},\mathfrak{h}_{32}\), and
    \(\mathfrak{h}_{44}\) correspond to 0.3125, 1.0, 2.0, and 2.75 seconds.  \textbf{bold} and \underline{underlining} mark the best and second-best value within each shape, metric and step. }
    \label{tab:official-2d-shape-rollout-short-long}
\end{table*}

\paragraph{Results} SG-JEPA gives the strongest overall performance, with its clearest advantage appearing as rollout errors accumulate.  On the square, the GRU and SSM variants lead every metric at every reported step.  SSM is slightly stronger at the shortest steps, while GRU is more accurate over longer recursive rollouts.  Figure~\ref{fig:2d-per-gravity-results} shows that these gains extend across most of the held-out gravity grid, rather than being confined to the
center of the training distribution.  The few reversals occur mainly for velocity at the highest gravity values.

The triangle is more competitive at short and intermediate steps, where the baselines lead several translation entries.  Their advantage does not persist: at \(\mathfrak{h}_{44}\), SG-JEPA (GRU) has the lowest position error and SG-JEPA (SSM) has the lowest velocity error.  DINO-WM retains the lowest cumulative-rotation error on the triangle, the only metric not led by
SG-JEPA at the longest rollout.  Across the two shapes, SG-JEPA therefore leads five of the six \(\mathfrak{h}_{44}\) shape-metric comparisons, supporting its advantage for long-horizon dynamics prediction.

\subsection{Shape generalization}
\label{app:house-shape-generalization}

In this separate ablation we study the model's ability of shape generalization. To do this we conduct the following experiment: We train SG-JEPA on both 2D right triangle and square datasets. And we ask whether it can generalize dynamics from these component shapes to the composite "house" shape (A "house" joins a right triangle and a square into one rigid body), which is unseen during training. We perform two tests, one trained with triangle + square ($T{+}S$), the other with triangle + square + pentagon ($T{+}S{+}P$).  We then compare the performance against SG-JEPA trained with only "house" dataset on downstream probe evaluation. Moreover we trained with both Tiny ViT ($D=192$) and Small ViT ($D=384$) encoders for these cases to rule out the possibility of model capacity limit when dealing with multiple datasets. All six world models follow the planar SG-JEPA recipe in
Appendix Section~\ref{app:model-details}.

\begin{table*}[!htbp]
    \centering
    \caption{House rollout error at $\mathfrak{h}_{44}$, averaged over the
    seven gravity values in the band $g\in[2.5,5.5]$.}
    \label{tab:house-shape-generalization}
    \small
    \setlength{\tabcolsep}{4pt}
    \begin{tabular}{llrrr}
        \toprule
        Encoder & World-model training data & Position $L_2$ (m) & Velocity $L_2$ (m/s) & Cumulative rotation (turns) \\
        \midrule
        Tiny  & $T{+}S$       & 2.681 & 2.846 & 0.821 \\
        Tiny  & $T{+}S{+}P$   & 2.733 & 2.792 & 0.832 \\
        Tiny  & House-trained & \textbf{1.723} & \textbf{1.444} & \textbf{0.193} \\
        \midrule
        Small & $T{+}S$       & 2.489 & 2.725 & 0.787 \\
        Small & $T{+}S{+}P$   & 2.556 & 2.762 & 0.833 \\
        Small & House-trained & \textbf{1.781} & \textbf{1.607} & \textbf{0.234} \\
        \bottomrule
    \end{tabular}
\end{table*}

\paragraph{Results} In Table~\ref{tab:house-shape-generalization} and Fig.~\ref{fig:house-shape-generalization} we present the results. Overall, SG-JEPA transfers a substantial part of the house's translational dynamics from its component shapes.  Position transfers most consistently, velocity is less stable, and rotation remains the hardest quantity to predict.  Direct house training is still better near the training distribution.  As shown in Table~\ref{tab:house-shape-generalization}, within $g\in[2.5,5.5]$, Small $T{+}S$ has $1.40\times$ the position error and $1.70\times$ the velocity error of the house-trained reference.  Adding pentagons brings no consistent improvement, so the observed transfer already
comes from the triangle and square components rather than from greater shape
diversity.

The main highlight is the OOD gravity result in
Figure~\ref{fig:house-shape-generalization}.  Across much of both gravity
tails, the $T{+}S$ model closely follows the house-trained reference on
position despite never seeing a house during world-model training.  With the
Small encoder, it matches or outperforms the house-trained model at several
extreme gravity values. 

\begin{figure*}[!htbp]
    \centering
    \includegraphics[width=\textwidth]{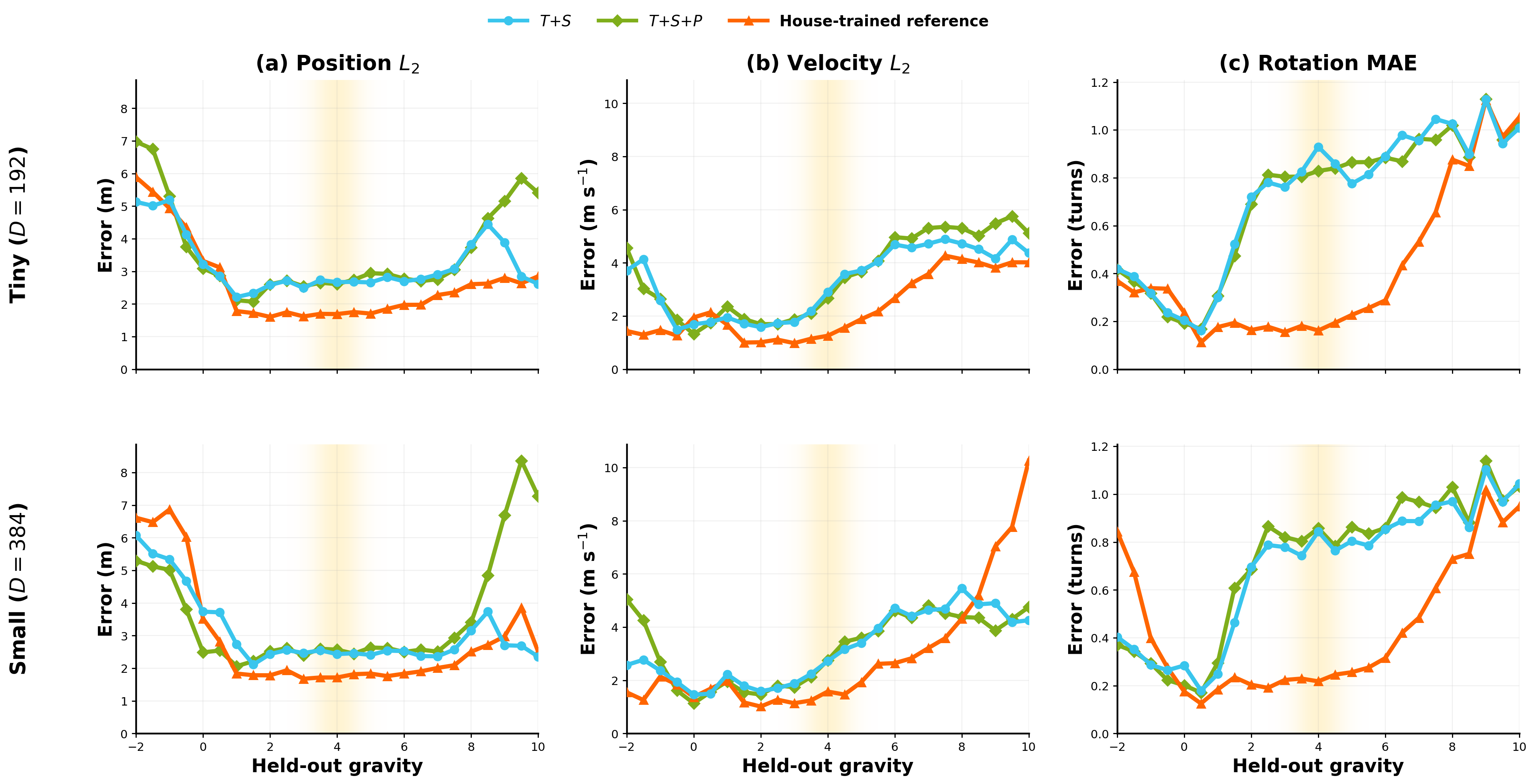}
    \caption{Per-gravity house rollout error at $\mathfrak{h}_{44}$.  Rows show
    the Tiny ($D=192$) and Small ($D=384$) encoders; columns show position,
    velocity, and cumulative-rotation error.}
    \label{fig:house-shape-generalization}
\end{figure*}

The true-latent control helps localize the remaining gap.  On Small $T{+}S$,
probing encoded ground-truth house frames gives error ratios of 0.90 for position, 1.23 for velocity, and 1.12 for cumulative rotation relative to the house-trained reference.  After autoregressive rollout, the ratios rise to 1.40, 1.70, and 3.36. Thus most of the rotation gap develops during recursive prediction, where small angular-velocity errors accumulate and the new contact geometry matters.

\section{3D datasets full results}
\subsection{Approach-Ball results}
\label{app:approach-ball-comparison}

In main text, Fig.~\ref{fig:main-3d}(a)(b) we report Approach-Ball position error over forecast steps and held-out gravity values.  Here we show the complete results with
velocity error and a compact summary of both physical quantities. Note that for this dataset we do not have ball rotation or friction, so the only physical quantities are position and velocity.  

\begin{table}[!htbp]
  \centering
  \footnotesize
  \setlength{\tabcolsep}{4pt}
  \caption{Approach-Ball physical-state prediction error on 1,600 held-out
  episodes.  The mean is taken over \(\mathfrak{h}_1\)--\(\mathfrak{h}_{44}\), and
  \(\mathfrak{h}_{44}\) is 2.75\,s at 16\,fps.  Errors are reported in
  physical units.  Lower is better; bold marks the best result.}
  \label{tab:approach-ball-aggregate}
  \begin{tabular}{@{}lcccc@{}}
    \toprule
    & \multicolumn{2}{c}{Position $L_2$ (m)} &
      \multicolumn{2}{c}{Velocity $L_2$ (m\,s$^{-1}$)} \\
    \cmidrule(lr){2-3}\cmidrule(l){4-5}
    Method & Mean & $\mathfrak{h}_{44}$ & Mean & $\mathfrak{h}_{44}$ \\
    \midrule
    Original LeWM    & 0.0986 & 0.1229 & 1.1361 & 1.1400 \\
    SG-JEPA (GRU)   & \textbf{0.0491} & \textbf{0.0764} & \textbf{0.5391} & 0.7829 \\
    SG-JEPA (SSM)   & 0.0495 & 0.0785 & 0.5664 & 0.8861 \\
    DINO-WM          & 0.0705 & 0.1152 & 0.5672 & \textbf{0.7411} \\
    \bottomrule
  \end{tabular}
\end{table}

\paragraph{Results} Position gives the clearest result.  The full state readout reinforces the finding in the main text: SG-JEPA (GRU) and SG-JEPA (SSM) have the two lowest position errors, both on average and at the final step, and reduce mean position error by roughly 30\% relative to DINO-WM.

Velocity does not show the same uniform advantage however.  GRU has a slightly lower
rollout-average error than DINO-WM, while SSM and DINO-WM are nearly tied. DINO-WM is more accurate at \(\mathfrak{h}_{44}\).  As shown in Figure~\ref{fig:approach-ball-velocity}(a), GRU or SSM is lower over much of the intermediate rollout, but DINO-WM closes the gap and becomes stronger near the end.

The gravity breakdown in Figure~\ref{fig:approach-ball-velocity}(b) also
separates the extrapolation regimes.  GRU is strongest across most of the
central and moderate-gravity range, DINO-WM at the high-gravity tail, and
Original LeWM at the few lowest values.  Taken together, these results show a
clear SG-JEPA advantage for position. For velocity, SG-JEPA is competitive with, but not uniformly better than, DINO-WM.  Each world model uses one training seed, so these episode averages do not measure variation across training runs.

\begin{figure*}[!htbp]
  \centering
  \includegraphics[width=0.9\textwidth]{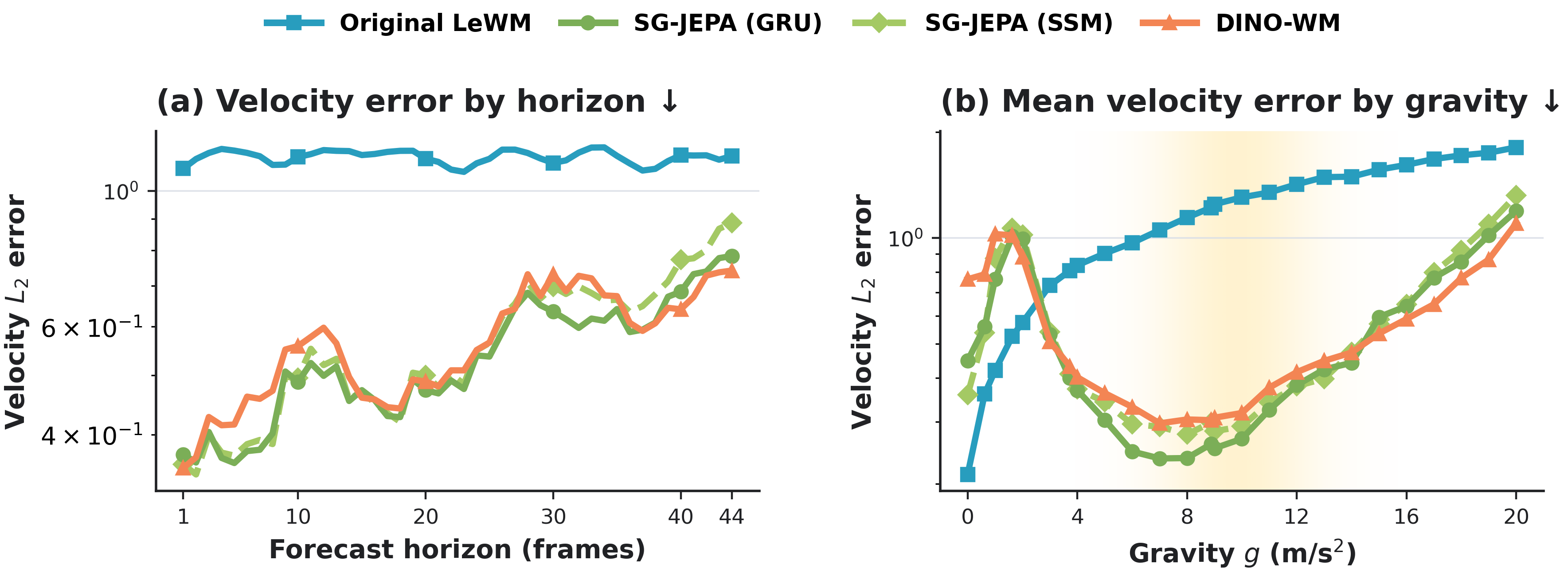}
  \caption{\textbf{Approach-Ball velocity prediction.}
  The evaluation and plotting setup matches the position results in
  Figure~\ref{fig:main-3d}(a,b).
  (a) Mean ball-velocity \(L_2\) error at each recursive forecast step.
  (b) Mean error over \(\mathfrak{h}_1\)--\(\mathfrak{h}_{44}\) at each of
  the 25 held-out gravity values.  The fading yellow
  background shows the training distribution
  \(g\sim\mathcal{N}(9.8,2.0^2)\).}
  \label{fig:approach-ball-velocity}
\end{figure*}

\subsection{Arm-Catcher-Ball control results}
\label{app:arm-catcher-ball-control}

In main text Figure~\ref{fig:main-3d}(c,d) we present the success rate comparison for Arm Catcher Ball for SG-JEPA (GRU) and DINO-WM. In this section we report additional results for SG-JEPA (SSM) and Original LeWM (Transformer predictor) for completeness. Diffusion policies are trained for each checkpoint respectively. At inference, the policy conditions on a 20-frame latent history, predicts 16 Cartesian actions, executes eight, and replans.  Sampling uses the complete 100-step cosine DDPM. Success requires the ball to enter the catcher and be latched before the episode ends.  We evaluate the same test set at 23 gravity values under five paired simulator-rollout seeds. 

\begin{figure*}[!htbp]
  \centering
  \includegraphics[width=0.9\textwidth]{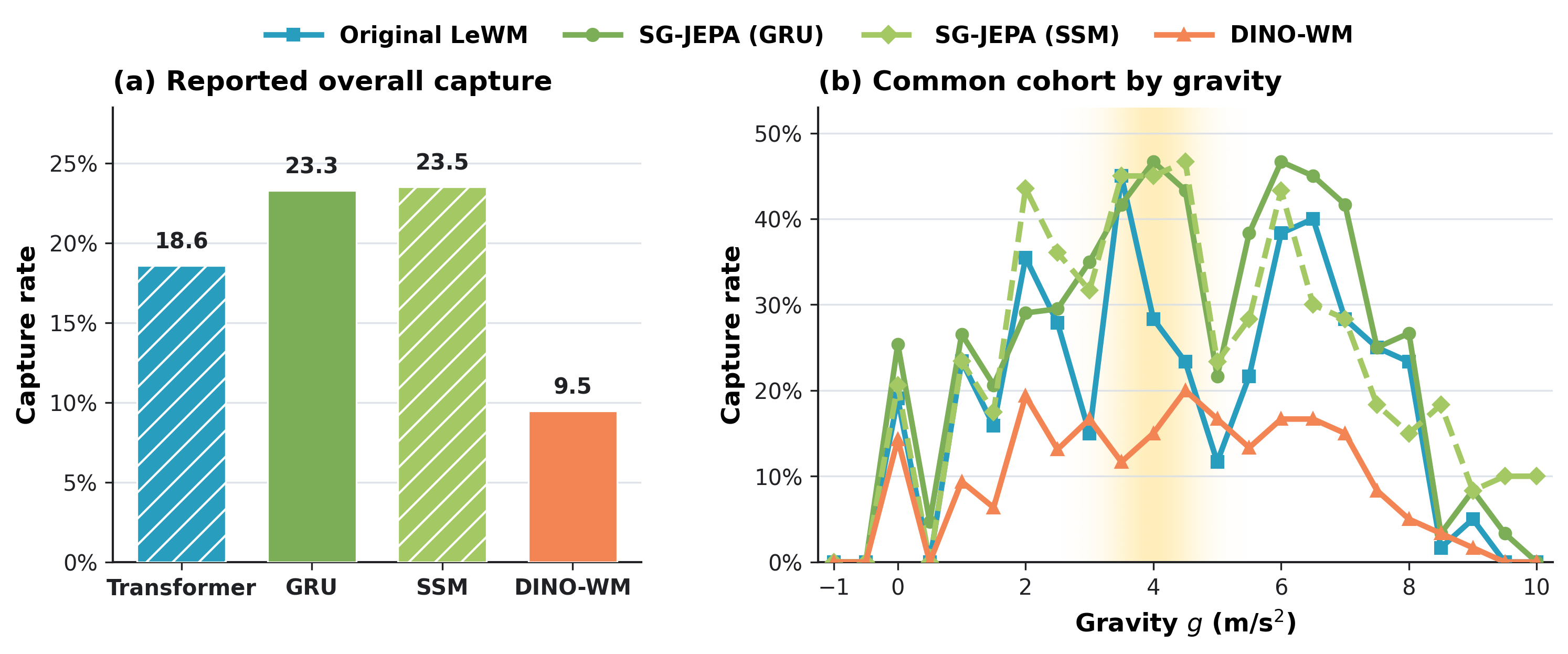}
  \caption{\textbf{Arm-Catcher-Ball predictor-family comparison.}
  (a) Overall capture rate. Results are averaged over 5 rollout executions. (b) Comparison of four methods by gravity. The fading yellow background shows the training distribution \(g\sim\mathcal{N}(4,0.5^2)\).}
  \label{fig:arm-catcher-ball-predictors}
\end{figure*}

\paragraph{Results} In Fig.~\ref{fig:arm-catcher-ball-predictors} we observe that SG-JEPA more than doubles the overall capture rate, improving on DINO-WM by $13.8\%$.  The gain is consistent across all five paired rollouts, ranging from $12.9\%$ to $14.3\%$.  It is also not confined to the center of the training distribution.  Fig.~\ref{fig:main-3d}(d) shows a positive mean advantage at every gravity, with a significant gap from \(g=0\) through \(g=9\).  At the training gravity, SG-JEPA reaches 42.3\% compared with 16.0\% for
DINO-WM.  Most remaining failures occur at the extreme gravity values or run to the
full horizon without a capture. We also compare Original LeWM and both SG-JEPA predictors. Original LeWM uses its one-step Transformer objective with \(\lambda_{\mathrm{SIG}}=0.09\), while the GRU and SSM variants use the five-step SG-JEPA objective with \(\lambda_{\mathrm{SIG}}=0.18\). Both SG-JEPA variants
remain ahead of the baselines overall.

Overall, Arm Catcher Ball is a phase-sensitive interception task: successful control requires the policy to preserve the ball’s velocity and contact phase, not merely its current position. As gravity increases, shorter bounce intervals amplify small timing errors and can shift the predicted trajectory to the wrong contact phase. SG-JEPA’s consistent advantage across a broad gravity range therefore indicates that its representation retains more control-relevant dynamical information under changing physics. Performance nevertheless degrades under negative and extreme gravity values, where the motion departs qualitatively from the training regime.

\subsection{Franka paddle-to-basket results}
\label{app:basket-results}

\begin{table*}[!htbp]
  \centering
  \small
  \setlength{\tabcolsep}{3.5pt}
  \caption{\textbf{Four-method comparison.}   Relaxed success records valid
  post-strike basket contact, while strict entry requires the ball to enter
  the basket.  ``Strict $\mid$ blade'' conditions strict entry on valid
  paddle-blade contact.  All values are percentages.}
  \label{tab:basket-overall}
  \begin{tabular}{lrrrr}
    \toprule
    Representation & Strict entry & Relaxed success & Blade contact & Strict $\mid$ blade \\
    \midrule
    SG-JEPA (GRU) & \textbf{30.36} & 64.72 & 93.70 & \textbf{32.40} \\
    Original LeWM & 28.80 & 62.00 & 90.46 & 31.84 \\
    SG-JEPA (SSM) & 30.00 & 61.22 & 93.54 & 32.07 \\
    DINO-WM & 27.60 & \textbf{67.30} & \textbf{95.86} & 28.79 \\
    \bottomrule
  \end{tabular}%
\end{table*}

In main-text, Figu~\ref{fig:main-3d}(c,e) compares SG-JEPA (GRU) and DINO-WM using five paired simulator-rollout seeds. Here we add results for Original LeWM and SG-JEPA (SSM) in Table~\ref{tab:basket-overall} for completeness.

\paragraph{Results} Both SG-JEPA pipelines give the highest strict-entry rates in the four-method comparison.  DINO-WM hits the ball slightly more often and has the highest
relaxed-success rate, but converts fewer hits into strict entries.  The advantage of SG-JEPA therefore appears after contact, where the policy must produce the appropriate outgoing position and velocity rather than merely intercept the ball.  GRU and SSM are effectively tied in this fixed rollout.

In Fig.~\ref{fig:basket-appendix} we show the per-gravity breakdown for paddle hit and final basket entry success rates respectively.  Paddle-hit rates remain high beyond the
two weakest gravity values, while strict entry varies substantially, reinforcing that post-impact targeting is the harder stage.  The entry ordering also changes sharply outside the center of the training distribution.  The two largest reversals persist under all five rollout seeds: at $g=13.9$~m/s$^2$, DINO-WM averages 61.2\% strict entry and SG-JEPA (GRU) 1.8\%; at $g=16.0$~m/s$^2$, SG-JEPA averages 58.6\% and DINO-WM 0.8\%. The five runs use the same episode IDs, action scaling, and fixed checkpoint hashes, with the rollout seed as the only change.  These are therefore reproducible task-specific reversals, although two gravity values are not enough to support a broader mechanistic conclusion.

\begin{figure*}[!htbp]
  \centering
  \includegraphics[width=\textwidth]{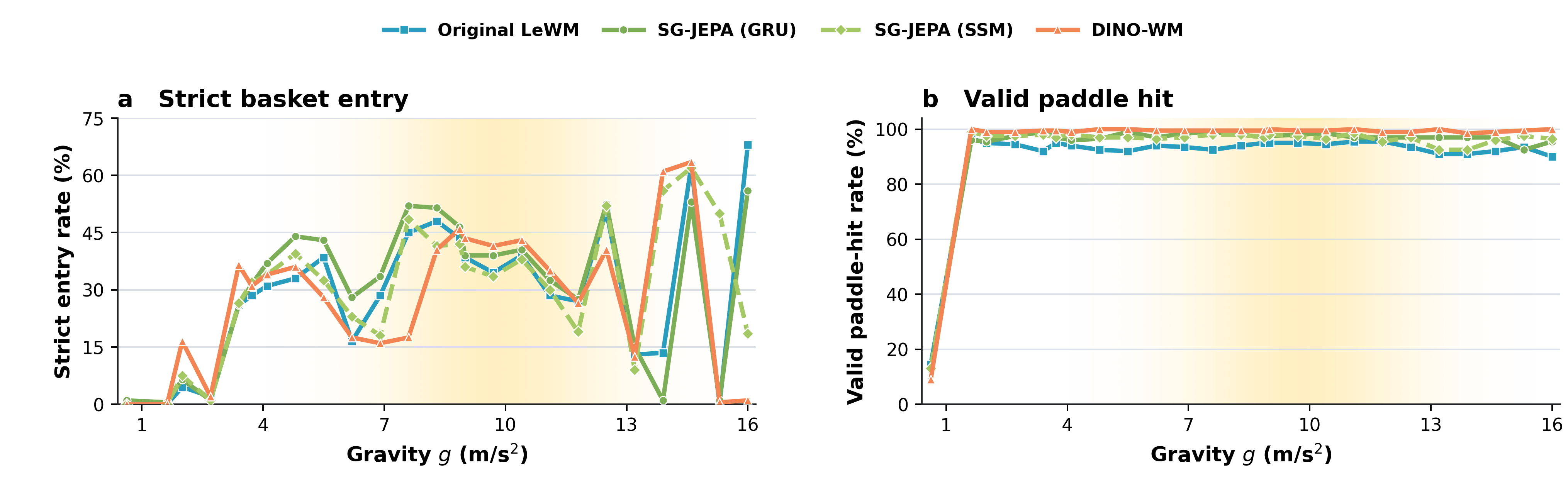}
  \caption{\textbf{Gravity-conditioned Franka Basket outcomes in the
  single-seed four-method comparison.}  (a) Strict basket entry and (b) valid
  paddle-blade contact.  Each point contains 200 paired episodes.  The fading
  yellow background follows the training-gravity density
  $g\sim\mathcal{N}(9.8,2.0^2)$.  All control and evaluation settings are
  matched across methods.}
  \label{fig:basket-appendix}
\end{figure*}

The high paddle-hit rates show that interception is not the main bottleneck. Success instead depends on controlling the ball’s post-impact state: small errors in contact timing, paddle orientation, or incoming velocity can change the outgoing impulse enough to miss the basket. Gravity further changes the required trajectory: weak gravity prolongs flight and can produce overshoot, whereas strong gravity demands a larger and more precisely directed impulse. SG-JEPA’s higher entry rate therefore suggests that its representation better preserves the velocity and contact-phase information needed for gravity-dependent impact control.

\subsection{Arm-Paddle-Ball results}
\label{app:arm-paddle-ball-control-details}

Arm Paddle Ball requires repeated contact control: the policy must keep the ball off the floor and produce at least one qualified bounce.  Main-text Fig.~\ref{fig:main-3d}(c,f) compares the final policy configurations. Here we evaluate the model's capability of understanding the ball dynamics by training probes and evaluated on held-out test set. 

\begin{figure*}[!htbp]
    \centering
    \includegraphics[width=0.9\textwidth]{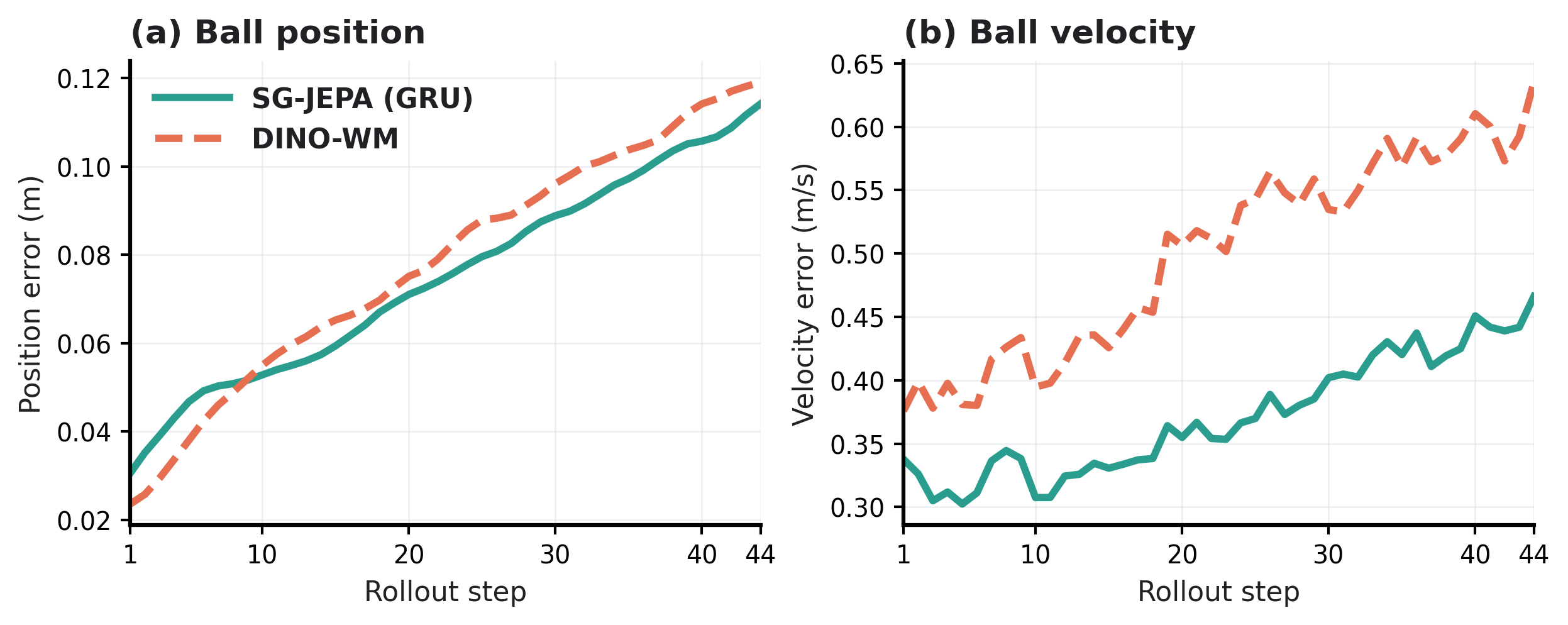}
    \caption{\textbf{Open-loop Arm-Paddle-Ball prediction.}  Mean Euclidean error for (a) ball position and (b) ball velocity over 5,000 held-out episodes.  SG-JEPA has lower position error from $\mathfrak{h}_9$ onward and lower velocity error at every forecast step.}
    \label{fig:arm-paddle-ball-probe-rollouts}
\end{figure*}

\paragraph{Prediction results.}
SG-JEPA has lower position error at 36 of 44 forecast steps and lower velocity error at all 44.  At $\mathfrak{h}_{44}$, position error falls from 0.119 to 0.114\,m and velocity error from 0.635 to 0.466\,m/s. The larger velocity gap is relevant to sustained bouncing, where small phase errors carry into the next contact. This probe comparison is evidence for more accurate open-loop dynamics.

\section{Sparse-Gravity Post-Training Ablation}
\label{app:sparse-gravity-posttraining}

In this ablation we study whether post-training on a subset of the test set with out-of-distribution gravity values improves long-horizon dynamics prediction between an beyond these values. We evaluate 2D right triangle and square datasets and four methods (Original LeWM, SG-JEPA (GRU), SG-JEPA (SSM) and DINO-WM). 

The target post-training set contains 3,200 episodes: 800 episodes at each of the four support gravity values $\mathcal{G}_{\mathrm{sup}}=\{0,2,6,8\}$.  The \emph{target-only} arm uses these 3,200 episodes.  The \emph{mixed} arm adds 5,000 episodes sampled from the original source-training distribution, for 8,200 episodes in total. So for each test value it's roughly $85:15\%$ split. These data are excluded from later evaluation to avoid leakage. Each model is warm-started from its corresponding pretrained source checkpoint and optimized for 20 epochs with a fresh optimizer. We then report rollout errors at horizon 44.  The held-out evaluation partition used here contains 2,500 episodes, with 100 episodes at every gravity in
$\{-2,-1.5,\ldots,9.5,10\}$. For condition $c$, metric $m$, and gravity $g$, let $E_{c,m}(g)$ be the mean error and let $E_{0,m}(g)$ denote the corresponding pretrained error.  The aggregate improvement is
\begin{equation}
    \Delta_c = 100\left[
    1 - \frac{1}{3|\mathcal{G}_{\mathrm{int}}|}
    \sum_{m=1}^{3}\sum_{g\in\mathcal{G}_{\mathrm{int}}}
    \frac{E_{c,m}(g)}{E_{0,m}(g)}
    \right],
    \label{eq:sparse-gravity-improvement}
\end{equation}
where $\mathcal{G}_{\mathrm{int}}=\{0,0.5,\ldots,8\}\setminus\mathcal{G}_{\mathrm{sup}}$ contains the 13 held-out interpolation gravity values. Thus, a positive value indicates a reduction in error relative to the same model's pretrained baseline. We also check whether post-training preserves performance near the source distribution.  A model passes this \emph{source-range retention} test if, for each metric, its mean error over $g\in\{2.5,3,\ldots,5.5\}$ is no more than 5\% above its pretrained error.

\paragraph{Results} Table~\ref{tab:sparse-gravity-epoch20} and
Figures~\ref{fig:sparse-gravity-triangle}--\ref{fig:sparse-gravity-square}
support four main observations.
\begin{itemize}[
    leftmargin=*,
    topsep=2pt,
    itemsep=2pt,
    parsep=0pt,
    partopsep=0pt
]
    \item Mixed post-training reduces aggregate error relative to the
    pretrained model in all eight shape and model combinations.  It also
    outperforms target-only post-training in every combination.  The mean
    reduction increases from 6.85\% to 14.05\%, a gain of 7.21 percentage
    points.
    \item SG-JEPA (GRU) improves by 20.84\% on the triangle and 13.60\% on
    the square.  The largest mixed improvement for each shape is 20.84\% for
    triangle SG-JEPA (GRU) and 20.81\% for square Original LeWM.
    \item The per-gravity curves show that adaptation does more than improve
    the four observed support gravity values.  Both post-training conditions reduce
    error at many of the unobserved gravity values between them, and the largest
    separations from the pretrained curves generally occur toward the low- and
    high-gravity ends.  The effect is therefore consistent with interpolation
    across the adapted gravity range.
    \item The gains are not uniform across all metrics.  Target-only DINO-WM
    slightly worsens triangle rotation by 0.22\%, while Mixed DINO-WM worsens
    square velocity by 4.13\%.  Around the central source range, where the
    pretrained curves are already lower, the changes are smaller and more
    method dependent.  Four of the eight Mixed configurations pass the
    source-range retention test; none of the Target-only configurations pass
    it.
\end{itemize}

The consistent advantage of Mixed over Target-only suggests that a modest
amount of data at a few target gravity values can improve prediction at the
unobserved gravity values between them.  Source replay also gives stronger aggregate
performance than using the target trajectories alone.  In settings where the
governing dynamics change, this offers an alternative to training a new world
model from scratch: collect a small target set and post-train with source
replay.

\begin{table*}[!htbp]
    \centering
    \small
    \setlength{\tabcolsep}{5pt}
    \begin{tabular}{lllrrrrc}
        \toprule
        Shape & Predictor & Arm & Overall & Position & Velocity & Rotation &
        \shortstack{Source\\retained} \\
        \midrule
        Triangle & Transformer & Target-only & 4.22 & 5.30 & 5.65 & 1.70 & no \\
         & Transformer & Mixed       & 13.80 & 17.29 & 13.64 & 10.47 & no \\
         & GRU         & Target-only & 17.05 & 17.04 & 20.83 & 13.29 & no \\
         & GRU         & Mixed       & 20.84 & 19.69 & 21.63 & 21.20 & yes \\
         & SSM         & Target-only & 6.61 & 2.73 & 11.29 & 5.80 & no \\
         & SSM         & Mixed       & 14.53 & 12.93 & 15.93 & 14.74 & yes \\
         & DINO-WM     & Target-only & 1.79 & 2.92 & 2.66 & $-0.22$ & no \\
         & DINO-WM     & Mixed       & 6.12 & 5.38 & 7.10 & 5.89 & yes \\
        \midrule
        Square & Transformer & Target-only & 10.41 & 11.46 & 8.21 & 11.58 & no \\
         & Transformer & Mixed       & 20.81 & 25.43 & 18.48 & 18.51 & no \\
         & GRU         & Target-only & 5.39 & 7.99 & 6.61 & 1.56 & no \\
         & GRU         & Mixed       & 13.60 & 9.09 & 9.60 & 22.12 & yes \\
         & SSM         & Target-only & 3.71 & 6.96 & 2.51 & 1.66 & no \\
         & SSM         & Mixed       & 12.15 & 17.84 & 10.95 & 7.66 & no \\
         & DINO-WM     & Target-only & 5.59 & 15.85 & 0.05 & 0.86 & no \\
         & DINO-WM     & Mixed       & 10.55 & 21.43 & $-4.13$ & 14.35 & no \\
        \midrule
        \multicolumn{3}{l}{Target-only mean} & 6.85 & 8.78 & 7.22 & 4.53 & --- \\
        \multicolumn{3}{l}{Mixed mean}       & 14.05 & 16.14 & 11.65 & 14.37 & --- \\
        \bottomrule
    \end{tabular}
    \caption{Improvement (\%) on the 13 held-out interpolation gravity values, relative to each predictor's pretrained baseline.  ``Overall''
    is $\Delta_c$ from Eq.~(\ref{eq:sparse-gravity-improvement}); ``Source retained''
    reports the source-range retention test defined above.}
    \label{tab:sparse-gravity-epoch20}
\end{table*}

Figures~\ref{fig:sparse-gravity-triangle} and
\ref{fig:sparse-gravity-square} provide the complete per-gravity breakdown.

\begin{figure*}[!htbp]
    \centering
    \includegraphics[width=\textwidth]{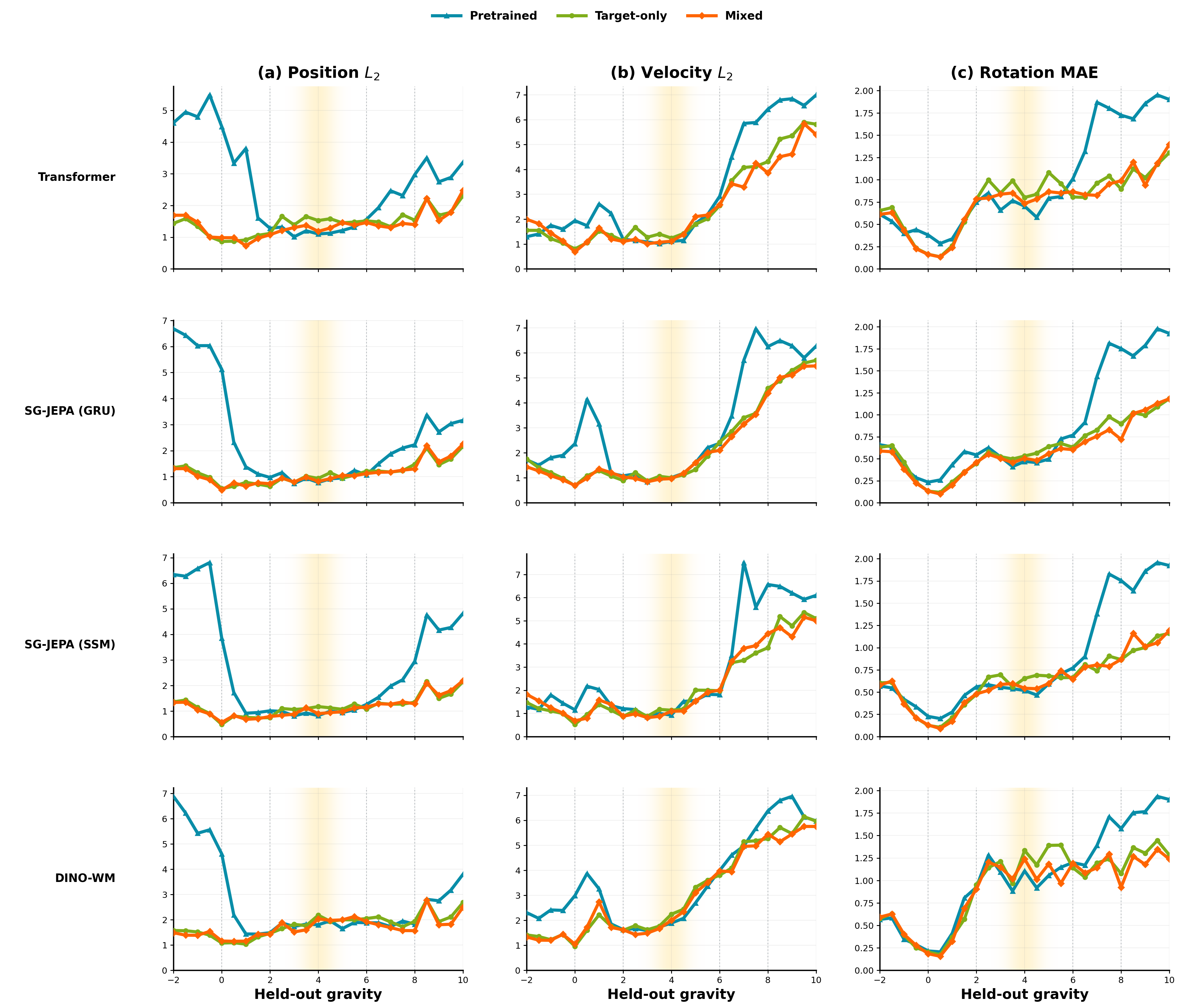}
    \caption{Right-triangle horizon-44 errors.  Rows correspond to
    the four predictors and columns to position, velocity, and cumulative
    rotation.  Curves show the pretrained baseline, target-only post-training,
    and mixed post-training over all 25 evaluation gravity values.}
    \label{fig:sparse-gravity-triangle}
\end{figure*}

\begin{figure*}[!htbp]
    \centering
    \includegraphics[width=\textwidth]{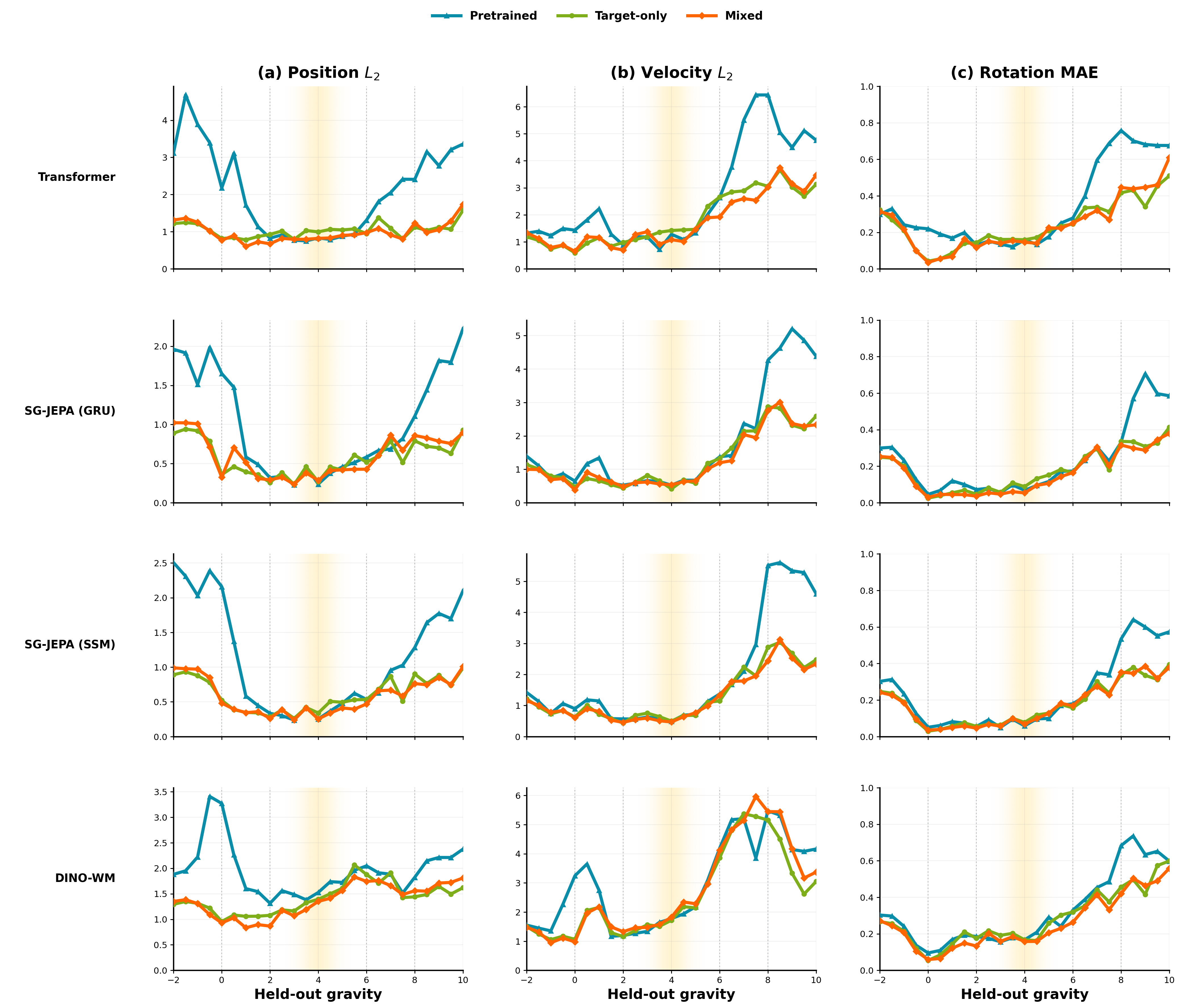}
    \caption{Square horizon-44 errors at epoch 20, using the same layout and
    conventions as Fig.~\ref{fig:sparse-gravity-triangle}.}
    \label{fig:sparse-gravity-square}
\end{figure*}

\section{Theory of prediction under changing physical laws}
\label{app:theory}

This appendix develops the theory used in Section~\ref{sec:dynamics-advantage}.  The analysis separates three questions. First, what makes a latent state locally predictive when gravity changes? Second, how does a local error enter a recursive rollout?  Third, why can a multi-step objective prefer a different representation from a one-step objective, e.g. one-step favours DINO-WM but after rollout SG-JEPA has advantage? We answer these questions with a linear feature model. The model is simple enough to analyze exactly, but it is not an identification of the neural encoder or the GRU with linear operators.

We first summarize the main results.
\begin{itemize}[leftmargin=1.5em,itemsep=0.25em,topsep=0.2em]
\item The one-step latent error separates into predictor error on learned features, information discarded by the representation that still affects the next state, and a conditionally zero-mean residual (Definition~\ref{def:appendix-local-defect} and Proposition~\ref{prop:appendix-teacher-forcing}).
\item If the true and learned dynamics share the same dependence on gravity, average training error bounds the systematic one-step test error, scaled by gravity coverage (Theorem~\ref{thm:appendix-local-ood-bound}). For affine free flight, the factor is $1+(g_\star-\mu_{\rm tr})^2/\sigma_{\rm tr}^2$.
\item Rollout recursion shows how local error and transition residual are transformed by later predicted dynamics.  The corresponding physical-state bound also retains readout error and the transition residual (Theorems~\ref{thm:appendix-exact-recursion} and~\ref{thm:appendix-central-bound}).
\item A multi-step loss applies a different quadratic geometry to one-step transition error.  Two representations can therefore exchange rank when the training horizon changes (Theorem~\ref{thm:appendix-ranking-reversal}).
\end{itemize}

\subsection{Setup: the model}
\label{app:theory-setup}

The model mirrors the training objective and fixes the notation used in the
proofs.

\paragraph{Implementation-aligned objective.}
Let $e_\vartheta$ be the shared trainable encoder and projector, and denote the latent vector $z_t=e_\vartheta(o_t)\in\mathbb R^{d_z}$ for observation $o_t$.  Let $c_t=q(u_t,g)$ be the action-conditioning of the action $u_t$ and gravity $g$.  The predictor receives a context of $H\geq1$ latent states.  Starting from a true context, it is applied recursively for $K\geq1$ training steps:
\begin{align}
\widehat z_{t+k}&=f\!\left(
z^{\rm roll}_{t+k-H:t+k-1},
c_{t+k-H:t+k-1}
\right),
\qquad k=1,\ldots,K, \\
z^{\rm roll}_{t+j}
&=
\begin{cases}
z_{t+j}, & j\leq 0,\\
\widehat z_{t+j}, & j\geq 1.
\end{cases}
\end{align}
For normalized weights
\begin{equation}
w_k=\frac{\gamma^{k-1}}{\sum_{j=1}^{K}\gamma^{j-1}},
\qquad
\gamma>0,
\qquad
w_k\geq0,
\qquad
\sum_{k=1}^{K}w_k=1,
\end{equation}
the latent rollout loss is
\begin{equation}
\mathcal L_K(\vartheta,f)
=
\sum_{k=1}^{K}w_k\,
\mathbb E\!\left[
\left\lVert\widehat z_{t+k}-z_{t+k}\right\rVert_2^2
\right].
\label{eq:appendix-rollout-objective}
\end{equation}
The expectation is over the training clips, including initial states, gravity values, action sequences, and stochastic residuals when present. The predictor helps determine which representations receive low training loss.  We express this dependence through the profiled representation risk
\begin{equation}
\mathcal J_{K,H,\mathcal F}(\vartheta)
=
\inf_{f\in\mathcal F_H}\mathcal L_K(\vartheta,f)
+\lambda_{\rm SIG}\mathcal R_{\rm SIG}(\vartheta).
\label{eq:appendix-profiled-risk}
\end{equation}
Here $\mathcal F_H$ is the predictor family for an $H$-step context. Unless marked as empirical, $\mathcal J_{K,H,\mathcal F}$ denotes a population risk; $\mathcal R_{\rm SIG}$ is the expected implemented batch statistic, with the expectation taken over data batches and random directions.  The idealized population characteristic-function calculation is defined separately in Section~\ref{app:theory-sigreg}.

For the core argument, let scalar gravity $g\in\mathbb R$ be fixed within an episode. We assume that the feature state $\phi_t\in\mathbb R^{d_\phi}$ follows the linear transition
\begin{equation}
\phi_{t+1}=T(g)\phi_t+\xi_{t+1}
\qquad
\mathbb E[\xi_{t+1}\mid\phi_t,g]=0.
\label{eq:appendix-feature-dynamics}
\end{equation}
When actions are present explicitly, the conditioning set in this equation also includes the known action input. The matrix $T(g)\in\mathbb R^{d_\phi\times d_\phi}$ gives the transition under gravity $g$, and $\xi_{t+1}$ is the remaining conditionally zero-mean residual.  Systematic approximation error does not belong in $\xi_{t+1}$; it must enter a separate deterministic term, such as the approximate-basis remainder introduced below. Section~\ref{app:theory-history} gives the exact finite-history form with actions. We model the learned representation as
\begin{equation}
z_t=W\phi_t,
\qquad
W\in\mathbb R^{d_z\times d_\phi},
\qquad
d_z\leq d_\phi,
\qquad
WW^\top=I_{d_z}.
\label{eq:appendix-linear-encoder}
\end{equation}
Row orthonormality fixes the latent coordinate scale inside the model. The learned one-step predictor is represented by $\widehat A(g)\in\mathbb R^{d_z\times d_z}$.  For the physical-state bounds, we assume that the physical quantity of interest is linear in the chosen feature coordinates:
\begin{equation}
s_t=C_{\rm phys}\phi_t,
\label{eq:appendix-physical-readout}
\end{equation}
where $C_{\rm phys}\in\mathbb R^{d_s\times d_\phi}$.  Let $r:\mathbb R^{d_z}\to\mathbb R^{d_s}$ be a frozen probe.  When the probe is linear, we write $r(z)=Rz$ with $R\in\mathbb R^{d_s\times d_z}$.  The experiments use learned probes, some of which read a short latent history.  For a history probe, the input to $r$ is the corresponding stacked latent window.  The exact selector for a linear history readout is given in Section~\ref{app:theory-history}.  The probe identities below hold for any frozen probe with a compatible input, while the matrix bounds use the linear specialization. $s_t$ is expressed in the fixed physical coordinates in which error is evaluated. We use $H$ only for context length, $K$ only for the training rollout length, and $\mathfrak{h}$ for a generic evaluation horizon.  Vector norms are Euclidean.  Matrix norms are Frobenius norms unless a subscript $2$ marks the spectral norm.

\paragraph{Fixed-gravity composition.}
For a fixed gravity and sampling interval, define
\begin{align}
S_g(0)&=I_{d_\phi},
&
S_g(\mathfrak h)&=T(g)^{\mathfrak h},\
\widehat S_g(0)&=I_{d_z},
&
\widehat S_g(\mathfrak h)&=\widehat A(g)^{\mathfrak h}.
\label{eq:appendix-semigroup-family}
\end{align}
For $k,\ell\in\mathbb N_0$,
\begin{equation}
S_g(k+\ell)=S_g(\ell)S_g(k),
\qquad
\widehat S_g(k+\ell)=\widehat S_g(\ell)\widehat S_g(k).
\label{eq:appendix-semigroup-law}
\end{equation}
Each gravity therefore indexes a pair of discrete evolution semigroups; gravity
is an index, not the composition variable. In the residual-free model these
operators describe the realized evolution. With conditionally zero-mean
residuals, the true powers describe conditional-mean propagation, while
realized trajectories retain the residual terms derived below. The rollout
loss trains the accuracy of the first $K$ learned iterates rather than the
algebraic composition identity itself.

\subsection{Local prediction error under a fixed gravity}
\label{app:theory-local-defect}

A latent state can fail locally for two reasons.  It may discard information needed for the next transition, or the predictor may evolve the retained information incorrectly.  The operators below separate these terms. Define the orthogonal projector onto the row space of $W$ and the two induced
operators
\begin{align}
P_W=W^\top W,\quad A_W(g)=WT(g)W^\top, \quad C_W(g)=WT(g)(I-P_W).
\label{eq:appendix-projected-operators}
\end{align}
Since $P_W\phi_t=W^\top z_t$,
\begin{align}
WT(g)\phi_t
&=WT(g)P_W\phi_t+WT(g)(I-P_W)\phi_t\\
&=A_W(g)z_t+C_W(g)\phi_t.
\label{eq:appendix-projector-decomposition}
\end{align}
Equivalently,
$C_W(g)=WT(g)-A_W(g)W$.  Encoding
Eq.~\ref{eq:appendix-feature-dynamics} therefore gives
\begin{equation}
z_{t+1}
=A_W(g)z_t+C_W(g)\phi_t+W\xi_{t+1}.
\label{eq:appendix-true-latent-transition}
\end{equation}

\begin{definition}[Local law-conditioned error]
\label{def:appendix-local-defect}
Let $E_A(g)=A_W(g)-\widehat A(g)$. The state-dependent conditional-mean error injected when the predictor starts from the true latent state is
\begin{equation}
\delta_t(g)=E_A(g)z_t+C_W(g)\phi_t.
\label{eq:appendix-local-defect}
\end{equation}
The term $E_A(g)z_t$ is predictor error on the state. The term $C_W(g)\phi_t$ is the closure error realized at state $\phi_t$: discarded features still affect the next latent.  The operator $C_W(g)$ can be nonzero even on trajectories for which this realized vector happens to vanish.
\end{definition}

We call the representation \emph{predictively closed} at gravity $g$ when
$C_W(g)=0$.  In that case, the conditional mean of the next latent depends
only on the current latent.  Residual variation can still make individual
transitions unpredictable.

In the residual-free linear model, predictive closure together with an exact
latent operator, $\widehat A(g)=A_W(g)$, gives
$WT(g)=\widehat A(g)W$. Hence
\begin{equation}
W S_g(\mathfrak h)
=\widehat S_g(\mathfrak h)W
\qquad
\text{for every }\mathfrak h\geq0.
\label{eq:appendix-exact-intertwining}
\end{equation}
With transition residuals, this intertwining identity applies to the
conditional-mean evolution; individual trajectories still contain the
propagated $W\xi_{t+1}$ terms.

\begin{proposition}[What teacher forcing measures]
\label{prop:appendix-teacher-forcing}
Let
$\widehat z_{t+1}^{\rm TF}=\widehat A(g)z_t$.  Then
\begin{align}
z_{t+1}-\widehat z_{t+1}^{\rm TF}
&=\delta_t(g)+W\xi_{t+1},
\label{eq:appendix-tf-latent}\\
Rz_{t+1}-R\widehat z_{t+1}^{\rm TF}
&=R\delta_t(g)+RW\xi_{t+1},
\label{eq:appendix-tf-decoded}\\
s_{t+1}-R\widehat z_{t+1}^{\rm TF}
&=(C_{\rm phys}-RW)\phi_{t+1}
+R\delta_t(g)+RW\xi_{t+1}.
\label{eq:appendix-tf-physical}
\end{align}
\end{proposition}

\begin{proof}
Subtract $\widehat A(g)z_t$ from
Eq.~\ref{eq:appendix-true-latent-transition}.  This gives
Eq.~\ref{eq:appendix-tf-latent}; applying $R$ gives
Eq.~\ref{eq:appendix-tf-decoded}.  For the physical identity, add and
subtract $RW\phi_{t+1}=Rz_{t+1}$:
\begin{align}
s_{t+1}-R\widehat z_{t+1}^{\rm TF}
&=C_{\rm phys}\phi_{t+1}-RW\phi_{t+1}
+R(z_{t+1}-\widehat z_{t+1}^{\rm TF}),
\end{align}
and substitute Eq.~\ref{eq:appendix-tf-latent}.
\end{proof}

For any frozen, possibly nonlinear probe $r$, the decoded-latent comparison is
$r(z_{t+1})-r(\widehat z_{t+1}^{\rm TF})$.  The readout residual
$s_{t+1}-r(z_{t+1})$ cancels because the same probe is applied to both latents.
If $r$ is $L_r$-Lipschitz on the line segment joining the two latents, then
\begin{equation}
\left\lVert
r(z_{t+1})-r(\widehat z_{t+1}^{\rm TF})
\right\rVert_2
\leq
L_r\left\lVert\delta_t(g)+W\xi_{t+1}\right\rVert_2.
\end{equation}

The teacher-forced diagnostic in Figure~\ref{fig:dynamics-advantage} uses this
same-probe comparison, so the true-state probe floor cancels for both linear and
nonlinear probes.  The remaining error still combines the two terms in
$\delta_t(g)$ with the transition residual.  Under the conditional-zero-mean
assumption, the expected cross term also vanishes after any fixed linear probe:
\begin{equation}
\mathbb E\!\left[
\delta_t(g)^\top R^\top RW\xi_{t+1}\mid g
\right]
=0.
\end{equation}
This follows from the tower property because $\delta_t(g)$ is measurable with
respect to $(\phi_t,g)$ and
$\mathbb E[\xi_{t+1}\mid\phi_t,g]=0$.
There is no corresponding cancellation guarantee for a nonlinear probe, and
the result says nothing about the readout-residual cross term in
Eq.~\ref{eq:appendix-physical-mse-decomposition}.  The diagnostic
cannot, by itself, assign the difference to
$E_A(g)z_t$ or $C_W(g)\phi_t$ separately.

\subsection{How gravity coverage controls test error}
\label{app:theory-law-transfer}

The semigroup law composes evolution over $\mathfrak h$ at fixed $g$.
Generalization across gravity is a separate question: it depends on how the
one-step operator varies with $g$ and which law directions training covers.
Low error on the training gravity values alone gives no bound at an unseen gravity,
so we state the required shared structure explicitly.

\paragraph{Shared law basis}
Suppose a column vector $\psi(g)=(\psi_1(g),\ldots,\psi_m(g))^\top$ describes how gravity enters the
transition, with finite second moments under the training distribution, and
\begin{align}
T(g)=\sum_{k=1}^{m}\psi_k(g)T_k,\quad & A_W(g)=\sum_{k=1}^{m}\psi_k(g)A_k,\\
C_W(g)=\sum_{k=1}^{m}\psi_k(g)C_k,\quad & \widehat A(g)=\sum_{k=1}^{m}\psi_k(g)\widehat A_k.
\label{eq:appendix-shared-law-basis}
\end{align}
The first expansion implies the next two with $A_k=WT_kW^\top$ and
$C_k=WT_k(I-P_W)$. The expansion for $\widehat A(g)$ is an additional
structural assumption on the learned predictor. Equation~\eqref{eq:appendix-shared-law-basis}
is therefore an exact-model assumption for the result below, not a consequence
of providing $g$ as an input. The neural predictor is not constrained to
satisfy it exactly.

Let $P_{\rm tr}$ be the training distribution over gravity and define
\begin{equation}
M_\psi
=\mathbb E_{g\sim P_{\rm tr}}[\psi(g)\psi(g)^\top].
\label{eq:appendix-law-moment}
\end{equation}
We assume $M_\psi\succ0$, so the training distribution covers every direction
in the chosen basis.  The \emph{law-coverage factor} at a test gravity
$g_\star$ is
\begin{equation}
\mathcal L_{\rm law}(g_\star)
=\psi(g_\star)^\top M_\psi^{-1}\psi(g_\star).
\label{eq:appendix-law-coverage}
\end{equation}
A large value means that the test law lies in a weakly covered direction of
the training design.

\begin{lemma}[Transfer for a matrix-valued law family]
\label{lem:appendix-matrix-transfer}
Let
$E(g)=\sum_{k=1}^{m}\psi_k(g)E_k$ be any matrix-valued residual and set $\epsilon_E^2=\mathbb E_{g\sim P_{\rm tr}}\left[\lVert E(g)\rVert_F^2\right]$. If $M_\psi\succ0$, then
\begin{equation}
\lVert E(g_\star)\rVert_F^2
\leq
\mathcal L_{\rm law}(g_\star)\epsilon_E^2.
\label{eq:appendix-matrix-transfer}
\end{equation}
\end{lemma}

\begin{proof}
Let $\mathsf{E}$ be the matrix whose $k$th column is
$\operatorname{vec}(E_k)$.  Then
$\operatorname{vec}(E(g))=\mathsf{E}\psi(g)$ and
\begin{align}
\epsilon_E^2
&=\mathbb E\!\left[\psi(g)^\top \mathsf{E}^\top \mathsf{E}\psi(g)\right]\\
&=\operatorname{tr}(\mathsf{E} M_\psi \mathsf{E}^\top)
=\lVert \mathsf{E} M_\psi^{1/2}\rVert_F^2.
\label{eq:appendix-training-residual-factor}
\end{align}
At $g_\star$,
\begin{align}
\lVert E(g_\star)\rVert_F
&=\left\lVert
\mathsf{E} M_\psi^{1/2}M_\psi^{-1/2}\psi(g_\star)
\right\rVert_2\\
&\leq
\lVert \mathsf{E} M_\psi^{1/2}\rVert_2
\left\lVert M_\psi^{-1/2}\psi(g_\star)\right\rVert_2\\
&\leq
\lVert \mathsf{E} M_\psi^{1/2}\rVert_F
\sqrt{\mathcal L_{\rm law}(g_\star)}.
\end{align}
Substitute Eq.~\ref{eq:appendix-training-residual-factor} and square
both sides.
\end{proof}

Define the average training-law errors
\begin{align}
\epsilon_{\rm cl}^2
&=\mathbb E_{g\sim P_{\rm tr}}
\left[\lVert C_W(g)\rVert_F^2\right],\\
\epsilon_{\rm op}^2
&=\mathbb E_{g\sim P_{\rm tr}}
\left[\lVert A_W(g)-\widehat A(g)\rVert_F^2\right].
\label{eq:appendix-training-errors}
\end{align}
The first measures how far the representation is from predictive closure.  The
second measures how well the predictor fits the transition retained by the
representation.  

\begin{theorem}[Local error at an unseen gravity]
\label{thm:appendix-local-ood-bound}
Under the shared-basis assumptions in
Eq.~\ref{eq:appendix-shared-law-basis}, suppose
$\lVert z_t\rVert_2\leq B_z$ and
$\lVert\phi_t\rVert_2\leq B_\phi$.  Then
\begin{equation}
\lVert\delta_t(g_\star)\rVert_2
\leq
\sqrt{\mathcal L_{\rm law}(g_\star)}
\left(B_z\epsilon_{\rm op}+B_\phi\epsilon_{\rm cl}\right).
\label{eq:appendix-local-ood-bound}
\end{equation}
\end{theorem}

\begin{proof}
Apply Lemma~\ref{lem:appendix-matrix-transfer} first to $C_W(g)$ and then to
$E_A(g)=A_W(g)-\widehat A(g)$.  This gives
\begin{align}
\lVert C_W(g_\star)\rVert_F
&\leq
\sqrt{\mathcal L_{\rm law}(g_\star)}\epsilon_{\rm cl},\\
\lVert E_A(g_\star)\rVert_F
&\leq
\sqrt{\mathcal L_{\rm law}(g_\star)}\epsilon_{\rm op}.
\end{align}
Using Definition~\ref{def:appendix-local-defect}, the triangle inequality, and
$\lVert Mx\rVert_2\leq\lVert M\rVert_F\lVert x\rVert_2$ yields
\begin{align}
\lVert\delta_t(g_\star)\rVert_2
&\leq
\lVert E_A(g_\star)\rVert_F\lVert z_t\rVert_2
+\lVert C_W(g_\star)\rVert_F\lVert\phi_t\rVert_2,
\end{align}
which gives the stated bound.
\end{proof}
The theorem is a sufficient transfer result. As a concrete example, for collision-free ballistic motion, gravity enters the transition affinely. This gives a concrete two-dimensional law basis.

\begin{corollary}[Coverage factor for affine gravity]
\label{cor:appendix-affine-coverage}
Let $\psi(g)=(1,g)^\top$.  If the training gravity has mean $\mu_{\rm tr}$ and
variance $\sigma_{\rm tr}^2>0$, then
\begin{equation}
\mathcal L_{\rm law}(g_\star)
=1+\frac{(g_\star-\mu_{\rm tr})^2}{\sigma_{\rm tr}^2}.
\label{eq:appendix-affine-coverage}
\end{equation}
\end{corollary}

\begin{proof}
For this basis,
\begin{equation}
M_\psi
=
\begin{bmatrix}
1&\mu_{\rm tr}\\
\mu_{\rm tr}&\mu_{\rm tr}^2+\sigma_{\rm tr}^2
\end{bmatrix}.
\end{equation}
Its determinant is $\sigma_{\rm tr}^2$, so
\begin{equation}
M_\psi^{-1}
=\frac{1}{\sigma_{\rm tr}^2}
\begin{bmatrix}
\mu_{\rm tr}^2+\sigma_{\rm tr}^2&-\mu_{\rm tr}\\
-\mu_{\rm tr}&1
\end{bmatrix}.
\end{equation}
Direct multiplication gives
\begin{align}
\begin{bmatrix}1&g_\star\end{bmatrix}
M_\psi^{-1}
\begin{bmatrix}1\\g_\star\end{bmatrix}
&=\frac{
\mu_{\rm tr}^2+\sigma_{\rm tr}^2
-2\mu_{\rm tr}g_\star+g_\star^2
}{\sigma_{\rm tr}^2}\\
&=1+\frac{(g_\star-\mu_{\rm tr})^2}{\sigma_{\rm tr}^2}.
\end{align}
\end{proof}

The moments in Eq.~\ref{eq:appendix-affine-coverage} are the mean and
variance of the actual training distribution.  The planar experiments draw
$g=\max\{G,0.1\}$ with $G\sim\mathcal N(4,0.5^2)$, so the exact factor uses
the moments of this clipped distribution. 

\begin{proposition}[Free-flight dynamics remain affine in gravity]
\label{prop:appendix-ballistic-affine}
For time step $\Delta>0$ and homogeneous state
$\phi_t=(p_t,v_t,1)^\top$, let
\begin{equation}
T_g=
\begin{bmatrix}
1&\Delta&-\tfrac12g\Delta^2\\
0&1&-g\Delta\\
0&0&1
\end{bmatrix}.
\end{equation}
For every integer $\mathfrak{h}\geq1$,
\begin{equation}
T_g^{\mathfrak{h}}=
\begin{bmatrix}
1&\mathfrak{h}\Delta&-\tfrac12g(\mathfrak{h}\Delta)^2\\
0&1&-g\mathfrak{h}\Delta\\
0&0&1
\end{bmatrix}.
\label{eq:appendix-ballistic-h-step}
\end{equation}
Thus the transition remains affine in $g$ at every free-flight horizon.
\end{proposition}

\begin{proof}
The formula holds at $\mathfrak{h}=1$.  Assume it holds at $\mathfrak{h}$.
Multiplying $T_g^{\mathfrak{h}}$ by
$T_g$ gives the next velocity coefficient
$-g(\mathfrak{h}+1)\Delta$.  The homogeneous contribution to position is
\begin{align}
-\tfrac12g(\mathfrak{h}\Delta)^2-\mathfrak{h}g\Delta^2-\tfrac12g\Delta^2
&=-\tfrac12g(\mathfrak{h}+1)^2\Delta^2.
\end{align}
The remaining entries also match
Eq.~\ref{eq:appendix-ballistic-h-step} with $\mathfrak{h}$ replaced by
$\mathfrak{h}+1$.
Induction completes the proof.
\end{proof}

This proposition applies only while the trajectory remains in free flight.
It is a statement about the displayed physical-coordinate feature state; an
arbitrary learned feature map need not inherit the same affine dependence.
In more than one spatial dimension, the same calculation applies to the
gravity-aligned position and velocity components, while tangential components
have no gravity forcing in this idealized free-flight model.
Contact can change the transition branch, as discussed in
Section~\ref{app:theory-events}.

\subsubsection{When the coverage result does not apply}

The coverage theorem needs both a covered basis and a predictor that uses that
basis.

\begin{proposition}[Missing law directions prevent a uniform bound]
\label{prop:appendix-no-coverage}
Suppose $M_\psi$ is singular and there is an
$a\in\ker(M_\psi)$ such that $a^\top\psi(g_\star)\neq0$.  Then a residual
family can have zero average training error and nonzero error at $g_\star$.
\end{proposition}

\begin{proof}
Because
$a^\top M_\psi a=\mathbb E[(a^\top\psi(g))^2]=0$, we have
$a^\top\psi(g)=0$ almost surely under $P_{\rm tr}$.  Choose any nonzero matrix
$E_0$ and set
$E(g)=(a^\top\psi(g))E_0$.  This residual vanishes almost surely on the
training distribution, but it is nonzero at $g_\star$.
\end{proof}

A singular $M_\psi$ is not automatically fatal.  If $\psi(g_\star)$ lies in
the covered range of $M_\psi$, the same type of statement can use the inverse
restricted to that range.  The failure above occurs when the test law has a
component in a direction that training never observes.

\begin{corollary}[Transfer from a finite gravity design]
\label{cor:appendix-finite-design}
Let $g_1,\ldots,g_n$ be observed gravity values and let
$\Psi\in\mathbb R^{n\times m}$ have rows $\psi(g_i)^\top$.  For a residual
$E(g)=\sum_k\psi_k(g)E_k$, suppose $\Psi$ has full column rank.  Then
\begin{equation}
\lVert E(g_\star)\rVert_F^2
\leq
\psi(g_\star)^\top(\Psi^\top\Psi)^{-1}\psi(g_\star)
\sum_{i=1}^{n}\lVert E(g_i)\rVert_F^2.
\label{eq:appendix-finite-design}
\end{equation}
\end{corollary}

\begin{proof}
Let $Y\in\mathbb R^{n\times q}$ have row
$\operatorname{vec}(E(g_i))^\top$, where $q$ is the number of entries in each
residual matrix.  If $D\in\mathbb R^{m\times q}$ stacks the vectorized
coefficient matrices, then $Y=\Psi D$.  Full column rank gives
$D=(\Psi^\top\Psi)^{-1}\Psi^\top Y$.  Hence
\begin{equation}
\operatorname{vec}(E(g_\star))^\top
=\psi(g_\star)^\top(\Psi^\top\Psi)^{-1}\Psi^\top Y.
\end{equation}
Cauchy-Schwarz bounds its squared norm by the squared norm of the row vector
times $\lVert Y\rVert_F^2$.  The row-vector norm is
$\psi(g_\star)^\top(\Psi^\top\Psi)^{-1}\psi(g_\star)$, which proves the
claim.
\end{proof}

\begin{theorem}[Conditioning alone does not guarantee extrapolation]
\label{thm:appendix-conditioning-limit}
Let $g_1,\ldots,g_n$ be distinct observed gravity values and let $g_\star$ be
different from all of them.  For any analytic matrix-valued function
$\widehat A_0(g)$ and any matrix $Q$, there is an analytic matrix-valued function
$\widehat A_1(g)$ such that
\begin{equation}
\widehat A_1(g_i)=\widehat A_0(g_i)
\quad\text{for every }i,
\qquad
\widehat A_1(g_\star)=\widehat A_0(g_\star)+Q.
\end{equation}
\end{theorem}

\begin{proof}
Define
\begin{equation}
q(g)=
\frac{\prod_{i=1}^{n}(g-g_i)}
{\prod_{i=1}^{n}(g_\star-g_i)}.
\end{equation}
The denominator is nonzero, $q(g_i)=0$, and $q(g_\star)=1$.  The function
$\widehat A_1(g)=\widehat A_0(g)+q(g)Q$ is analytic and has the required values.
\end{proof}

Even an analytic conditional model can agree at every observed gravity and
differ arbitrarily at an unseen one.  Supplying gravity is therefore not, by
itself, an extrapolation guarantee.

If a residual has the approximate form
\begin{equation}
E(g)=\sum_{k=1}^{m}\psi_k(g)E_k+\rho_E(g),
\end{equation}
write
\begin{equation}
E_{\rm basis}(g)=\sum_{k=1}^{m}\psi_k(g)E_k,
\qquad
\epsilon_{E,{\rm basis}}^2
=\mathbb E_{g\sim P_{\rm tr}}
\lVert E_{\rm basis}(g)\rVert_F^2.
\end{equation}
Lemma~\ref{lem:appendix-matrix-transfer} applies to $E_{\rm basis}$ and gives
\begin{equation}
\lVert E(g_\star)\rVert_F
\leq
\sqrt{\mathcal L_{\rm law}(g_\star)}\epsilon_{E,{\rm basis}}
+\lVert\rho_E(g_\star)\rVert_F.
\label{eq:appendix-approximate-basis-bound}
\end{equation}
Here $\epsilon_{E,{\rm basis}}$ is not automatically bounded by the observed
training residual: the basis component and $\rho_E$ may cancel at the training
gravity values.  For $E_A$, call the remainder $\rho_{\rm op}$; for $C_W$, call it
$\rho_{\rm cl}$.  Applying the formula to both residuals adds
$B_z\lVert\rho_{\rm op}(g_\star)\rVert_F
+B_\phi\lVert\rho_{\rm cl}(g_\star)\rVert_F$ to the right-hand side of
Eq.~\ref{eq:appendix-local-ood-bound}.  Training error cannot control a
remainder that appears only in a region with little or no training mass.

\subsection{Local errors under repeated composition}
\label{app:theory-recursion}

Teacher forcing measures the local error before predictions are reused. Free
rollout feeds each predicted latent back into the predictor, so later learned
iterates act on every earlier defect. The following recursion makes this
dependence exact.

At a fixed gravity $g$, let the true and predicted latent transitions be
\begin{align}
z_{j+1}
&=A_W(g)z_j+C_W(g)\phi_j+W\xi_{j+1},\\
\widehat z_{j+1}
&=\widehat A(g)\widehat z_j,
\qquad
\widehat z_0=z_0.
\label{eq:appendix-true-predicted-rollout}
\end{align}
Define the rollout error $e_j=z_j-\widehat z_j$.

\begin{theorem}[Exact recursive latent error]
\label{thm:appendix-exact-recursion}
For every $j\geq0$,
\begin{equation}
e_{j+1}
=\widehat A(g)e_j+\delta_j(g)+W\xi_{j+1}.
\label{eq:appendix-one-step-error-recursion}
\end{equation}
Consequently, for every $\mathfrak{h}\geq1$,
\begin{equation}
e_{\mathfrak{h}}
=\sum_{j=0}^{\mathfrak{h}-1}
\widehat A(g)^{\mathfrak{h}-1-j}
\left[\delta_j(g)+W\xi_{j+1}\right].
\label{eq:appendix-exact-error-recursion}
\end{equation}
\end{theorem}

\begin{proof}
Subtract the predicted transition from the true transition:
\begin{align}
e_{j+1}
&=A_W(g)z_j+C_W(g)\phi_j+W\xi_{j+1}
-\widehat A(g)\widehat z_j.
\end{align}
Add and subtract $\widehat A(g)z_j$:
\begin{align}
e_{j+1}
&=\widehat A(g)(z_j-\widehat z_j)
+[A_W(g)-\widehat A(g)]z_j
+C_W(g)\phi_j+W\xi_{j+1}\\
&=\widehat A(g)e_j+\delta_j(g)+W\xi_{j+1},
\end{align}
which proves Eq.~\ref{eq:appendix-one-step-error-recursion}.  Since
$e_0=0$, the first two steps are
\begin{align}
e_1&=\delta_0+W\xi_1,\\
e_2&=\widehat A(\delta_0+W\xi_1)+\delta_1+W\xi_2.
\end{align}
Repeating the substitution gives
Eq.~\ref{eq:appendix-exact-error-recursion}.  The same conclusion also
follows by induction on $\mathfrak{h}$.
\end{proof}

Since $\widehat A(g)^q=\widehat S_g(q)$, each one-step error is transformed by
the remaining learned semigroup iterate. The terms are vectors, so they may
reinforce or cancel one another. The identity does not imply that error, or the
difference between two models, must grow monotonically with the horizon.

\subsubsection{Physical readout and the central bound}

Latent error and physical error are not the same quantity.  The readout can
hide some latent errors, and the representation can omit physical information
even when its latent rollout is accurate.

\begin{proposition}[Exact physical readout decomposition]
\label{prop:appendix-physical-decomposition}
Let $s_{\mathfrak{h}}=C_{\rm phys}\phi_{\mathfrak{h}}$ and let $r$ be any
frozen probe.  Define the
true-latent readout residual
\begin{equation}
b_r(\phi_{\mathfrak{h}})=s_{\mathfrak{h}}-r(z_{\mathfrak{h}}).
\end{equation}
Then
\begin{equation}
s_{\mathfrak{h}}-r(\widehat z_{\mathfrak{h}})
=b_r(\phi_{\mathfrak{h}})+r(z_{\mathfrak{h}})-r(\widehat z_{\mathfrak{h}}).
\label{eq:appendix-physical-decomposition}
\end{equation}
The corresponding mean-squared error is
\begin{align}
\mathbb E\lVert s_{\mathfrak{h}}-r(\widehat z_{\mathfrak{h}})\rVert_2^2
={}&\mathbb E\lVert b_r(\phi_{\mathfrak{h}})\rVert_2^2
+\mathbb E\lVert r(z_{\mathfrak{h}})-r(\widehat z_{\mathfrak{h}})\rVert_2^2\nonumber\\
&+2\mathbb E\!\left[
b_r(\phi_{\mathfrak{h}})^\top
\{r(z_{\mathfrak{h}})-r(\widehat z_{\mathfrak{h}})\}
\right].
\label{eq:appendix-physical-mse-decomposition}
\end{align}
For the linear probe $r(z)=Rz$,
\begin{equation}
b_r(\phi_{\mathfrak{h}})=(C_{\rm phys}-RW)\phi_{\mathfrak{h}},
\qquad
r(z_{\mathfrak{h}})-r(\widehat z_{\mathfrak{h}})=Re_{\mathfrak{h}}.
\label{eq:appendix-linear-probe-specialization}
\end{equation}
\end{proposition}

\begin{proof}
Add and subtract $r(z_{\mathfrak{h}})$:
\begin{align}
s_{\mathfrak{h}}-r(\widehat z_{\mathfrak{h}})
&=s_{\mathfrak{h}}-r(z_{\mathfrak{h}})
+r(z_{\mathfrak{h}})-r(\widehat z_{\mathfrak{h}}).
\end{align}
Expanding the squared Euclidean norm and taking expectations gives
Eq.~\ref{eq:appendix-physical-mse-decomposition}.  Substituting
$r(z)=Rz$ gives Eq.~\ref{eq:appendix-linear-probe-specialization}.
\end{proof}
For the remainder of this subsection, specialize to $r(z)=Rz$ and write
$\widehat s_{\mathfrak{h}}=R\widehat z_{\mathfrak{h}}$.
For $\mathfrak{h}\geq1$, define the physical rollout gain
\begin{equation}
\Gamma_{\mathfrak{h}}(g;R,\widehat A)
=\sum_{q=0}^{\mathfrak{h}-1}\lVert R\widehat A(g)^q\rVert_2.
\label{eq:appendix-rollout-gain}
\end{equation}
This quantity bounds how strongly the learned dynamics and readout can
propagate a sequence of local latent errors.

\begin{theorem}[Pathwise physical error under an unseen gravity]
\label{thm:appendix-central-bound}
Assume the shared law basis in
Eq.~\ref{eq:appendix-shared-law-basis}.  Suppose
$\lVert z_j\rVert_2\leq B_z$ and
$\lVert\phi_j\rVert_2\leq B_\phi$ along the test trajectory at $g_\star$.
Then
\begin{align}
\lVert s_{\mathfrak{h}}-\widehat s_{\mathfrak{h}}\rVert_2
\leq{}&
\lVert(C_{\rm phys}-RW)\phi_{\mathfrak{h}}\rVert_2
\nonumber\\
&+\Gamma_{\mathfrak{h}}(g_\star;R,\widehat A)
\sqrt{\mathcal L_{\rm law}(g_\star)}
\left(B_z\epsilon_{\rm op}+B_\phi\epsilon_{\rm cl}\right)
\nonumber\\
&+\sum_{j=0}^{\mathfrak{h}-1}
\left\lVert
R\widehat A(g_\star)^{\mathfrak{h}-1-j}W\xi_{j+1}
\right\rVert_2.
\label{eq:appendix-central-bound}
\end{align}
\end{theorem}

\begin{proof}
Proposition~\ref{prop:appendix-physical-decomposition} and
Theorem~\ref{thm:appendix-exact-recursion} give
\begin{align}
s_{\mathfrak{h}}-\widehat s_{\mathfrak{h}}
={}&(C_{\rm phys}-RW)\phi_{\mathfrak{h}}\\
&+\sum_{j=0}^{\mathfrak{h}-1}
R\widehat A(g_\star)^{\mathfrak{h}-1-j}\delta_j(g_\star)\\
&+\sum_{j=0}^{\mathfrak{h}-1}
R\widehat A(g_\star)^{\mathfrak{h}-1-j}W\xi_{j+1}.
\end{align}
Apply the triangle inequality.  For the conditional-mean defect terms,
Theorem~\ref{thm:appendix-local-ood-bound} gives the uniform bound
\begin{equation}
\lVert\delta_j(g_\star)\rVert_2
\leq
\sqrt{\mathcal L_{\rm law}(g_\star)}
\left(B_z\epsilon_{\rm op}+B_\phi\epsilon_{\rm cl}\right).
\end{equation}
Therefore
\begin{align}
&\sum_{j=0}^{\mathfrak{h}-1}
\left\lVert
R\widehat A(g_\star)^{\mathfrak{h}-1-j}\delta_j(g_\star)
\right\rVert_2\\
&\quad\leq
\sum_{j=0}^{\mathfrak{h}-1}
\lVert R\widehat A(g_\star)^{\mathfrak{h}-1-j}\rVert_2
\lVert\delta_j(g_\star)\rVert_2\\
&\quad\leq
\Gamma_{\mathfrak{h}}(g_\star;R,\widehat A)
\sqrt{\mathcal L_{\rm law}(g_\star)}
\left(B_z\epsilon_{\rm op}+B_\phi\epsilon_{\rm cl}\right).
\end{align}
Keeping the transition-residual terms explicit proves Eq.~\ref{eq:appendix-central-bound}.
\end{proof}

The bound separates the four quantities that affect physical rollout quality:
the readout residual, the local unseen-law error, its recursive propagation,
and the transition residuals.  It is an upper bound, not an additive model of
the observed error.  The factor $\Gamma_{\mathfrak{h}}$ is a worst-case bound
and does not describe cancellation between propagated residuals.  Because the readout residual, trajectory bounds, rollout gain, and realized transition residuals remain in the statement, this is not a stand-alone generalization guarantee based only on the law-coverage factor.

\subsubsection{Cross-step second moments}

The orientation and cross-step dependence of the one-step errors also matter.  Let
\begin{equation}
\varepsilon_j=\delta_j(g)+W\xi_{j+1},
\qquad
S_{ij}(g)=\mathbb E[\varepsilon_i\varepsilon_j^\top\mid g].
\end{equation}
From Eq.~\ref{eq:appendix-exact-error-recursion}, with
$\widehat A=\widehat A(g)$,
\begin{align}
\mathbb E[\lVert Re_{\mathfrak{h}}\rVert_2^2\mid g]
=\sum_{i=0}^{\mathfrak{h}-1}\sum_{j=0}^{\mathfrak{h}-1}
\operatorname{tr}\!\left(
R\widehat A^{\mathfrak{h}-1-i}S_{ij}(g)
(\widehat A^{\mathfrak{h}-1-j})^\top R^\top
\right).
\label{eq:appendix-rollout-second-moment}
\end{align}
To derive this identity, write
$Re_{\mathfrak{h}}=\sum_jR\widehat A^{\mathfrak{h}-1-j}\varepsilon_j$,
expand the outer product, and take
the trace.  No zero-mean assumption is needed because $S_{ij}(g)$ is a second
moment rather than a covariance.  Its off-diagonal entries record cross-step
dependence.  A reset comparison changes these second moments as well as the
propagation path, so it does not isolate a single spectral gain.

The same idea does not depend on linearity.  If the learned nonlinear
transition $\widehat F_g$ is $L_g$-Lipschitz on a set containing the true and
predicted trajectories, and
$d_j=\lVert\widehat F_g(x_j)-F_g(x_j)\rVert_2$, then
\begin{equation}
\lVert\widehat x_{\mathfrak{h}}-x_{\mathfrak{h}}\rVert_2
\leq\sum_{j=0}^{\mathfrak{h}-1}L_g^{\mathfrak{h}-1-j}d_j.
\label{eq:appendix-nonlinear-compounding}
\end{equation}
Indeed, one step gives
$\lVert\widehat x_{j+1}-x_{j+1}\rVert_2
\leq L_g\lVert\widehat x_j-x_j\rVert_2+d_j$, and iteration from a common
initial state proves the claim.  This nonlinear bound supplies intuition, but
the exact linear recursion is the result used above.

\subsection{Why multi-step composition can change the learned representation}
\label{app:theory-horizon-geometry}

The rollout horizon changes the geometry used to compare representations, even
in a stable linear system. This section isolates the conditional-mean,
action-free system and omits residual variation. The quadratic costs below
therefore measure transition bias rather than irreducible noise. The
finite-history version with actions appears in Section~\ref{app:theory-history}.

Fix one smooth regime, suppress $g$, and let
$\widehat S_A(\mathfrak h)=A^{\mathfrak h}$ for a candidate latent operator
$A$. Define the finite-horizon semigroup-intertwining defect
\begin{equation}
D_{\mathfrak{h}}(W;A)
=W S_g(\mathfrak h)-\widehat S_A(\mathfrak h)W
=WT^{\mathfrak{h}}-A^{\mathfrak{h}}W,
\qquad
D_1=WT-AW.
\label{eq:appendix-h-step-defect}
\end{equation}
When $A=\widehat A(g)$, $D_1\phi_t=\delta_t(g)$. When $A=A_W(g)$,
$D_1=C_W(g)$ is the representation closure residual. A general fitted $A$
contains both closure and predictor error.

\begin{theorem}[Finite-horizon intertwining-defect geometry]
\label{thm:appendix-closure-geometry}
For every $\mathfrak{h}\geq1$,
\begin{equation}
D_{\mathfrak{h}}
=\sum_{j=0}^{\mathfrak{h}-1}A^{\mathfrak{h}-1-j}D_1T^j.
\label{eq:appendix-closure-telescope}
\end{equation}
Let
\begin{equation}
M_{\mathfrak{h}}(A,T)
=\sum_{j=0}^{\mathfrak{h}-1}(T^j)^\top\otimes A^{\mathfrak{h}-1-j}.
\label{eq:appendix-defect-propagator}
\end{equation}
For the feature second-moment matrix
$\Sigma_\phi=\mathbb E[\phi\phi^\top]\succeq0$ and nonnegative weights
$w_1,\ldots,w_K$,
\begin{equation}
\sum_{k=1}^{K}w_k
\operatorname{tr}(D_k\Sigma_\phi D_k^\top)
=\operatorname{vec}(D_1)^\top G_K\operatorname{vec}(D_1),
\label{eq:appendix-closure-quadratic}
\end{equation}
where
\begin{equation}
G_K
=\sum_{k=1}^{K}w_k
M_k^\top(\Sigma_\phi\otimes I_{d_z})M_k
\succeq0.
\label{eq:appendix-closure-gramian}
\end{equation}
If $w_1>0$ and $\Sigma_\phi\succ0$, then $G_K\succ0$.
\end{theorem}

\begin{proof}
For $\mathfrak{h}=1$, Eq.~\ref{eq:appendix-closure-telescope} is the
definition of $D_1$.  Suppose the identity holds at $\mathfrak{h}$.  Then
\begin{align}
D_{\mathfrak{h}+1}
&=WT^{\mathfrak{h}+1}-A^{\mathfrak{h}+1}W\\
&=(WT^{\mathfrak{h}}-A^{\mathfrak{h}}W)T
+A^{\mathfrak{h}}(WT-AW)\\
&=D_{\mathfrak{h}}T+A^{\mathfrak{h}}D_1\\
&=\sum_{j=0}^{\mathfrak{h}-1}
A^{\mathfrak{h}-1-j}D_1T^{j+1}+A^{\mathfrak{h}}D_1\\
&=\sum_{j=0}^{\mathfrak{h}}A^{\mathfrak{h}-j}D_1T^j.
\end{align}
Thus the telescoping identity holds by induction.

Using $\operatorname{vec}(LXR)=(R^\top\otimes L)\operatorname{vec}(X)$ in each term of Eq.~\ref{eq:appendix-closure-telescope} gives
\begin{equation}
\operatorname{vec}(D_{\mathfrak{h}})
=M_{\mathfrak{h}}\operatorname{vec}(D_1).
\end{equation}
Also,
\begin{equation}
\operatorname{tr}(D_{\mathfrak{h}}\Sigma_\phi D_{\mathfrak{h}}^\top)
=\operatorname{vec}(D_{\mathfrak{h}})^\top
(\Sigma_\phi\otimes I_{d_z})
\operatorname{vec}(D_{\mathfrak{h}}).
\end{equation}
Substitution and summation give
Eq.s~\ref{eq:appendix-closure-quadratic} and
\ref{eq:appendix-closure-gramian}.  Every summand in $G_K$ is positive
semidefinite.  At $\mathfrak{h}=1$, $M_1=I$, so the first summand is
$w_1(\Sigma_\phi\otimes I_{d_z})$, which is positive definite when
$w_1>0$ and $\Sigma_\phi\succ0$.
\end{proof}

The telescoping identity shows how a one-step failure of intertwining enters
later compositions. If $D_1=0$, then $D_{\mathfrak h}=0$ at every horizon. If
$D_1\neq0$, the true and learned transition geometry determines whether the
propagated terms reinforce or cancel. For fixed $A$, $T$, and $\Sigma_\phi$,
$G_K$ weights these one-step defect directions by their later consequences.
Two defects with the same one-step size can therefore have different
multi-step costs. If fitting a new representation also changes $A$, one common
$G_K$ need not rank all representations.

\begin{theorem}[When training and evaluation weight defects differently]
\label{thm:appendix-profile-mismatch}
Let $G_{\rm tr}\succ0$ and $G_{\rm ev}\succeq0$ be the defect-weighting matrices for a
training horizon profile and an evaluation profile on the same defect space.
Set
\begin{equation}
Q=G_{\rm tr}^{-1/2}G_{\rm ev}G_{\rm tr}^{-1/2}.
\end{equation}
Then:
\begin{enumerate}[label=(\alph*)]
\item If $G_{\rm ev}=cG_{\rm tr}$ for some $c\geq0$, every defect has
evaluation cost $c$ times its training cost.
\item If $Q$ is not a scalar multiple of the identity, there are two defects
with equal training cost and different evaluation cost.
\item Let $S$ be a nonzero subspace of the whitened coordinate $y$, let $P_S$
be its orthogonal projector, and define the compressed operator
$Q_S=P_SQ|_S$.  If $Q_S\succ0$, the largest ratio between evaluation costs of
equal-training-cost defects whose whitened coordinates lie in $S$ is
$\lambda_{\max}(Q_S)/\lambda_{\min}(Q_S)$.
\end{enumerate}
\end{theorem}

\begin{proof}
For a vectorized defect $d$, let $y=G_{\rm tr}^{1/2}d$.  Its training cost is
\begin{equation}
d^\top G_{\rm tr}d=y^\top y,
\end{equation}
and its evaluation cost is
\begin{equation}
d^\top G_{\rm ev}d=y^\top Qy.
\end{equation}
If $G_{\rm ev}=cG_{\rm tr}$, then $Q=cI$, proving part (a).  If $Q$ is not
scalar, choose unit eigenvectors $y_1,y_2$ with distinct eigenvalues.  They
have the same training cost but different evaluation costs, proving part (b).
For $y\in S$, the evaluation cost is $y^\top Q_Sy$.  The Rayleigh quotient of
$Q_S$ ranges between its smallest and largest eigenvalues.  Choosing the
corresponding unit eigenvectors in $S$ attains the ratio in part (c).  If a
subspace is specified in the original defect coordinate $d$, its image under
$G_{\rm tr}^{1/2}$ is the corresponding subspace $S$ in whitened coordinates.
\end{proof}

\subsubsection{A stable ranking reversal}

Explicit representations can make one-step and two-step profiling prefer
opposite choices.  The construction below is strictly stable, so the reversal
does not rely on an unstable transition.

\begin{theorem}[One-step preference can reverse under rollout]
\label{thm:appendix-ranking-reversal}
Let $\phi_t\in\mathbb R^4$ have identity second moment and
\begin{equation}
T=\operatorname{diag}\!\left(
\frac25,\frac45,\frac3{20},\frac35
\right).
\end{equation}
Restrict the rank-one representation to
\begin{equation}
W_P=\frac1{\sqrt2}(1,1,0,0),
\qquad
W_F=\frac1{\sqrt2}(0,0,1,1),
\end{equation}
and let the latent predictor be the scalar $A=a$.  Then:
\begin{enumerate}[label=(\alph*)]
\item the profiled one-step loss prefers $W_P$;
\item using each representation's one-step-optimal predictor, the horizon-two
loss prefers $W_F$;
\item for the normalized objective
\begin{equation}
\frac{\lVert D_1\rVert_F^2+\gamma\lVert D_2\rVert_F^2}{1+\gamma},
\end{equation}
full profiling prefers $W_F$ whenever the sufficient condition
$\gamma>3600/9007\approx0.39969$, including $\gamma=0.95$.
\end{enumerate}
\end{theorem}

\begin{proof}
For a representation that averages two eigenmodes $\lambda$ and $\mu$,
\begin{equation}
\lVert D_{\mathfrak{h}}(W;a)\rVert_F^2
=\frac12\left[
(\lambda^{\mathfrak{h}}-a^{\mathfrak{h}})^2
+(\mu^{\mathfrak{h}}-a^{\mathfrak{h}})^2
\right].
\label{eq:appendix-two-mode-cost}
\end{equation}
At one step, differentiating with respect to $a$ gives the minimizer
$a=(\lambda+\mu)/2$ and minimum $(\lambda-\mu)^2/4$.  Therefore
\begin{align}
\min_a\lVert D_1(W_P;a)\rVert_F^2
&=\frac{(\frac45-\frac25)^2}{4}=\frac1{25},\\
\min_a\lVert D_1(W_F;a)\rVert_F^2
&=\frac{(\frac35-\frac3{20})^2}{4}=\frac{81}{1600}.
\end{align}
Since $1/25<81/1600$, one-step profiling prefers $W_P$.

The one-step-optimal predictors are $a_P=3/5$ and $a_F=3/8$.  At horizon two,
Equation~\eqref{eq:appendix-two-mode-cost} gives
\begin{align}
\lVert D_2(W_P;a_P)\rVert_F^2
&=\frac12\left[
\left(\frac4{25}-\frac9{25}\right)^2
+\left(\frac{16}{25}-\frac9{25}\right)^2
\right]
=\frac{37}{625},\\
\lVert D_2(W_F;a_F)\rVert_F^2
&=\frac12\left[
\left(\frac9{400}-\frac9{64}\right)^2
+\left(\frac9{25}-\frac9{64}\right)^2
\right]\\
&=\frac12\left[
\left(-\frac{189}{1600}\right)^2
+\left(\frac{351}{1600}\right)^2
\right]
=\frac{79461}{2560000}.
\end{align}
Here $79461/2560000<37/625$, which proves the horizon-two reversal.

For the joint two-step objective, allow the horizon-one and horizon-two scalar
predictions for $W_P$ to vary independently.  This relaxation can only lower
its profiled cost, and it gives the lower bound
\begin{equation}
\inf_a
\frac{\lVert D_1(W_P;a)\rVert_F^2
+\gamma\lVert D_2(W_P;a)\rVert_F^2}{1+\gamma}
\geq
\frac{\frac1{25}+\gamma\frac{36}{625}}{1+\gamma}.
\label{eq:appendix-WP-lower-bound}
\end{equation}
For $W_F$, evaluate the admissible predictor $a=2/5$.  Its one-step and
two-step costs are
\begin{align}
\lVert D_1(W_F;2/5)\rVert_F^2&=\frac{41}{800},\\
\lVert D_2(W_F;2/5)\rVert_F^2&=\frac{377}{12800}.
\end{align}
Hence its profiled cost is at most
\begin{equation}
\frac{\frac{41}{800}+\gamma\frac{377}{12800}}{1+\gamma}.
\label{eq:appendix-WF-upper-bound}
\end{equation}
The upper bound in Eq.~\ref{eq:appendix-WF-upper-bound} is smaller than
the lower bound in Eq.~\ref{eq:appendix-WP-lower-bound} when
\begin{align}
\gamma\left(\frac{36}{625}-\frac{377}{12800}\right)
&>\frac{41}{800}-\frac1{25},\\
\gamma\frac{9007}{320000}&>\frac{3600}{320000},\\
\gamma&>\frac{3600}{9007}.
\end{align}
In that range, even the displayed upper bound for $W_F$ is below a lower bound
for $W_P$, so full profiling must prefer $W_F$.
\end{proof}

This construction proves that the objective can change representation
preference as the rollout horizon changes.  It does not prove that a longer
rollout always helps, that a particular predictor architecture causes the
change, or that this four-mode system is the mechanism learned by SG-JEPA.

\subsection{Relating the theory to the history-based predictor}
\label{app:theory-history}

The preceding sections used a memoryless predictor to keep the mechanism
visible.  SG-JEPA instead reads a finite latent history together with actions
and gravity.  A companion state turns this interface into a one-step system on
an augmented space, so the same algebra applies without treating the GRU as a
literal matrix.

For context length $H$, define the feature-history state and its encoded form
\begin{equation}
X_t^{(H)}=
\begin{bmatrix}
\phi_{t-H+1}\\[-0.2em]
\vdots\\[-0.2em]
\phi_t
\end{bmatrix}
\in\mathbb R^{Hd_\phi},
\qquad
\mathbb W=I_H\otimes W,
\qquad
Z_t^{(H)}=\mathbb W X_t^{(H)}.
\label{eq:appendix-history-state}
\end{equation}
Let $a_t\in\mathbb R^{d_a}$ contain every known exogenous action input needed
for the transition from $t$ to $t+1$.  If the interface uses an action window,
$a_t$ denotes that fixed-dimensional stacked input.  The true and predicted
rollouts receive the same action sequence.  Gravity remains fixed during the
rollout and explicit through the law-indexed operators.  Closed-loop action
differences would introduce additional terms and are outside this prediction
analysis.

For the nonlinear predictor, let
$\widehat\Phi_{\theta,g}^{(k)}
\left(Z_t^{(H)};a_{t:t+k-1}\right)$ denote the predicted companion state
obtained after the consecutive action block. Recursive application gives the
control-sequence composition law
\begin{align}
&\widehat\Phi_{\theta,g}^{(k+\ell)}
\left(Z_t^{(H)};a_{t:t+k+\ell-1}\right)
\nonumber\\
&\quad=
\widehat\Phi_{\theta,g}^{(\ell)}
\left(
\widehat\Phi_{\theta,g}^{(k)}
\left(Z_t^{(H)};a_{t:t+k-1}\right);
a_{t+k:t+k+\ell-1}
\right).
\label{eq:appendix-controlled-composition}
\end{align}
The predictor parameters are shared, but each numerical transition is indexed
by its action input. In the action-free case, the companion update generates
an ordinary discrete semigroup. With known actions, it obeys the controlled
composition law above rather than an autonomous semigroup on an individual
latent state.

Consider the linear companion systems
\begin{align}
X_{t+1}^{(H)}
&=\mathbb T(g)X_t^{(H)}
+\mathbb B(g)a_t+\mathbb J_\xi\xi_{t+1},
\label{eq:appendix-true-companion}\\
\widehat Z_{t+1}^{(H)}
&=\mathbb A(g)\widehat Z_t^{(H)}
+\mathbb D(g)a_t.
\label{eq:appendix-predicted-companion}
\end{align}
The first $H-1$ block rows of $\mathbb T(g)$ and $\mathbb A(g)$ shift the
history.  The final block row performs the true or learned update.  The matrix
$\mathbb J_\xi\in\mathbb R^{Hd_\phi\times d_\phi}$ inserts the new residual in
the last feature block.  The remaining dimensions are
\begin{align}
\mathbb T(g)&\in\mathbb R^{Hd_\phi\times Hd_\phi},
&\mathbb B(g)&\in\mathbb R^{Hd_\phi\times d_a},\\
\mathbb A(g)&\in\mathbb R^{Hd_z\times Hd_z},
&\mathbb D(g)&\in\mathbb R^{Hd_z\times d_a}.
\end{align}
The additive maps $\mathbb B(g)$ and $\mathbb D(g)$ are assumptions of this
model.  Nonlinear interactions among state, action, and gravity in
the implemented predictor fall outside this additive model.
For $1\leq L\leq H$, let
\begin{equation}
P_{H,L}
=
\begin{bmatrix}
0_{Ld_z\times(H-L)d_z}&I_{Ld_z}
\end{bmatrix}
\in\mathbb R^{Ld_z\times Hd_z}
\label{eq:appendix-history-selector}
\end{equation}
select the newest $L$ latent blocks, and write $P_H=P_{H,1}$.  At the
start of rollout, both systems use the same true context:
\begin{equation}
\widehat Z_t^{(H)}=\mathbb W X_t^{(H)}.
\label{eq:appendix-exact-context-initialization}
\end{equation}

For $\mathfrak{h}\geq1$ and $0\leq j<\mathfrak{h}$, define
\begin{align}
\mathfrak D_{\mathfrak{h}}(g)
&=P_H\left[
\mathbb W\mathbb T(g)^{\mathfrak{h}}
-\mathbb A(g)^{\mathfrak{h}}\mathbb W
\right],
\label{eq:appendix-history-defect}\\
\mathfrak F_{\mathfrak{h},j}(g)
&=P_H\left[
\mathbb W\mathbb T(g)^{\mathfrak{h}-1-j}\mathbb B(g)
-\mathbb A(g)^{\mathfrak{h}-1-j}\mathbb D(g)
\right],
\label{eq:appendix-action-defect}\\
\mathfrak N_{\mathfrak{h},j}(g)
&=P_H\mathbb W\mathbb T(g)^{\mathfrak{h}-1-j}\mathbb J_\xi.
\label{eq:appendix-history-residual-map}
\end{align}

\begin{theorem}[Exact finite-history rollout error]
\label{thm:appendix-finite-history}
Under Eq.s~\ref{eq:appendix-true-companion} through
\ref{eq:appendix-exact-context-initialization}, the newest latent error at
horizon $\mathfrak{h}$ is
\begin{equation}
z_{t+\mathfrak{h}}-\widehat z_{t+\mathfrak{h}}
=\mathfrak D_{\mathfrak{h}}(g)X_t^{(H)}
+\sum_{j=0}^{\mathfrak{h}-1}\mathfrak F_{\mathfrak{h},j}(g)a_{t+j}
+\sum_{j=0}^{\mathfrak{h}-1}\mathfrak N_{\mathfrak{h},j}(g)\xi_{t+j+1}.
\label{eq:appendix-finite-history-error}
\end{equation}
If $V_{t,\mathfrak{h}}$ stacks $X_t^{(H)}$, the $\mathfrak{h}$ future actions,
and the $\mathfrak{h}$ transition residuals, and
$\mathfrak M_{\mathfrak{h}}(g)$ is the corresponding block row in
Equation~\eqref{eq:appendix-finite-history-error}, then
\begin{equation}
\mathbb E\!\left[
\lVert z_{t+\mathfrak{h}}-\widehat z_{t+\mathfrak{h}}\rVert_2^2\mid g
\right]
=\operatorname{tr}\!\left(
\mathfrak M_{\mathfrak{h}}(g)
\mathbb E[V_{t,\mathfrak{h}}V_{t,\mathfrak{h}}^\top\mid g]
\mathfrak M_{\mathfrak{h}}(g)^\top
\right).
\label{eq:appendix-finite-history-second-moment}
\end{equation}
No centering or independence assumption is required.
\end{theorem}

\begin{proof}
Repeated substitution in the true companion system gives
\begin{align}
X_{t+\mathfrak{h}}^{(H)}
={}&\mathbb T(g)^{\mathfrak{h}}X_t^{(H)}\\
&+\sum_{j=0}^{\mathfrak{h}-1}
\mathbb T(g)^{\mathfrak{h}-1-j}
\left[\mathbb B(g)a_{t+j}+\mathbb J_\xi\xi_{t+j+1}\right].
\label{eq:appendix-unrolled-true-companion}
\end{align}
The exact initial context and recursive predictor give
\begin{align}
\widehat Z_{t+\mathfrak{h}}^{(H)}
={}&\mathbb A(g)^{\mathfrak{h}}\mathbb W X_t^{(H)}\\
&+\sum_{j=0}^{\mathfrak{h}-1}
\mathbb A(g)^{\mathfrak{h}-1-j}\mathbb D(g)a_{t+j}.
\label{eq:appendix-unrolled-predicted-companion}
\end{align}
The true newest latent is $P_H\mathbb W X_{t+\mathfrak{h}}^{(H)}$, and the
predicted newest latent is $P_H\widehat Z_{t+\mathfrak{h}}^{(H)}$.  Apply these selectors to Eq.s~\ref{eq:appendix-unrolled-true-companion} and \ref{eq:appendix-unrolled-predicted-companion}, subtract, and group the context, action, and transition-residual terms.  Their coefficients are exactly $\mathfrak D_{\mathfrak{h}}$, $\mathfrak F_{\mathfrak{h},j}$, and $\mathfrak N_{\mathfrak{h},j}$.

For the second statement, write the error as
$\mathfrak M_{\mathfrak{h}}(g)V_{t,\mathfrak{h}}$.  Then
\begin{align}
\mathbb E\lVert\mathfrak M_{\mathfrak{h}}V_{t,\mathfrak{h}}\rVert_2^2
&=\mathbb E\operatorname{tr}\!\left(
\mathfrak M_{\mathfrak{h}}V_{t,\mathfrak{h}}V_{t,\mathfrak{h}}^\top
\mathfrak M_{\mathfrak{h}}^\top
\right)\\
&=\operatorname{tr}\!\left(
\mathfrak M_{\mathfrak{h}}
\mathbb E[V_{t,\mathfrak{h}}V_{t,\mathfrak{h}}^\top]
\mathfrak M_{\mathfrak{h}}^\top
\right).
\end{align}
Conditioning throughout on $g$ proves
Eq.~\ref{eq:appendix-finite-history-second-moment}.
\end{proof}

Theorem~\ref{thm:appendix-finite-history} is stated for the newest latent to
match the rollout loss.  Replacing $P_H$ by $P_{H,L}$ in
Eq.s~\ref{eq:appendix-history-defect}--\ref{eq:appendix-history-residual-map}
gives the same exact identity for the stacked window
$z_{t+\mathfrak{h}-L+1:t+\mathfrak{h}}
-\widehat z_{t+\mathfrak{h}-L+1:t+\mathfrak{h}}$.
A linear probe on that window is
then represented by a matrix multiplying $P_{H,L}$.

\begin{corollary}[Memoryless results on the companion state]
\label{cor:appendix-companion-reduction}
Suppose the augmented true and learned operators share the law basis used in
Section~\ref{app:theory-law-transfer}.  Then the local-defect, law-transfer,
and recursive-error results apply on the companion state after the substitutions
\begin{equation}
W\mapsto\mathbb W,
\qquad
T(g)\mapsto\mathbb T(g),
\qquad
\widehat A(g)\mapsto\mathbb A(g),
\end{equation}
followed by the selector $P_H$.  For the horizon-geometry result, define
\begin{align}
\mathbb D_{\mathfrak{h}}
&=\mathbb W\mathbb T^{\mathfrak{h}}
-\mathbb A^{\mathfrak{h}}\mathbb W,\nonumber\\
\mathbb M_{\mathfrak{h}}
&=\sum_{j=0}^{\mathfrak{h}-1}
(\mathbb T^j)^\top\otimes\mathbb A^{\mathfrak{h}-1-j},
\qquad
\Sigma_X=\mathbb E[X_t^{(H)}X_t^{(H)\top}].
\end{align}
The quadratic form for newest-block loss uses
\begin{equation}
G_K^{(P)}
=\sum_{k=1}^{K}w_k\mathbb M_k^\top
\left(\Sigma_X\otimes P_H^\top P_H\right)\mathbb M_k
\succeq0.
\label{eq:appendix-companion-gramian}
\end{equation}
For a linear $L$-block history probe, replace $P_H$ by $P_{H,L}$ in this
quadratic form and compose the result with the probe matrix.
This matrix need not be positive definite on every augmented defect because
$P_H$ ignores the older output blocks.  If $w_1>0$ and $\Sigma_X\succ0$, it is
positive definite after restriction to the admissible companion-defect space
\begin{equation}
\mathcal S_{\rm adm}
=\left\{
\operatorname{vec}\!\left(
\begin{bmatrix}
0_{(H-1)d_z\times Hd_\phi}\\ D
\end{bmatrix}
\right):
D\in\mathbb R^{d_z\times Hd_\phi}
\right\}.
\label{eq:appendix-admissible-companion-defects}
\end{equation}
The zero rows are the matched history shifts.  If the true and learned action maps differ,
the action residuals $\mathfrak F_{\mathfrak{h},j}(g)a_{t+j}$ remain as the explicit extra
terms in Eq.~\ref{eq:appendix-finite-history-error}.  The shared basis
must include an intercept to represent the gravity-independent shift rows of
the companion matrices.
\end{corollary}

\begin{proof}
Eq.s~\ref{eq:appendix-true-companion} and \ref{eq:appendix-predicted-companion} have the same algebraic form as the memoryless systems on their augmented spaces.  Applying the earlier results therefore produces augmented latent errors.  Multiplication by $P_H$ selects the output that the model evaluates.  The action terms vanish only when their true and learned augmented maps agree; otherwise the exact difference is
$\mathfrak F_{\mathfrak{h},j}(g)a_{t+j}$.  For horizon geometry,
$\operatorname{vec}(\mathbb D_{\mathfrak{h}})
=\mathbb M_{\mathfrak{h}}\operatorname{vec}(\mathbb D_1)$.
Selecting the newest block changes the output metric from the identity to
$P_H^\top P_H$, which gives
Eq.~\ref{eq:appendix-companion-gramian}.  Its $\mathfrak{h}=1$ term is strictly
positive for every nonzero defect in $\mathcal S_{\rm adm}$ when $w_1>0$ and
$\Sigma_X\succ0$: $P_H$ retains its only potentially nonzero block row, and
$\Sigma_X$ is positive definite.
\end{proof}

The companion construction is an exact bridge for the history-state model.  It
does not assert that the nonlinear GRU is globally linear.  It states what the
model must contain to represent history, known actions,
gravity, and recursive feedback without dropping terms.

\subsection{Regime changes and contact}
\label{app:theory-events}

A time-homogeneous hybrid map can still generate a discrete semigroup on the
full state. Contact changes the active local branch, not the underlying update
rule, so the branch matrices may vary along a trajectory. The smooth
law-transfer analysis applies within one regime, while the telescoping identity
below applies along a fixed, shared branch sequence. It does not cover errors
that change the event time or send the prediction onto a different branch. The
later hinge result concerns one continuous piecewise-affine boundary relative
to an affine feature dictionary.

Let the true transitions along a path be
$T_0,\ldots,T_{\mathfrak{h}-1}$ and the latent transitions be
$A_0,\ldots,A_{\mathfrak{h}-1}$. Define the branch-local transition defects
\begin{equation}
\Delta_j^{\rm br}=WT_j-A_jW.
\end{equation}

\begin{lemma}[Telescoping identity across changing regimes]
\label{lem:appendix-hybrid-telescope}
With an empty product equal to the identity,
\begin{align}
&WT_{\mathfrak{h}-1}\cdots T_0
-A_{\mathfrak{h}-1}\cdots A_0W
\nonumber\\
&\quad=
\sum_{j=0}^{\mathfrak{h}-1}
A_{\mathfrak{h}-1}\cdots A_{j+1}
\Delta_j^{\rm br}
T_{j-1}\cdots T_0.
\label{eq:appendix-hybrid-telescope}
\end{align}
\end{lemma}

\begin{proof}
For each $j=0,\ldots,\mathfrak{h}$, define
\begin{equation}
Q_j
=A_{\mathfrak{h}-1}\cdots A_jWT_{j-1}\cdots T_0,
\end{equation}
where $Q_{\mathfrak{h}}=WT_{\mathfrak{h}-1}\cdots T_0$ and
$Q_0=A_{\mathfrak{h}-1}\cdots A_0W$. Consecutive terms satisfy
\begin{align}
Q_{j+1}-Q_j
&=A_{\mathfrak{h}-1}\cdots A_{j+1}
(WT_j-A_jW)
T_{j-1}\cdots T_0\\
&=A_{\mathfrak{h}-1}\cdots A_{j+1}
\Delta_j^{\rm br}T_{j-1}\cdots T_0.
\end{align}
Summing $Q_{j+1}-Q_j$ from $j=0$ to $\mathfrak{h}-1$ cancels all intermediate $Q_j$
and yields Eq.~\ref{eq:appendix-hybrid-telescope}.
\end{proof}

The identity has the same interpretation as the fixed-regime telescope, but
the operators before and after each local defect now depend on the active
branch.

\begin{theorem}[A continuous piecewise-affine event requires a hinge]
\label{thm:appendix-event-hinge}
Let an open connected set $U\subset\mathbb R^d$ intersect both sides of the
hyperplane $a^\top x=b$, with $a\neq0$.  Suppose a continuous map
$F:U\to\mathbb R^m$ is affine on each side:
\begin{equation}
F(x)=
\begin{cases}
A_-x+c_-, & a^\top x\leq b,\\
A_+x+c_+, & a^\top x\geq b.
\end{cases}
\label{eq:appendix-piecewise-affine-event}
\end{equation}
Then there is a vector $u\in\mathbb R^m$ such that
\begin{equation}
F(x)=A_-x+c_-+u(a^\top x-b)_+,
\label{eq:appendix-hinge-form}
\end{equation}
where $(r)_+=\max\{r,0\}$.  If $u\neq0$, the hinge
$(a^\top x-b)_+$ is not in the affine feature dictionary on $U$, so affine
features alone cannot represent $F$.  Appending this hinge is sufficient.  In
addition, any feature dictionary that contains the affine functions and admits
an exact linear readout of $F$ must contain the hinge in its linear span.
\end{theorem}

\begin{proof}
Choose a point $x_0\in U$ on the boundary $a^\top x_0=b$.  Such a point exists
because $U$ is connected and intersects both open half-spaces.  Let
$v\in\ker(a^\top)$.  For sufficiently small $\tau$, both
$x_0+\tau v$ and $x_0-\tau v$ remain on the boundary and in $U$.  Continuity
of the two affine formulas on the boundary gives
\begin{equation}
(A_+-A_-)x+(c_+-c_-)=0
\end{equation}
for every boundary point $x$ in a neighborhood of $x_0$.  Subtract this
identity at $x_0+\tau v$ and $x_0$ to obtain
\begin{equation}
(A_+-A_-)v=0
\qquad\text{for every }v\in\ker(a^\top).
\end{equation}
Each row of $A_+-A_-$ therefore annihilates the $(d-1)$-dimensional space
$\ker(a^\top)$.  Its row space lies in the one-dimensional span of $a^\top$,
so some $u\in\mathbb R^m$ satisfies
\begin{equation}
A_+-A_-=ua^\top.
\label{eq:appendix-hinge-linear-change}
\end{equation}
Evaluate the boundary identity at $x_0$.  Since $a^\top x_0=b$,
\begin{equation}
0=ua^\top x_0+c_+-c_-=ub+c_+-c_-,
\end{equation}
and hence $c_+-c_-=-ub$.  On the upper half-space,
\begin{align}
A_+x+c_+
&=A_-x+c_-+u(a^\top x-b),
\end{align}
while the added term is zero on the lower half-space.  This proves
Eq.~\ref{eq:appendix-hinge-form}.

If $u\neq0$ and the hinge were affine on $U$, it would vanish as an affine
function on the nonempty open set $U\cap\{a^\top x<b\}$.  Its coefficients
would then be zero, forcing it to vanish throughout $U$.  This contradicts its
nonzero values on $U\cap\{a^\top x>b\}$.  Thus an affine dictionary cannot
represent the hinge.  Adding the displayed hinge is sufficient by
Eq.~\ref{eq:appendix-hinge-form}.  Finally, because $u\neq0$, choose
$q\in\mathbb R^m$ with $q^\top u=1$.  Eq.~\ref{eq:appendix-hinge-form}
then gives
\begin{equation}
(a^\top x-b)_+
=q^\top F(x)-q^\top(A_-x+c_-).
\end{equation}
If a dictionary contains affine functions and linearly represents $F$, the
right-hand side lies in its linear span.  The dictionary must therefore span
the hinge, which completes the proof.
\end{proof}

\begin{corollary}[Ballistic contact creates a gravity hinge]
\label{cor:appendix-ballistic-hinge}
Suppose the free-flight height at a fixed time is
$y_{\rm free}(g)=a-bg$ with $b>0$, and a floor clips the height to
\begin{equation}
y(g)=\max\{a-bg,0\}.
\end{equation}
Let $g_c=a/b$.  On any interval that crosses $g_c$,
\begin{equation}
y(g)=b(g_c-g)_+
\end{equation}
is not affine in $g$.
\end{corollary}

\begin{proof}
For $g<g_c$, the height is $a-bg=b(g_c-g)$ and has derivative $-b$.  For
$g>g_c$, the height is zero and has derivative zero.  The derivative changes
at $g_c$, so no single affine function agrees on an interval crossing that
point.  The displayed hinge agrees with both branches.
\end{proof}

The result covers one continuous piecewise-affine boundary.  It does not cover
discontinuous, frictional, restitution-dependent, or multi-contact dynamics,
nor does it imply a literal neural hinge, a larger latent for every collision,
or better contact modeling by SG-JEPA.  It affects latent closure only if the
retained quantity responds to the slope change $u$.

\subsection{Supplementary calculation: SIGReg and its limits}
\label{app:theory-sigreg}

SIGReg enters Equation~\eqref{eq:appendix-profiled-risk}, but it is separate
from the composition and law-coverage arguments. Near Gaussian identity
covariance, its population statistic penalizes scale mismatch and unequal
directional variance, while a finite batch introduces bias. This local
calculation neither proves noncollapse nor attributes the dynamics difference
to SIGReg.

For a batch of $B\geq1$ latent samples
$z_1,\ldots,z_B\in\mathbb R^{d_z}$, a random direction
$a\sim\operatorname{Unif}(\mathbb S^{d_z-1})$, knots $t_\ell$, and weights
$\omega_\ell\geq0$, consider
\begin{equation}
\mathcal R_{\rm SIG}
=B\sum_\ell\omega_\ell
\left|
\frac1B\sum_{i=1}^{B}e^{\mathrm{i}t_\ell a^\top z_i}
-e^{-t_\ell^2/2}
\right|^2.
\label{eq:appendix-sigreg-statistic}
\end{equation}
Here $\mathrm i^2=-1$.  The batch average is the empirical characteristic
function, a Fourier summary of the distribution after projection onto $a$.
For a zero-mean Gaussian with covariance $C$, define the direction-averaged
population functional
\begin{equation}
\mathcal R_{\rm SIG}^{\rm pop}(C)
=\mathbb E_a\!\left[
B\sum_\ell\omega_\ell
\left|
\mathbb E_{z\sim\mathcal N(0,C)}
e^{\mathrm i t_\ell a^\top z}
-e^{-t_\ell^2/2}
\right|^2
\right].
\label{eq:appendix-sigreg-population-functional}
\end{equation}

\begin{proposition}[Local covariance penalty induced by SIGReg]
\label{prop:appendix-sigreg}
Assume some nonzero knot has positive weight.  If $C=I+\Delta_C\succ0$,
$\Delta_C=\Delta_C^\top$, and $\lVert\Delta_C\rVert_F$ is small, then the
functional in Eq.~\ref{eq:appendix-sigreg-population-functional}
satisfies
\begin{align}
\mathcal R_{\rm SIG}^{\rm pop}(I+\Delta_C)
={}&\kappa_{B,d_z}
\left[2\lVert\Delta_C\rVert_F^2+(\operatorname{tr}\Delta_C)^2\right]
+O(\lVert\Delta_C\rVert_F^3),
\label{eq:appendix-sigreg-population-expansion}\\
\kappa_{B,d_z}
={}&\frac{B}{d_z(d_z+2)}
\sum_\ell\omega_\ell
\frac{t_\ell^4}{4}e^{-t_\ell^2}.
\label{eq:appendix-sigreg-kappa}
\end{align}
For a finite batch of independent samples and a fixed direction with
$v=a^\top Ca$, the exact expectation at one knot is
\begin{equation}
B\left(e^{-t^2v/2}-e^{-t^2/2}\right)^2
+1-e^{-t^2v}.
\label{eq:appendix-sigreg-finite-batch}
\end{equation}
Consequently, after averaging over both the independent batch and the random
direction, for a constant $C_0$ that does not depend on $\Delta_C$,
\begin{align}
\mathbb E\mathcal R_{\rm SIG}^{\rm emp}(I+\Delta_C)
={}&C_0+\beta_{1,d_z}\operatorname{tr}\Delta_C
\nonumber\\
&+\kappa_{B,d_z}^{\rm emp}
\left[2\lVert\Delta_C\rVert_F^2+(\operatorname{tr}\Delta_C)^2\right]
+O(\lVert\Delta_C\rVert_F^3),
\label{eq:appendix-sigreg-empirical-expansion}\\
\beta_{1,d_z}
={}&\frac1{d_z}\sum_\ell\omega_\ell t_\ell^2e^{-t_\ell^2},\nonumber\\
\kappa_{B,d_z}^{\rm emp}
={}&\frac{B-2}{d_z(d_z+2)}
\sum_\ell\omega_\ell
\frac{t_\ell^4}{4}e^{-t_\ell^2}.\nonumber
\end{align}
\end{proposition}

\begin{proof}
For fixed $a$, the scalar $a^\top z$ is Gaussian with variance
$v=a^\top Ca$.  Its characteristic function at $t$ is
$e^{-t^2v/2}$.  Write $v=1+\eta$, where
$\eta=a^\top\Delta_Ca$.
A Taylor expansion around $\eta=0$ gives
\begin{align}
e^{-t^2(1+\eta)/2}-e^{-t^2/2}
&=e^{-t^2/2}\left(e^{-t^2\eta/2}-1\right)\\
&=-\frac{t^2}{2}e^{-t^2/2}\eta+O(\eta^2),\\
\left(e^{-t^2(1+\eta)/2}-e^{-t^2/2}\right)^2
&=\frac{t^4}{4}e^{-t^2}\eta^2+O(|\eta|^3).
\end{align}
For a uniform direction on the unit sphere,
\begin{align}
\mathbb E_a[\eta]
&=\frac{\operatorname{tr}\Delta_C}{d_z},\\
\mathbb E_a[\eta^2]
&=\frac{2\lVert\Delta_C\rVert_F^2+(\operatorname{tr}\Delta_C)^2}
{d_z(d_z+2)}.
\label{eq:appendix-spherical-moments}
\end{align}
Multiplying by $B\omega_\ell$, summing over knots, and using the second
identity proves Eq.~\ref{eq:appendix-sigreg-population-expansion}.

For the finite batch, let
$Y_i=e^{\mathrm{i}ta^\top z_i}$.  Since $|Y_i|=1$,
\begin{align}
\mathbb E\left|
\frac1B\sum_{i=1}^{B}Y_i-e^{-t^2/2}
\right|^2
={}&\left|e^{-t^2v/2}-e^{-t^2/2}\right|^2\\
&+\frac1B\left(1-e^{-t^2v}\right).
\end{align}
Multiplication by $B$ gives
Eq.~\ref{eq:appendix-sigreg-finite-batch}.  Finally,
\begin{equation}
1-e^{-t^2(1+\eta)}
=1-e^{-t^2}+t^2e^{-t^2}\eta
-\frac{t^4}{2}e^{-t^2}\eta^2+O(|\eta|^3).
\end{equation}
Combining this with the population mismatch, then applying both identities in
Eq.~\ref{eq:appendix-spherical-moments}, gives the stated values of
$\beta_{1,d_z}$ and $\kappa_{B,d_z}^{\rm emp}$.
\end{proof}

Because at least one nonzero knot has positive weight, $\kappa_{B,d_z}>0$ and
the population functional has a strict local quadratic minimum at identity
covariance.  The expected finite-batch statistic also has a linear trace term,
so identity covariance need not minimize it.  When $B>2$,
$\kappa_{B,d_z}^{\rm emp}>0$ as well, but the linear term remains.  The
expansion supports SIGReg's local scale and isotropy role.  It does not analyze
collapse at $C=0$, connect effective rank to prediction, or explain the
representation difference.

\end{document}